\documentclass[acmsmall,nonacm,preprint]{acmart}
\AtBeginDocument{%
  }

\usepackage{float}
\usepackage{booktabs} 
\usepackage{amsmath}
\usepackage{mathtools}
\usepackage{amsthm}

\theoremstyle{plain}
\newtheorem{theorem}{Theorem}[section]

\newtheorem{lemma}[theorem]{Lemma}
\newtheorem{corollary}[theorem]{Corollary}
\theoremstyle{definition}

\theoremstyle{remark}

\usepackage{algorithm}

\usepackage[normalem]{ulem}

\usepackage{enumitem}
\usepackage[noend]{algpseudocode}
\usepackage{tikz} 
\usetikzlibrary{positioning, shapes.geometric, fit, backgrounds}
\usepackage[dvipsnames]{xcolor}
\usepackage{pifont}
\usepackage{multirow}
\usepackage{enumitem}
\usepackage{mathtools}
\usepackage{hyperref}
\usepackage{caption}
\usepackage{subcaption}
\usetikzlibrary{positioning}
\usepackage{amsthm}

\usetikzlibrary{external}
\usepackage{tabularx}
\newcommand{\cmark}{\ding{51}}%
\newcommand{\ie}{{\em i.e., }}
\newcommand{\eg}{{\em e.g., }}

\renewcommand\vec{\mathbf}
\newcommand{\xbar}[1][t]{\bar{\vec{x}}_{#1}}
\newcommand{\xv}[1][t]{\vec{x}^v_{#1}}
\usepackage{xparse}

\NewDocumentCommand{\xs}{ O{r_t} O{t} }{\vec{x}^{#1}_{#2}}

\NewDocumentCommand{\gxs}{ O{v_t} O{t} }{\vec{x}^{#1}_{#2}}

\NewDocumentCommand{\F}{ O{v_t} O{r_t} O{t} O{t+\hat{\tau}_t}}{
  F_{#1}(\xs[#2][#3], \xi_{#4})
}

\NewDocumentCommand{\Fz}{ O{v_z} O{r_z} O{z+\hat{\tau}_z}}{
  F_{#1}(\xs[#2][z], \xi_{#3})
}

\NewDocumentCommand{\f}{ O{v_t} O{\xs[r_t][t]} }{
  f_{#1}(#2)
}

\NewDocumentCommand{\ff}{ O{1} O{v_t} O{\xs[r_t][t]} }{
  f^{#1}_{#2}(#3)
}

\NewDocumentCommand{\fz}{ O{v_z} O{\xs[r_z][z]} }{
  f_{#1}(#2)
}

\NewDocumentCommand{\gF}{ O{v_t} O{v_t} O{t} O{t+\hat{\tau}_t}}{
  F_{#1}(\xs[#2][#3], \xi_{#4})
}
\NewDocumentCommand{\gFz}{ O{v_z} O{v_z} O{z+\hat{\tau}_z}}{
  F_{#1}(\xs[#2][z], \xi_{#3})
}

\NewDocumentCommand{\GF}{ O{t} O{t+\hat{\tau}_t}}{
  F(\vec{X}_{#1}, \xi_{#2})
}
\NewDocumentCommand{\GFz}{O{z+\hat{\tau}_z}}{
  F(\vec{X}_z, \xi_{#1})
}

\NewDocumentCommand{\gf}{ O{v_t} O{\xs[v_t][t]} }{
  f_{#1}(#2)
}
\NewDocumentCommand{\Gf}{ O{t} }{
  f(\vec{X}_#1)
}

\NewDocumentCommand{\gfz}{ O{v_z} O{\xs[v_z][z]} }{
  f_{#1}(#2)
}

\NewDocumentCommand{\lrt}{ O{r_t} O{t} }{l^{#1}_{#2}}

\NewDocumentCommand{\drt}{ O{r_t} O{t} }{d^{#1}_{#2}}

\newcommand{\xtild}[1][t]{\tilde{\vec{x}}_{#1}}

\DeclareMathOperator{\EX}{\mathbb{E}}
\newcommand{\norm}[1]{\| #1 \|^2}
\newcommand{\fnorm}[1]{\| #1 \|_F^2}
\newcommand{\normm}[1]{\| #1 \|}
\newcommand{\inpr}[2]{\langle #1 , #2 \rangle}

\newcommand{\xz}{\vec{x}_0}

\newcommand{\data}[1]{\mathcal{D}_#1}

\usepackage{xspace,relsize}

\newcommand{\ours}{Multi-Walk\xspace}
\newcommand{\master}{Node $0$\xspace}
\newcommand{\gos}{Gossip\xspace}
\newcommand{\agos}{Asynchronous Gossip\xspace}

\newcommand{\wrt}{w.r.t.\xspace}

\usepackage{hyperref}
\usepackage{url}
\usetikzlibrary{patterns}

\begin{document}

\title[A Tale of Two Learning Algorithms]{A Tale of Two Learning Algorithms: Multiple Stream Random Walk and Asynchronous Gossip}

\author{Peyman Gholami}
\email{pghola2@uic.edu}
\author{Hulya Seferoglu}
\email{hulya@uic.edu}
\affiliation{%
  \institution{University of Illinois at Chicago}
  \city{Chicago}
  \state{Illinois}
  \country{USA}
}









\begin{abstract}
Although gossip and random walk-based learning algorithms are widely known for decentralized learning, there has been limited theoretical and experimental analysis to understand their relative performance for different graph topologies and data heterogeneity. We first design and analyze an asynchronous random walk-based learning algorithm with multiple streams (walks), which we name ``Multi-Walk''. We provide a convergence analysis for Multi-Walk w.r.t number of iterations (computation), wall-clock time, and communication cost. We also present a convergence analysis for ``\agos'', addressing the notable gap in the existing literature by explicitly accounting for its computational and communication overhead.
Our results show that Multi-Walk has better convergence in terms of iterations as compared to \agos in graphs with large diameters (e.g., cycles), while its relative performance, as compared to \agos, depends on the number of walks and the data heterogeneity in graphs with small diameters (e.g., complete graphs).  In wall-clock time analysis, we observe a linear speed-up with the number of walks and nodes in Multi-Walk and \agos, respectively. Finally, we show that Multi-Walk outperforms \agos{} in communication overhead, except in small-diameter topologies with extreme data heterogeneity. 
These results highlight the effectiveness of each algorithm across varying graph topologies and levels of data heterogeneity.
Our \href{https://anonymous.4open.science/r/Asynchronous_Decentralized_Learning/README.md}{codes} are available for reproducibility.

\end{abstract}


\maketitle

\section{Introduction}
Decentralized learning has gained significant attention as a robust alternative to traditional federated learning 
approaches, addressing critical limitations such as communication bottlenecks and single points of failure \cite{Tsits1984DCN, Nedi2009DistributedSM, mcmahan2023communicationefficientlearningdeepnetworks}. Among decentralized methods, two prominent approaches have emerged: gossip and random walk-based algorithms. While both paradigms have been extensively studied \cite{Boyd2006RandomizedGA, Lian2017CanDA, pmlr-v97-koloskova19a, incremental, Ayache2020PrivateWR, Sun2018OnMC, needell2015stochasticgradientdescentweighted}, 
a gap remains in understanding their relative performance and trade-offs across different graph topologies and data heterogeneity. 
Specifically, a comprehensive analysis comparing their convergence rates, communication cost, and computational overhead is still lacking, which constitutes the primary focus of this work.


Gossip algorithms advocate that nodes in a graph iteratively update their models with Stochastic Gradient Descent (SGD) \cite{Robbins1951ASA, bottou2018optimizationmethodslargescalemachine} and exchange the updated models with their neighbors, leading to global consensus over time. 
Gossip can employ synchronous communication \cite{Lian2017CanDA,unified-koloskova20a}, where nodes must wait for all nodes to update their model in each round. However, in the presence of straggler nodes or nodes with varying computation speeds \cite{kairouz2021advancesopenproblemsfederated}, synchronous gossip results in significant idle times for fast nodes and creates bottlenecks \cite{chen2017revisitingdistributedsynchronoussgd}.
Asynchronous gossip algorithms \cite{async1,Tsitsiklis1984DistributedAD,Recht2011HogwildAL} have been developed to leverage resources more effectively, allowing nodes to compute gradients using a stale model and communicate in a asynchronous manner, thereby eliminating the need to wait for all nodes \cite{lian2018asynchronousdecentralizedparallelstochastic, accelerate, nadiradze2022asynchronousdecentralizedsgdquantized, bornstein2022swiftrapiddecentralizedfederated, unifiedasync}.
In both synchronous and asynchronous cases, gossip incurs high communication costs due to frequent message exchange among nodes.

The random walk-based learning algorithms suggest that one node at a time updates a model with its local data. The node then randomly selects a neighbor and sends the updated model to it. This neighbor becomes the next activated node and updates the model using its own local data. This process repeats until convergence. Random walk-based algorithms \cite{Ayache2020PrivateWR,Sun2018OnMC, needell2015stochasticgradientdescentweighted} are typically single stream, i.e., only one node updates the model at any given time, which leads to slow convergence. 
While utilizing multiple streams can improve convergence, the coordination and interaction among these streams remain largely unexplored in the context of random walk-based learning—a gap this paper aims to address.




To investigate the relative performance of gossip-based and random walk-based learning under varying graph topologies and data heterogeneity, we first design and analyze a multi-stream random walk algorithm, which we call asynchronous ``Multi-Walk''. We then conduct a comprehensive comparison of Multi-Walk and \agos in terms of iteration count (computation), wall-clock time, and communication cost. Our main contributions are as follows:

\textbf{Design of Multi-Walk algorithm.} 
We design, for the first time in the literature, a random walk-based learning algorithm with multiple asynchronous streams, Multi-Walk.
The core idea behind Multi-Walk is to improve the convergence rate of random walk-based methods by initiating multiple random walks (streams) simultaneously across the graph. This strategy increases the number of concurrent computations, enabling the algorithm to improve its convergence rate. 
Multi-Walk allows for a trade-off between convergence speed and resource utilization by adjusting the number of walks. There is no need for special coordination among the walks, as each walk operates independently on the graph. Furthermore, we demonstrate that Multi-Walk achieves a linear speedup with the number of walks. 
To the best of our knowledge, Multi-Walk is the first asynchronous random walk-based learning algorithm to leverage multiple parallel streams.

\textbf{Comprehensive analysis of Multi-Walk and \agos.} 
We provide an in-depth examination of both Multi-Walk and \agos algorithms. Specifically, we analyze their convergence properties \wrt iterations (computation), wall-clock time, and communication overhead. This detailed comparison addresses a significant gap in the literature, offering insights into the performance trade-offs of these methods. 
We analyze both algorithms under the assumption of non-convex, smooth, and heterogeneous loss functions, without any upper bounds on computation or communication delays.

\textbf{Theoretical insights.} Our analysis demonstrates that Multi-Walk exhibits superior performance on graphs with larger diameters, while \agos is likely a better choice for small-diameter graphs in terms of iteration complexity. 
Specifically, Multi-Walk outperforms \agos in both iteration complexity and communication overhead on graph topologies such as cycles.
We showed that in iid setting, on graphs where 
\(
p = \mathcal{O}\left( \frac{1}{V} \right),
\)
with \( V \) representing the number of nodes in the network graph, Multi-Walk shows superior performance compared to \agos. Here, \( p \) refers to the spectral gap of \( \vec{P}^\top \vec{P} \), where \( \vec{P} \) is the mixing matrix of \agos. Intuitively, \( p^{-1/2} \) correlates with the graph's diameter. 
When evaluating convergence in terms of clock time, \agos benefits from a linear speed-up with the number of nodes. Multi-Walk outperforms \agos when considering convergence in terms of communication overhead except in small-diameter graphs with extreme data heterogeneity (non-iid). This highlights the effectiveness of each algorithm in different scenarios.

\textbf{Empirical validation.} 
We conduct experiments to validate our theoretical findings. The results confirm that Multi-Walk converges faster \wrt iterations for graphs with larger diameters, such as cycles. However, this advantage does not hold for topologies with smaller diameters, such as complete graphs.
We also examine the impact of non-iid data in an Erdős–Rényi topology, observing behavior consistent with our theoretical analysis.
To highlight the benefits of Multi-Walk over \agos in communication-constrained settings, we conducted experiments measuring convergence rates \wrt total transmitted bits during the fine-tuning of OPT-$125$M \citep{OPT} as a large language model. Overall, the experiments provide valuable insights into the performance trade-offs between gossip and random walk-based decentralized learning algorithms.
\section{Related Work}\label{related}

Decentralized optimization algorithms have been extensively explored in the literature, where nodes in a graph collaborate with their neighbor nodes to solve optimization problems \cite{Tsits1984DCN, Nedi2009DistributedSM, Duchi_2012, yuan2015convergencedecentralizedgradientdescent, Digest}. These algorithms mostly rely on mixing information among nodes, leading to a considerable communication overhead.
Decentralized algorithms based on \gos involve a mixing step where nodes compute their new models by mixing their own and neighbors' models \cite{Xiao2003FastLI,Lian2017CanDA,unified-koloskova20a}. Model updates propagate gradually over the graph due to iterative gossip averaging. However, this is costly in terms of communication as it requires $\mathcal{O} (|\mathcal{E}|)$ data exchange per model update for a graph with an edge set of $\mathcal{E}$, where $|\cdot|$ is the size of a set.

The study of asynchronous optimization has its roots in earlier works such as those by \citet{async1}, with one of the first convergence results for Asynchronous SGD provided by \citet{Tsits1984DCN}.
Many works have focused on asynchronous algorithms in federated learning settings \cite{Agarwal2011DistributedDS, Lian2015AsynchronousPS, Zheng2016AsynchronousSG, feyzmahdavian2023asynchronousiterationsoptimizationnew,Mishchenko2022AsynchronousSB,koloskova2022sharperconvergenceguaranteesasynchronous}.
Going to decentralized setting along with asynchrony, \citet{Assran_2021} addresses asymmetric asynchronous communication (push-sum), but to guarantee convergence, their approach requires all nodes to participate in computations synchronously at each iteration.
\citet{nadiradze2022asynchronousdecentralizedsgdquantized} explores quantized gossip communication; however, their work does not account for delays in communication or computation. A wait-free decentralized algorithm that allows nodes to have different computation speeds is proposed in \cite{bornstein2022swiftrapiddecentralizedfederated}, but it does not consider any communication delays. A framework for communication acceleration on time-varying topologies with local stochastic gradient steps is considered in \citet{accelerate}, but it does not consider computation or communication delays.
%
%
\citet{lian2018asynchronousdecentralizedparallelstochastic} introduces the Asynchronous Decentralized Stochastic Gradient Descent algorithm (AD-PSGD), one of the most prominent asynchronous decentralized learning methods. In this paper, we consider the same algorithm as \agos and further analyze it to understand its relative performance as compared to Multi-Walk.
In particular, \citet{lian2018asynchronousdecentralizedparallelstochastic} derive a convergence rate under the assumption of an upper bound on computation delays, and their result is valid only when the number of iterations exceeds a certain threshold. Our analysis in this paper relaxes these assumptions, leading to a more comprehensive convergence  proof. 
%
%
\citet{unifiedasync} introduces the Asynchronous SGD on Graph (AGRAF SGD) algorithm, which operates with a continuous \texttt{while true} loop for communication among nodes without any assumptions about the frequency and amount of communication. This design makes it theoretically infeasible to quantify the communication overhead.

On the other end of the spectrum of decentralized learning algorithms are random walk-based approaches \cite{Ayache2020PrivateWR}. When there is only a single walk in the graph, the problem closely resembles to data sampling for stochastic gradient descent, \eg \citet{Sun2018OnMC, needell2015stochasticgradientdescentweighted}, and the distinction between synchronous and asynchronous operations becomes irrelevant.
\citet{hendrikx23a} analyzes random walk-based decentralized learning under the assumption of strong duality, which holds only in convex settings. 
Their analysis also relies on the computation of full gradients, which is a strong and often impractical assumption in many real-world applications.
Moreover, their algorithm needs synchronization and centralized coordination across the network. As compared to this line of work, we design, for the first time in the literature, a random walk-based
learning algorithm with multiple streams, Multi-Walk, where multiple walks operate on a graph in an asynchronous manner.

\section{Setup and Algorithm Design}
We model the underlying network topology with a connected graph $G=(\mathcal{V},\mathcal{E})$, where $\mathcal{V}$ is the set of vertices (nodes) and $\mathcal{E}$ is the set edges.
The vertex set contains $V$ nodes, \ie $ |\mathcal{V}|= V$. 
If node $i$ is connected to node $j$ through a communication link, then $\{i, j\}$ is in the edge set, \ie $\{i, j\} \in \mathcal{E}$.
The set of the nodes that node $i$ is connected to and can transmit data is called the neighbors of node $i$, and the neighbor set of node $i$ is denoted by $\mathcal{N}_i$.

Assume that the nodes in the network jointly minimize a $d$-dimensional function  $f: \mathbb{R}^d \rightarrow \mathbb{R}$.
The goal is to solve optimization problems where the elements of the objective function (i.e., the data used in machine learning tasks) are distributed across different nodes,
\begin{align}
    \min_{\vec{x} \in 	\mathbb{R}^d}\left[f(\vec{x}) := \frac{1}{V} \sum_{v \in \mathcal{V}} \left[f_v(\vec{x}) = \EX_{\xi \sim  \mathcal{D}_v} \F[v][][][] \right]\right]\label{e1}.
\end{align}
$\F[v][][][]: \mathbb{R}^d \rightarrow \mathbb{R}$ is the loss function of $\vec{x}$ associated with data sample $\xi$ at node $v$.
The loss function on local dataset $\data{v}$ at node $v$ is $f_v(\vec{x})$.
We provide a table of notations in Appendix \ref{appendixa}.

\subsection{Asynchronous Multi-Walk Algorithm}

This section presents our novel Multi-Walk  algorithm, which considers the standard asynchronous SGD for model updates. To achieve consensus, communication is performed using multiple walks. Multi-Walk algorithm is summarized in Algorithm \ref{alg:MW}, and detailed in the following. 

\begin{algorithm}[t!]
\caption{Asynchronous Multi-Walk with $R$ walks}\label{alg:MW} 
\begin{algorithmic}[1]
\State Start walk $r$ at node $r-1$, which sets $\xs[r][0] = \xz$, where $r \in \{1,\dots, R\}$.
\State \master initializes $\{u^{r}\}_{r \in \{1,\dots, R\}}$ with $\xz$. 
\State Set $l=1$, which is the last walk that visited \master. 
\For{$t=0$ to $T-1$}
\If {Node $v_t$ finishes the calculation of $\nabla \F[v_t][r_t][t-\tau_t][t]$ at point $\xs[r_t][t-\tau_t]$, which was transmitted to node $v_t$ by one of its neighbors
via walk $r_t$}
iteration $t$ is started and node $v_t$ executes lines 6-12. 
%
\State $\xs[r_t][t+1] = \xs[r_t][t-\tau_t] - \eta_t \nabla \F[v_t][r_t][t-\tau_t][t]$
\If{$v_t = 0$}
    \State $\xs[r_t][t+1] = u^{l} + \frac{1}{R}(\xs[r_t][t+1] - u^{r_t})$. 
    \State $u^{r_t} = \xs[r_t][t+1]$. 
    \State $l = r_t$. 
\EndIf
\State Choose the next node based on matrix $\vec{P}$.
\State Send $\xs[r_t][t+1]$ to the next node via walk $r_t$. 
\EndIf
\EndFor
\end{algorithmic}
\end{algorithm}

First, we assume that there are $R$ walks over the graph, where $R \leq V$.
Without loss of generality, we initialize walk $r$ at node $r-1$ by setting $\xs[r][0] = \xz$, where $r \in \{1,\dots, R\}$. $\xz$ is the global initial model. These nodes start computing the stochastic gradient at $\xz$ using their local data. 
In order to mix the information among walks, we consider a dedicated node that we assume to be \master without loss of generality.\footnote{We note that \master may become unavailable or fail due to underlying network conditions. This issue is addressed in Section \ref{sec:Node0Fail}.} 
We also define $\{u^{r}\}_{r \in \{1,\dots, R\}}$, where $u^{r}$ is a copy of walk~$r$'s model at the most recent instance when that walk was at \master. 
At \master, we initialize $\{u^{r}\}_{r \in \{1,\dots, R\}}$ with $\xz$ that will be used in the mixing. Assume that $l$ is the last walk that visited node \master, and $l$ is initialized with $1$, \ie $l=1$. 
Throughout the algorithm, each node receiving a model via a walk computes its gradient at its own pace, using its local data and the received model.
On line 5, once a node (denoted as $v_t$) completes computing the gradient using the model received via walk $r_t$, iteration $t$ is started.
This model was last updated at iteration $t - \tau_t$ by a neighbor of $v_t$, or corresponds to the initial model $\xz$.
We note that only one gradient computation completion event happens in each iteration. On line 6, node $v_t$ incorporates the computed gradient to update the model using the step size \(\eta_t\).  Note that communicating the model of walk $r_t$ to node $v_t$ and computing the gradient takes \(\tau_t\) iterations. 
Now, if \(v_t\) is \master, we need to mix the current walk, \(r_t\), with other walks.
This is done in lines 8--10. On line 8, we incorporate the newly introduced updates of walk \(r_t\), i.e., \((\xs[r_t][t+1] - u^{r_t})\), which have not been mixed before, into the latest model (\(u^{l}\)) with a weight of \(\frac{1}{R}\).
We update the last applied model of walk \(r_t\) (\(u^{r_t}\)) and the latest walk (\(l\)) on lines 9 and 10.
Finally, node $v_t$ chooses the next node based on the transition matrix \(\vec{P}\) and sends the model.
We note that $\vec{P}$ is the transition matrix of a Markov chain, representing each walk, where $p_{ij}$ in row $i$ and column $j$ of $\vec{P}$ denotes the probability of choosing the next node as $j$ given that the current node is $i$.
Figure \ref{fig:mw} depicts the operation of the Multi-Walk algorithm on a $3$-node network employing two parallel walks ($R=2$). Furthermore, the diagram visualizes the sequence of iterations over real time.

\begin{figure*}[bt]
\centering
\scalebox{0.83}{
\begin{tikzpicture}[scale=1.2,>=stealth]


  \shade[left color=gray!10, right color=gray!10] (0,0) rectangle (13,3.25);

    \shade[left color=gray!20, right color=gray!20] (0,.5) rectangle (13,1);

    \shade[left color=gray!20, right color=gray!20] (0,1.25) rectangle (13,1.75);

    \shade[left color=gray!20, right color=gray!20] (0,2) rectangle (13,2.5);


  \draw[ultra thick,->] (0,0) node[anchor=east]{\textbf{Time}} -- (13,0);


  \node[anchor=east] at (0,2.25) {\textbf{Node 0}};

  \node[anchor=east] at (0,1.5) {\textbf{Node 1}};

  \node[anchor=east] at (0,.75) {\textbf{Node 2}};


  \foreach \x/\t in {3/0, 4/1, 6/2, 8.5/3, 11/4, 12/5} {

      \draw[thick] (\x,-0.2) -- (\x,0.2);

      \node[below] at (\x,-0.2) {\small \textbf{$t=\t$}};

  }



  \draw[fill=NavyBlue!80,rounded corners=3pt] (0,2.25) rectangle (4,2.5);

  \draw[fill=NavyBlue!30,rounded corners=3pt] (4,2.25) rectangle (6.5,2.5);

  \draw[fill=teal!80,rounded corners=3pt] (8,2) rectangle (11,2.25);

  \draw[fill=teal!30,rounded corners=3pt] (11,2) rectangle (13,2.25);

  \draw[fill=teal!30] (12.5,2) rectangle (13,2.25);

  \draw[teal!30, line width=2pt] (12.5,2.005) -- (12.5,2.244);

  \draw[teal!30, line width=2pt] (12.98,2.005) -- (12.98,2.244);


  \draw[fill=NavyBlue!80,rounded corners=3pt] (6.5,1.5) rectangle (8.5,1.75);

  \draw[fill=NavyBlue!30,rounded corners=3pt] (8.5,1.5) rectangle (9.5,1.75);

  \draw[fill=teal!80,rounded corners=3pt] (5,1.25) rectangle (6,1.5);

  \draw[fill=teal!30,rounded corners=3pt] (6,1.25) rectangle (8,1.5);


  \draw[fill=NavyBlue!80,rounded corners=3pt] (9.5,0.75) rectangle (12,1);

  \draw[fill=NavyBlue!30,rounded corners=3pt] (12,0.75) rectangle (12.75,1);

  \draw[fill=NavyBlue!30] (12.5,0.75) rectangle (13,1);

  \draw[NavyBlue!30, line width=2pt] (12.5,.755) -- (12.5,.994);

  \draw[NavyBlue!30, line width=2pt] (12.98,.755) -- (12.98,.994);

  \draw[fill=teal!80,rounded corners=3pt] (0,0.5) rectangle (3,0.75);

  \draw[fill=teal!30,rounded corners=3pt] (3,0.5) rectangle (5,0.75);


  \foreach \x in {3,4,5,6,6.5,8,8.5,9.5,11,12} {

      \draw[dashed,gray!70] (\x,0) -- (\x,3.25);

  }


\node[draw,fill=white,rounded corners=3pt,anchor=north west,inner sep=2pt] at (0,3.25) {

    \begin{tabular}{@{}c@{\hskip 2pt}c@{\hskip 7pt}c@{\hskip 2pt}c@{\hskip 7pt}c@{\hskip 2pt}c@{\hskip 7pt}c@{\hskip 2pt}c@{}}

    \tikz \draw[fill=NavyBlue!80,rounded corners=2pt] (0,0) rectangle (0.5,0.25); & \small Computation (Walk 1) & 

    \tikz \draw[fill=NavyBlue!30,rounded corners=2pt] (0,0) rectangle (0.5,0.25); & \small Communication (Walk 1) & 

    \tikz \draw[fill=teal!80,rounded corners=2pt] (0,0) rectangle (0.5,0.25); & \small Computation (Walk 2) & 

    \tikz \draw[fill=teal!30,rounded corners=2pt] (0,0) rectangle (0.5,0.25); & \small Communication (Walk 2) \\

    \end{tabular}

};

\end{tikzpicture}

}
\caption{
Example of Multi-Walk in a $3$-node network with $2$ walks ($R=2$), where $t$ represents the iteration number.}
\label{fig:mw}
\end{figure*}

\subsection{\agos Algorithm}

\begin{algorithm}[t!]
\caption{Asynchronous \gos (AD-PSGD)}\label{alg:AsyncGoss} 
\begin{algorithmic}[1]
\State Initialize local models $\xv = \xz$ in all nodes. All nodes start computing the stochastic gradient.
\For{$t=0$ to $T-1$}
\State Node $v_t$ is randomly sampled from all nodes. 
\If {Node $v_t$ finishes computing the gradient at point $\xs[v_t][t-\tau_t]$, \ie $\nabla\F[v_t][v_t][t-\tau_t][t]$} iteration $t$ is triggered. Node $v_t$ executes lines 5-7.
\State $\xs[v_t][t+\frac{1}{2}] = \xs[v_t][t] - \eta_t \nabla \F[v_t][v_t][t-\tau_t][t]$.
\State $\xs[v][t+1] = \sum_{i \in \mathcal{N}_v} p_{vi} \xs[i][t+\frac{1}{2}]$ (gossip averaging for all $v \in \mathcal{V}$ based on mixing matrix $\vec{P}$)
\State Start computing gradient at point $\xs[v_t][t+1]$.
\EndIf
\EndFor
\end{algorithmic}
\end{algorithm}

This section presents the \agos algorithm based on \citet{lian2018asynchronousdecentralizedparallelstochastic}.\footnote{We note that we include the description of \agos in this section for completeness as we will provide its comprehensive convergence analysis in the next section. We also note that \agos is named as Asynchronous Decentralized Stochastic Gradient Descent (AD-PSGD) in \citet{lian2018asynchronousdecentralizedparallelstochastic}. We will use \agos and AD-PSGD interchangeably in the rest of the paper.} During the course of the algorithm, all nodes are engaged in gradient computations. At iteration \(t\), node $v_t$ is selected randomly among all the nodes. When node $v_t$  finishes computing the gradient at point $\xs[v_t][t-\tau_t]$, \ie $\nabla\F[v_t][v_t][t-\tau_t][t]$, iteration $t$ is triggered (line 4). 
The gradient is computed with a delay \(\tau_t\) and subsequently applied to the current model of node \(v_t\), i.e., \(\xs[v_t][t]\), using learning rate \(\eta_t\) (line 5). At the end of each iteration, a gossip averaging step is performed based on mixing matrix \(\vec{P}\) (line 6), where $p_{ij}$, which is the element of \(\vec{P}\), is the weight of node $j$'s model in the weighted averaging used to find node $i$'s new model. After gossip averaging is finished, node $v_t$ starts computing gradient at point $\xs[v_t][t+1]$ (line 7).

\section{Convergence Analysis}
We use the following standard assumptions in our analysis.
\begin{enumerate}[leftmargin=*]
\item \textbf{Smooth local loss.} \label{as1}
$f_v(\vec{x})$ is differentiable and its gradient is $L$-Lipschitz for $v \in \mathcal{V}$, \ie $\normm{\nabla f_v(\vec{y})-\nabla f_v(\vec{x})} \leq~L \normm{\vec{y}-\vec{x}}, \quad \forall \vec{x},\vec{y} \in \mathbb{R}^d$.

\item \textbf{Bounded local variance.} \label{as2}
The variance of the stochastic gradient is bounded for $v \in \mathcal{V}$, \ie  $ \EX_{\xi \sim \mathcal{D}_v} \norm{\nabla \F[v][][][] - \nabla f_v(\vec{x})}\leq \sigma^2$.

\item \textbf{Bounded diversity.} \label{as3}
The diversity of the local loss functions are bounded for $v \in \mathcal{V}$, \ie  $ \norm{\nabla f_v(\vec{x}) - \nabla f(\vec{x})}\leq~\zeta^2$.

\item \textbf{Transition (mixing) matrix.}\label{as4}
In Algorithm \ref{alg:MW}, \(\vec{P}\) is the transition matrix of an irreducible and aperiodic Markov chain, representing each walk.
In Algorithm \ref{alg:AsyncGoss}, it defines the mixing step of the gossip averaging. Matrix \(\vec{P}\) is doubly stochastic (\(\vec{P} \vec{1} = \vec{1}\),  \(\vec{1}^\top \vec{P} = \vec{1}^\top\)) and the spectral gaps of \(\vec{P}^\top \vec{P}\) and \(\vec{P}\) are denoted by \(p\) and \(p'\), respectively.

\end{enumerate}

\subsection{Convergence rate \wrt iterations} \label{conv}

\begin{theorem}\label{T1} \textbf{Multi-Walk.}
Let assumptions \ref{as1}-\ref{as4} hold, with a constant and small enough learning rate $\eta$ (potentially depending on $T$), after $T$ iterations of Algorithm \ref{alg:MW}, $\frac{1}{T}\sum_{t=0}^{T-1}\EX\norm{\nabla f(\xs)}$ is
\begin{small}
\begin{align}
       \mathcal{O} \bigg( \frac{FLRH}{T}    +    \frac{R\zeta^2}{p'T}   +   \sqrt{\frac{  FL  \left(\sigma^2    +    \zeta^2\right)}{T}}   +   (\hspace{-1pt}\frac{FLR\sqrt{V\sigma^2 + H^2\zeta^2}}{T}  )\hspace{-1pt}^{\frac{2}{3}}    \bigg),\label{cove_rate_MW} 
\end{align}
\end{small}
where $F := f(\xz)-f^*$, and $H^2$ is the second moment of the first return time to \master for the Markov chain representing each walk.\footnote{Specifically,
$H^2 = \mathbb{E}[h^2]$,
where
$h = \min\{\, k \ge 1 : X_k = 0 \mid X_0 = 0\}$
represents the number of steps it takes for the Markov chain representing each walk, starting from \master ($X_0 = 0$), to return to \master for the first time.}
\hfill $\Box$
\end{theorem}
\begin{proof}
    Please refer to Appendix \ref{appendixb}.
\end{proof}

\begin{theorem}\label{T2} \textbf{\agos.}
Let assumptions \ref{as1}-\ref{as4} hold, with a constant and small enough learning rate $\eta$ (potentially depending on $T$), after $T$ iterations of Algorithm \ref{alg:AsyncGoss}, $\frac{1}{T}\sum_{t=0}^{T-1}\EX\norm{\nabla \gf[]}$ is
\begin{small}
\begin{align}
       & \mathcal{O} \bigg(\frac{FLV}{pT} + \sqrt{\frac{FL\left(\sigma^2 + \zeta^2\right)}{T}} + (\frac{FLV\sqrt{\frac{\sigma^2}{p} + \frac{\zeta^2}{p^2}}}{T})^{\frac{2}{3}} \bigg), \label{cove_rate_gos} 
\end{align}
\end{small}
where $F := f(\xz)-f^*$. \hfill $\Box$
\end{theorem}
\begin{proof}
    Please refer to Appendix \ref{appendixc}.
\end{proof}

\textbf{Dominant terms.}
The dominant term in both (\ref{cove_rate_MW}) and (\ref{cove_rate_gos}) is identically given by $\sqrt{\frac{FL\left(\sigma^2 + \zeta^2\right)}{T}}$.
Focusing on the next most significant term for comparison, in (\ref{cove_rate_MW}), this term is given by $(\frac{FLR\sqrt{V\sigma^2 + H^2\zeta^2}}{T})^{\frac{2}{3}}$, whereas in (\ref{cove_rate_gos}), it is $(\frac{FLV\sqrt{\frac{\sigma^2}{p} + \frac{\zeta^2}{p^2}}}{T})^{\frac{2}{3}}$.
Note that (\ref{cove_rate_MW}) includes a non-dominating term that describes the rate at which walks converge to their steady state. This term is related to the spectral gap of $\vec{P}$, represented by $p'$. 
In the following, we compare the dominant terms in the convergence rates of both algorithms in different settings.

\textbf{Homogeneous data distribution.} 
In iid setting ($\zeta = 0$), the differentiating factor in the second dominant term of convergence rate is $\frac{V}{\sqrt{p}}$ for \agos and $R\sqrt{V}$ for Multi-Walk. 
Specifically, for graphs with $p = \mathcal{O}(\frac{V}{R^2})$, Multi-Walk outperforms, while for $p = \Omega(\frac{V}{R^2})$, \agos converges faster \wrt iterations.
It is interesting to observe that the graph's topology does not impact the performance of Multi-Walk in iid setting, and the only factors are the number of nodes and walks.
We compare convergence rate and communication overhead for each algorithm in Table~\ref{iid_table} across three different graph topologies, using the commonly employed Metropolis-Hastings matrix, $\vec{P}$, where 
\(
p_{ij} = p_{ji} = \min \left\{ \frac{1}{\operatorname{deg}(i) + 1}, \frac{1}{\operatorname{deg}(j) + 1} \right\}, \quad \text{for } \{i,j\} \in \mathcal{E}.
\)
Note that computation overhead is the same for both and equal to the number of iterations, \ie $T$, and we do not include that in the table.
We observe that for both cycle and $2$D-torus topologies, Multi-Walk outperforms \agos in convergence rate.
However, when the graph diameter decreases (i.e., $p$ increases), such as in the case of a complete graph, Multi-Walk loses its advantage.
It is important to note that Multi-Walk consistently maintains lower communication overhead; in each iteration, it involves at most one communication step, whereas \agos activates multiple edges for mixing based on the graph topology

\begin{table*}[t]
\caption{Comparison of the convergence rate and communication overhead in \textbf{iid} setting for Metropolis-Hastings $\vec{P}$.}
\label{iid_table}
\begin{center}
{\fontsize{10.2}{28}\selectfont
\begin{sc}
\begin{tabular}{@{\hskip 1pt}l@{\hskip 1pt}l@{\hskip 4pt}c@{\hskip 0pt}c}
\toprule
\addlinespace[-.1cm]
Topology & Algorithm & Convergence rate & Comm-cost \\
\midrule
\multirow{2}{*}{Cycle ($p = \Theta(\frac{1}{V^2})$)} & Multi-Walk & ${\mathcal{O} \bigg(\frac{\sigma}{\sqrt{T}} + (\frac{R\sqrt{V\sigma^2}}{T})^{\frac{2}{3}} \bigg)}$\cmark & ${\Theta(T)}$ \\
 & \agos & $\mathcal{O} \bigg(\frac{\sigma }{\sqrt{T}} + (\frac{V\sqrt{V^2 \sigma^2}}{T})^{\frac{2}{3}} \bigg)$ & $\Theta(VT)$ \\
\addlinespace[.2cm]
\midrule
\multirow{2}{*}{2d-torus ($p = \Theta(\frac{1}{V})$)} & Multi-Walk & ${\mathcal{O} \bigg(\frac{\sigma}{\sqrt{T}} + (\frac{R\sqrt{V\sigma^2}}{T})^{\frac{2}{3}} \bigg)}$\cmark & ${\Theta(T)}$ \\
 & \agos & ${\mathcal{O} \bigg( \frac{\sigma}{\sqrt{T}} + (\frac{V\sqrt{V\sigma^2 }}{T})^{\frac{2}{3}} \bigg)}$ & $\Theta(VT)$ \\
 \addlinespace[.2cm]
\midrule
\multirow{2}{*}{Complete ($p = 1$)} & Multi-Walk & $\mathcal{O} \bigg( \frac{\sigma}{\sqrt{T}} + (\frac{R\sqrt{V\sigma^2}}{T})^{\frac{2}{3}} \bigg)$[\cmark \textit{if} \( R = \mathcal{O}( \sqrt{V} )\)] & ${\Theta(T)}$ \\
 & \agos & ${\mathcal{O} \bigg( \frac{\sigma}{\sqrt{T}} + (\frac{V\sqrt{\sigma^2 }}{T})^{\frac{2}{3}} \bigg)}$[\cmark \textit{if} \( R = \mathbf{\varOmega}( \sqrt{V} )\)] & $\Theta(V^2T)$ \\
 \addlinespace[.3cm]
\bottomrule
\end{tabular}
\end{sc}
}
\end{center}
\end{table*}

\begin{table*}[t]
\caption{Comparison of the convergence rate and communication overhead in \textbf{noniid} setting for Metropolis-Hastings $\vec{P}$.}
\label{noniid_table}
\begin{center}
{\fontsize{10.2}{28}\selectfont
\begin{sc}
\begin{tabular}{llcc}
\toprule
\addlinespace[-.1cm]
Topology & Algorithm & Convergence rate & Comm-cost \\
 \addlinespace[.2cm]
\midrule
\multirow{2}{*}{Cycle ($p = \Theta(\frac{1}{V^2})$)} & Multi-Walk & ${\mathcal{O} \bigg(\sqrt{\frac{\sigma^2 + \zeta^2}{T}} + (\frac{R\sqrt{V\sigma^2 + V^3\zeta^2}}{T})^{\frac{2}{3}} \bigg)}$\cmark & ${\Theta(T)}$ \\
 & \agos & $\mathcal{O} \bigg(\sqrt{\frac{\sigma^2 + \zeta^2}{T}} + (\frac{V\sqrt{V^2\sigma^2 + V^4 \zeta^2 }}{T})^{\frac{2}{3}} \bigg)$ & $\Theta(VT)$ \\
  \addlinespace[.2cm]
\midrule
\multirow{2}{*}{2d-torus ($p = \Theta(\frac{1}{V})$)} & Multi-Walk & ${\mathcal{O} \bigg(\sqrt{\frac{\sigma^2 + \zeta^2}{T}} + (\frac{R\sqrt{V\sigma^2 + H^2\zeta^2}}{T})^{\frac{2}{3}} \bigg)}$ & ${\Theta(T)}$ \\
 & \agos & ${\mathcal{O} \bigg( \sqrt{\frac{\sigma^2 + \zeta^2}{T}} + (\frac{V\sqrt{V\sigma^2 + V^2 \zeta^2 }}{T})^{\frac{2}{3}} \bigg)}$ & $\Theta(VT)$ \\
\midrule
\multirow{2}{*}{Complete ($p = 1$)} & Multi-Walk & ${\mathcal{O} \bigg(\sqrt{\frac{\sigma^2 + \zeta^2}{T}} + (\frac{R\sqrt{V\sigma^2 + V^2\zeta^2}}{T})^{\frac{2}{3}} \bigg)}$ & ${\Theta(T)}$ \\
 & \agos & ${\mathcal{O} \bigg( \sqrt{\frac{\sigma^2 + \zeta^2}{T}} + (\frac{V\sqrt{\sigma^2 + \zeta^2 }}{T})^{\frac{2}{3}} \bigg)}$ & $\Theta(V^2T)$ \\
 \addlinespace[.3cm]
\bottomrule
\end{tabular}
\end{sc}
}
\end{center}
\end{table*}

\textbf{Heterogeneous data distribution.}
In non-iid setting, $\zeta^2$ is multiplied by $H^2$ for Multi-Walk and by $p^2$ for \agos.
We derived $H^2$ for cycle and complete topologies with Metropolis-Hastings transition matrix in Appendix~\ref{appendixe}, and the comparison is summarized in Table~\ref{noniid_table}. 
We observe that for the cycle topology, Multi-Walk consistently demonstrates faster convergence in terms of iterations. However, this advantage diminishes as we move to topologies with smaller diameters.
In complete topology, we observe that \(\zeta^2\) is multiplied by \(V^2\) in Multi-Walk, whereas it is multiplied by \(1\) in \agos.
This indicates that, as we transition to increasingly non-iid settings in small-diameter topologies, Multi-Walk perform poorly. 


\subsection{Convergence rate \wrt transmitted bits}
\label{conv-bits}

\begin{table*}[t] 
\bgroup
\def\arraystretch{1.1}
\caption{Analysis in total transmitted bits ($B$).}
\label{bits}
\begin{center}
{\fontsize{10.2}{28}\selectfont
\begin{sc}
\begin{tabular}{@{\hskip 1pt}lcc}
\toprule
\addlinespace[-.1cm]
Algorithm & Convergence rate & Comp-cost\\
\midrule
Multi-Walk & $\mathcal{O} \bigg(\frac{FLRHm}{B} + \frac{R\zeta^2m}{p'B}+ \sqrt{\frac{FLm\left(\sigma^2 + \zeta^2\right)}{B}} + (\frac{FLRm\sqrt{V\sigma^2 + H^2\zeta^2}}{B})^{\frac{2}{3}} \bigg)$  & $\Theta(\frac{B}{m})$ \\
\agos & $\mathcal{O} \bigg(\frac{FLVm\|\vec{P}\|_0}{pB} + \sqrt{\frac{FLm\|\vec{P}\|_0\left(\sigma^2 + \zeta^2\right)}{B}} +(\frac{FLVm\|\vec{P}\|_0\sqrt{\frac{\sigma^2}{p} + \frac{\zeta^2}{p^2}}}{B})^{\frac{2}{3}} \bigg)$ & $\Theta(\frac{B}{m\|\vec{P}\|_0})$ \\
\addlinespace[.3cm]
\bottomrule
\end{tabular} 
\end{sc}
}
\end{center}
\egroup
\end{table*}

Assume the model size is $m$ bits. Each iteration of Algorithm \ref{alg:MW} and \ref{alg:AsyncGoss} communicates one and $\|\vec{P}\|_0$ models, respectively. $\|\vec{P}\|_0$ denote the number of non-zero elements of mixing matrix $\vec{P}$.
\begin{corollary}
    Under the condition of Theorem \ref{T1}, \ref{T2}, we get
    the convergence rate of Algorithm \ref{alg:MW}, and \ref{alg:AsyncGoss} as shown in Table \ref{bits} where $B$ represents total transmitted bits.
\end{corollary}

The dominating term here is $\sqrt{\frac{FLm\left(\sigma^2 + \zeta^2\right)}{B}}$ for Multi-Walk and 
$\sqrt{\frac{FLm\|\vec{P}\|_0\left(\sigma^2 + \zeta^2\right)}{B}}$ for \agos. Thus, we observe that Multi-Walk outperforms \agos in terms of transmitted bits when the second dominating term is not comparable in magnitude. Intuitively, in every model transmission, Multi-Walk executes approximately one computation per model transmission, while \agos performs $\|\vec{P}\|_0$ model transmissions per computation. Therefore, Multi-Walk is a better choice when there is a restriction on the amount of communicated bits.

In the extreme noniid regime (large $\zeta$), the second dominant term is proportional to $H^2\zeta^2$ for Multi-Walk, compared to $\frac{\zeta^2}{p^2}$ for \agos. This suggests that \agos is advantageous in highly non-i.i.d. settings with small graph diameters. In the extreme non-i.i.d. case, the second term becomes comparable to the leading term, and in small-diameter graphs, this term specifically favors \agos. Taking a complete graph as an example, the terms simplify to $V^2\zeta^2$ for Multi-Walk versus $\zeta^2$ for \agos, a difference that significantly favors \agos in such settings.




\subsection{Convergence rate \wrt wall-clock time}
\label{conv-time}

In Algorithm \ref{alg:MW}, assume each walk performs one iteration (computation and communication) with a rate-$\frac{1}{d}$ exponential random variable, independent across walks and over time.
The value of $d$ is determined by the average computation and communication delay in the network.
Thus, each walk does one iteration in Algorithm \ref{alg:MW} according to a rate-$\frac{1}{d}$ Poisson process. Equivalently, this corresponds to all iterations in Algorithm \ref{alg:MW} are according to a rate-$\frac{R}{d}$ Poisson process at times $\{Z_t\}_{t=0}^{T-1}$ where $\{Z_t - Z_{t-1}\}_{t=1}^{T-1}$, denoting the $t$-th iteration duration,
are i.i.d. exponentials of rate $\frac{R}{d}$. Therefore, we have $\EX\left[Z_t\right] = \frac{td}{R}$ and for any $\delta >0 $:
\begin{align}
    Pr \left ( |Z_t - \frac{td}{R}| \geq \frac{\delta td}{R} \right) \leq 2 \exp\left( \frac{-\delta^2 t}{2}\right).\label{realtime}
\end{align}
(\ref{realtime}) follows directly from Cramer’s theorem \cite{Boyd2006RandomizedGA}.
Hence, by multiplying the terms obtained regarding iterations by \(\frac{d}{R}\), we obtain the corresponding terms in real time. In other words, the convergence rate in Theorem \ref{T1} can be transformed to real time (\(Z\)) by substituting \(T\) with \(\frac{RZ}{d}\).

For Algorithm \ref{alg:AsyncGoss}, we assume each node has a clock that ticks at the times of a rate-$\frac{1}{d}$ Poisson process.
Here, the value of $d$ is determined by the average computation and gossip communication delay for nodes.
And the same result of (\ref{realtime}) is valid by replacing $R$ with $V$.

\begin{table*}[t] 
\bgroup
\caption{Analysis in wall-clock time ($Z$).}
\label{real_time}
\begin{center}
{\fontsize{10.2}{28}\selectfont
\begin{sc}
\begin{tabular}{@{\hskip 0pt}l@{\hskip -10pt}c@{\hskip 0pt}c@{\hskip 5pt}c}
\toprule
\addlinespace[-.1cm]
Algorithm & Convergence rate & Comm-cost & Comp-cost\\
\midrule
Multi-Walk & $\mathcal{O} \bigg(\frac{FLHd}{Z} +\frac{\zeta^2d}{p'Z} + \sqrt{\frac{FLd\left(\sigma^2 + \zeta^2\right)}{RZ}} + (\frac{FLd\sqrt{\sigma^2V + \zeta^2H^2}}{Z})^{\frac{2}{3}} \bigg)$  & ${\Theta(\frac{ZRm}{d})}$   & $\Theta(\frac{ZR}{d})$ \\
\agos & $\mathcal{O} \bigg(\frac{FLd}{pZ} + \sqrt{\frac{FLd\left(\sigma^2 + \zeta^2\right)}{VZ}} + (\frac{FLd\sqrt{\frac{\sigma^2}{p} + \frac{\zeta^2}{p^2}}}{Z})^{\frac{2}{3}} \bigg)$ & $\Theta(\frac{ZVm\|\vec{P}\|_0}{d})$  & $\Theta(\frac{ZV}{d})$\\
\addlinespace[.3cm]
\bottomrule
\end{tabular} 
\end{sc}
}
\end{center}
\egroup
\end{table*}

\begin{corollary}
    Under the condition of Theorem \ref{T1}, \ref{T2}, we get
    the convergence rate of Algorithms \ref{alg:MW} and \ref{alg:AsyncGoss} as shown in Table \ref{real_time} where $Z$ represents wall-clock time.
\end{corollary}

The dominating term here is $\sqrt{\frac{FLd\left(\sigma^2 + \zeta^2\right)}{RZ}}$ for Multi-Walk and 
$\sqrt{\frac{FLd\left(\sigma^2 + \zeta^2\right)}{VZ}}$ for \agos. 
This highlights the advantage of \agos when considering real-time performance. The reason is that all nodes operate simultaneously, enabling multiple iterations to be completed in a shorter period of time in terms of wall-clock duration.
We also observe that Multi-Walk achieves a linear speed-up proportional to the number of walks, making it competitive with \agos \wrt wall-clock time.
Increasing the number of walks reduces the impact of the dominant term. If we consider the second dominant term, given by  
$(\frac{FLd\sqrt{\sigma^2V + \zeta^2H^2}}{Z})^{\frac{2}{3}}$ for Multi-Walk and 
$(\frac{FLd\sqrt{\sigma^2/p + \zeta^2/p^2}}{Z})^{\frac{2}{3}}$ for \agos, we observe that this term favors Multi-Walk for topologies with large diameters.
Here, we also observe that the computation and communication cost of \agos is higher than that of Multi-Walk in real time.
The communication overhead for \agos is proportional to \( V \|\vec{P}\|_0 \) because all nodes are active, and gossip is used for information dissemination. In contrast, for Multi-Walk, it is proportional to \( R \), as there are \( R \) active walks, each performing one peer-to-peer communication. The computation overhead is proportional to \( R \) and \( V \) in Multi-Walk and \agos, respectively, as Multi-Walk and \agos have  \( R \) and \( V \) concurrent active nodes calculating gradients.




\section{\label{sec:Node0Fail} Resilience of Multi-Walk against Node Failures }\label{fail_resilience}

Any node in the graph may fail. If a node other than Node 0 fails in Multi-Walk, Algorithm~\ref{alg:MW}, two cases arise. (i) If no walk is currently associated with the failed node, the algorithm continues to operate unchanged. (ii) If the failed node is participating in a computation or communication phase of a walk, the information associated with that walk is lost. This failure model has been studied in \cite{10901339} for the single random-walk case, and the extension to multiple concurrent walks is straightforward.

On the other hand, if Node 0 fails in Multi-Walk, there may be a natural concern if this causes a single point of failure.
Fortunately, this is not the case. In addition to \master, all streams in the graph maintain a copy of the global model. Even though these streams may hold slightly outdated versions of the global model, each still retains a valid model state. Therefore, if \master fails, any other node can take its place and resume the aggregation process across the streams.

Nodes periodically exchange heartbeat signals to monitor the liveness of \master. If \master becomes unresponsive for a certain period, nodes initiate a local communication protocol with their one-hop neighbors, exchanging messages that include information such as node degree, bandwidth, available memory, and compute capacity. Based on this information, the nodes collaboratively reach a consensus and elect a new \master, selecting the node best suited to take over the aggregation task. 
The newly selected \master waits for a walk to arrive. The local copies $\{u^{r}\}_{r \in \{1,\dots, R\}}$ are then initialized with the model parameters from the arriving walk (similar to line 2 of Algorithm \ref{alg:MW}), after which the Multi-Walk algorithm resumes. 
Regardless of the specific \master selection mechanism, it is evident that failures of the \master can impact convergence time. Next, we analyze the convergence behavior under scenarios where \master failures occur and a new \master is selected after failures.





Let us assume there are $E$ failures throughout the learning process, corresponding to $E+1$ different \master selected until convergence. Let $H^2_i$ be the second moment of the first return time to \master chosen after the $i$-th failure. Assumption \ref{as4} concerning the transition matrix must hold after each failure. Accordingly, we define $p'_i$ as the spectral gap of the transition matrix $\vec{P}$ after the $i$-th failure, which may involve changes to the network topology. For this analysis, we assume that failures do not alter the global loss function and that no streams are lost.

\begin{theorem}\label{T3} \textbf{Multi-Walk with \master failures.}
Let assumptions \ref{as1}-\ref{as4} hold, with a constant and small enough learning rate $\eta$ (potentially depending on $T$), after $T$ iterations of Algorithm \ref{alg:MW} with $E$ \master failures handled by new \master selection, 
$\frac{1}{T}\sum_{t=0}^{T-1}\EX\norm{\nabla f(\xs)}$ is
\begin{align}
       \mathcal{O} \bigg(\frac{FLRE\max\limits_{i \in \{0,...,E\}} H_i}{T}+\sum_{i=0}^{E}\frac{R\zeta^2}{p'_iT} + \sqrt{\frac{FL\left(\sigma^2  +  \zeta^2\right)}{T}}
       +(\frac{FLR\sqrt{EV\sigma^2 + E^2\zeta^2 \max\limits_{i \in \{0,...,E\}} H^2_i}}{T})^{\frac{2}{3}} \bigg),\label{cove_rate_MW_fail}
\end{align}
where $F := f(\xz)-f^*$.
\end{theorem}
\begin{proof}
    Please refer to Appendix \ref{appendixd}.
\end{proof}

\textbf{Impact of failures.}
We see from  (\ref{cove_rate_MW_fail}) that Multi-Walk algorithm still converges to a stationary point despite \master failures. However, these failures negatively impact the convergence rate due to; $E \max\limits_{i \in \{0,...,E\}} H_i$ in the first term, the summation $\sum_{i=0}^{E}(\cdot)$ in the second, and the factor of $E$ and $E^2 \max\limits_{i \in \{0,...,E\}} H^2_i$ in the last. This degradation is expected, as each \master failure leads to some loss of updated information and additional time to recover. Nevertheless, it is noteworthy that convergence is still guaranteed despite these interruptions.



\section{Experiments}

In this section, we validate our theoretical results through empirical experiments, which include the following:
Section \ref{exp-topologu} verifies the impact of network graph topology on the convergence rate.
Section \ref{exp-noniid} explores the impact of data heterogeneity on the convergence rate in two different graph topologies with small and large diameters.
Section \ref{exp-comm} evaluates the communication efficiency of Multi-Walk in bandwidth-constrained environments through an LLM fine-tuning task. 
Finally, Section \ref{exp_res} investigates the impact of \master failure on \ours to verify its resilience. We also compare this with \agos, observing its performance when the same sequence of nodes fails.

We use two machine learning tasks: (i) \emph{Image classification} on CIFAR-$10$ \cite{cifar10} using ResNet-$20$ \cite{He2015resnet}; and (ii) \emph{LLM fine-tuning} of OPT-$125$M \citep{OPT} as a large language model on the Multi-Genre Natural Language Inference (MultiNLI) corpus \citep{multi-nli}.
The details of the image classification and LLM fine-tuning tasks are specified in Table \ref{table:imagenet_settings} and \ref{table:llm_settings}, respectively.

We repeat each experiment $10$ times and present the error bars associated with the randomness of the optimization. In every figure, we include the average and standard deviation error bars. 
In the figures, we use ``MW'' as an abbreviation for \ours.

\begin{table}[hbt]
\centering
\caption{Default experimental settings for the image classification training}
\label{table:imagenet_settings}
\resizebox{\textwidth}{!}{%
\begin{tabular}{ll}
\toprule
\textbf{Dataset} & CIFAR-$10$ \citep{cifar10}, licensed under the MIT License \\
\midrule
\textbf{Architecture} & ResNet-$20$ \cite{He2015resnet}, licensed under the MIT License \\ 
\textbf{Loss function} & cross entropy \\ 
\textbf{Accuracy objective} & top-1 accuracy \\ 
\midrule
\textbf{Number of nodes} & 20 \\ 
\textbf{Topology} & cycle, complete, Erdős–Rényi \\ 
\textbf{Data distribution} & iid (shuffled and split), non-iid (based on labels) \\ 
\textbf{Local Steps $\tau$} & 5 \\
\midrule
\textbf{Optimizer} & SGD with momentum\\
\textbf{Batch size} & 32 per client \\ 
\textbf{Momentum} & 0.9 (Nesterov) \\ 
\textbf{Initial learning rate} & 0.05\\ 
\textbf{Learning rate schedule} & multiplied by $0.1$ once after $75$ and once after $90$ percent of the training\\ 
\textbf{Training time} & $15$ minutes for $\alpha \in \{10,1\}$ and $30$ minutes for $\alpha = 0.1$ \\ 
\textbf{Weight decay} & $10^{-4}$ \\ 
\textbf{Learning rate warm-up time} & $2$ minutes \\ 
\midrule
\textbf{Repetitions} & 10\\ 
\textbf{Reported metric} & Mean and standard deviation (1-sigma error bars) of the aggregated model's \\
& training loss and accuracy, accounting for randomness in network conditions \\
& and algorithmic factors such as random walk-based next node selection \\
& and neighbor selection in gossip-based averaging.\\
\bottomrule
\end{tabular}%
}
\end{table}

\begin{table}[h!]
\centering
\caption{Default experimental settings for the large language model fine-tuning}
\label{table:llm_settings}
\resizebox{\textwidth}{!}{%
\begin{tabular}{ll}
\toprule
\textbf{Dataset} & Multi-Genre Natural Language Inference (MultiNLI) corpus \cite{multi-nli}, \\
& released under the CC BY-SA 4.0 License \\ 
\midrule
\textbf{Architecture} & OPT-$125$M \cite{OPT}, released by Meta AI under the OPT License \\
& (a custom non-commercial research license) \\ 
\textbf{Loss function} & cross entropy \\ 
\midrule
\textbf{Number of nodes} & 20 \\ 
\textbf{Topology} & Erdős–Rényi \\ 
\textbf{Data distribution} & iid (shuffled and split), non-iid (based on genre) \\ 
\textbf{Local Steps $\tau$} & 1 \\
\midrule
\textbf{Optimizer} & Adam \\
\textbf{Batch size} & 16 sentences per client \\ 
\textbf{Adam $\beta_1$} & 0.9 \\
\textbf{Adam $\beta_2$} & 0.999 \\
\textbf{Adam $\epsilon$} & $10^{-8}$ \\
\textbf{Initial learning rate} & $10^{-4}$\\ 
\textbf{Learning rate schedule} & multiplied by $0.1$ once after $75$ and once after $90$ percent of the training\\ 
\textbf{Training time} & $15$ minutes \\ 
\textbf{Weight decay} & $10^{-4}$ \\ 
\textbf{Learning rate warm-up time} & $2$ minutes \\ 
\midrule
\textbf{Repetitions} & 10\\ 
\textbf{Reported metric} & Mean and standard deviation (1-sigma error bars) of the aggregated model's \\
& training loss and accuracy, accounting for randomness in network conditions \\
& and algorithmic factors such as random walk-based next node selection \\
& and neighbor selection in gossip-based averaging.\\ 
\bottomrule
\end{tabular}%
}
\vspace{-10pt}
\end{table}

\begin{figure}[tb]
\centering
\scalebox{0.83}{
\begin{tikzpicture}[
    node/.style={circle, draw, fill=black, inner sep=2pt},
    conn/.style={thick, color=blue!60!black},
    clusterBox/.style={draw=gray!60, fill=gray!10, rounded corners, inner sep=12pt}
]


\node[node] (ca1) at (0,0) {};
\node[node] (ca2) at (0.7,0.6) {};
\node[node] (ca3) at (1.4,0) {};
\node[node] (ca4) at (0.7,-0.6) {};

\node[node] (nw1) at (4,1.5) {};
\node[node] (nw2) at (4.7,2.1) {};
\node[node] (nw3) at (5.4,1.5) {};
\node[node] (nw4) at (4.7,0.9) {};

\node[node] (ia1) at (8,3) {};
\node[node] (ia2) at (8.7,3.6) {};
\node[node] (ia3) at (9.4,3) {};
\node[node] (ia4) at (8.7,2.4) {};

\node[node] (il1) at (11.5,1.5) {};
\node[node] (il2) at (12.2,2.1) {};
\node[node] (il3) at (12.9,1.5) {};
\node[node] (il4) at (12.2,0.9) {};

\node[node] (ks1) at (7,-2) {};
\node[node] (ks2) at (7.7,-1.4) {};
\node[node] (ks3) at (8.4,-2) {};
\node[node] (ks4) at (7.7,-2.6) {};

\begin{pgfonlayer}{background}
    \node[clusterBox, fit=(ca1)(ca2)(ca3)(ca4), label=below:{\textbf{CA}}] {};
    \node[clusterBox, fit=(nw1)(nw2)(nw3)(nw4), label=below:{\textbf{NV}}] {};
    \node[clusterBox, fit=(ia1)(ia2)(ia3)(ia4), label=above:{\textbf{IA}}] {};
    \node[clusterBox, fit=(il1)(il2)(il3)(il4), label=below:{\textbf{IL}}] {};
    \node[clusterBox, fit=(ks1)(ks2)(ks3)(ks4), label=below:{\textbf{KS}}] {};
\end{pgfonlayer}

\draw[conn] (ca1) -- (nw3);
\draw[conn] (ca2) -- (ks4);
\draw[conn] (nw2) -- (ia3);
\draw[conn] (ia2) -- (il1);
\draw[conn] (il3) -- (ks2);
\draw[conn] (ca4) -- (ia1);
\draw[conn] (nw4) -- (il4);
\draw[conn] (ks1) -- (ia4);
\draw[conn] (ks3) -- (nw1);
\draw[conn] (ca3) -- (il2);

\draw[conn] (ca1) -- (ca2);
\draw[conn] (ca3) -- (ca4);

\draw[conn] (nw1) -- (nw4);
\draw[conn] (nw2) -- (nw3);

\draw[conn] (ia1) -- (ia2);
\draw[conn] (ia3) -- (ia4);

\draw[conn] (il1) -- (il4);
\draw[conn] (il2) -- (il3);

\draw[conn] (ks1) -- (ks2);
\draw[conn] (ks3) -- (ks4);

\end{tikzpicture}
}
\caption{A $20$-node decentralized system distributed across five geographic clusters (CA, NV, IA, IL, KS). The links depict the overlay network, which is configurable based on the desired graph topology.}
\label{fig:clustered-network}
\end{figure}

\begin{figure*}[tb]
     \centering
     \begin{subfigure}[b]{0.329\textwidth}
         \centering
         \includegraphics[width=1.025\textwidth]{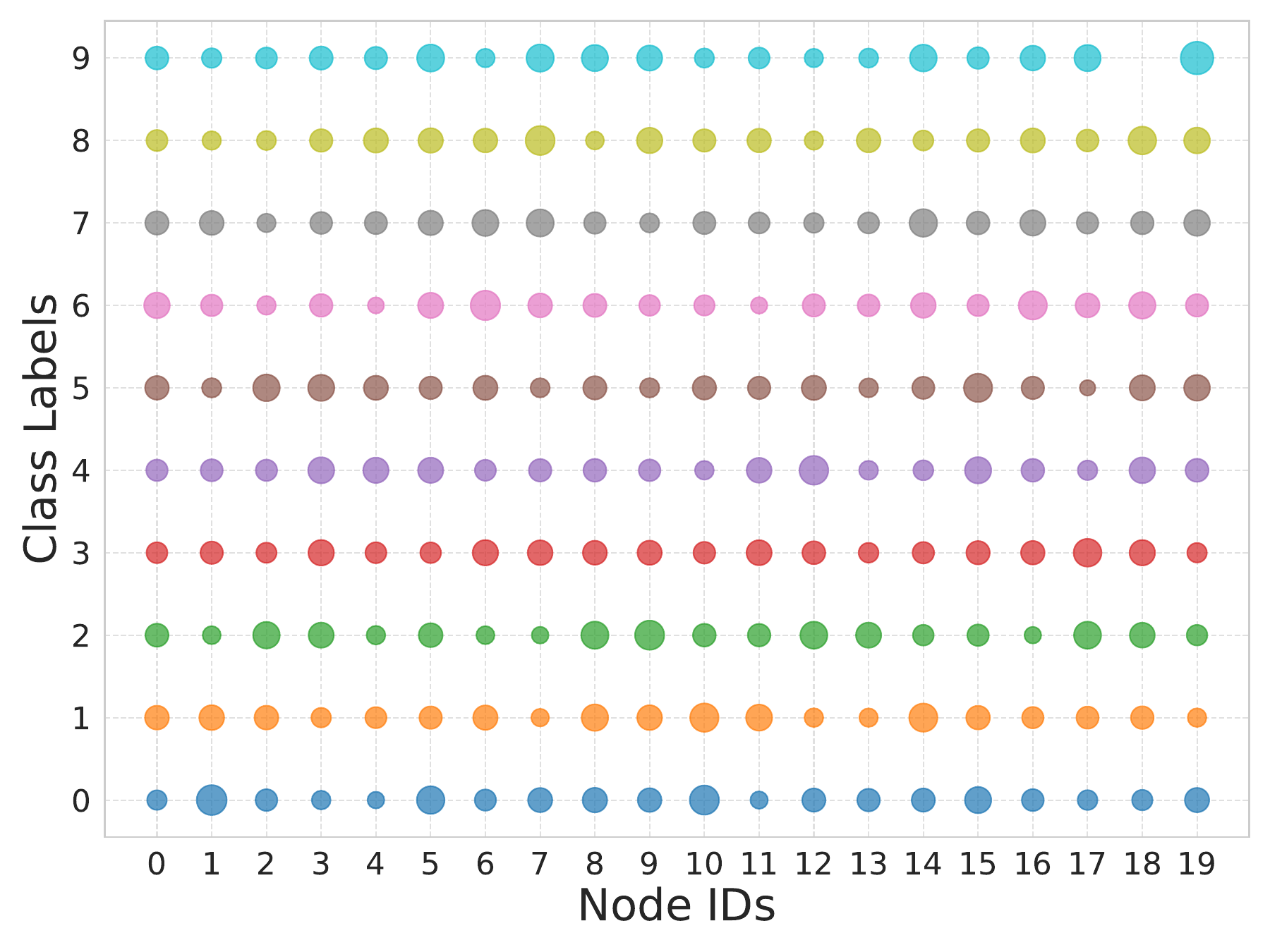}
         \caption{$\alpha = 10$.}
         \label{10}
     \end{subfigure}
     \begin{subfigure}[b]{0.329\textwidth}
         \centering
        \includegraphics[width=1.025\textwidth]{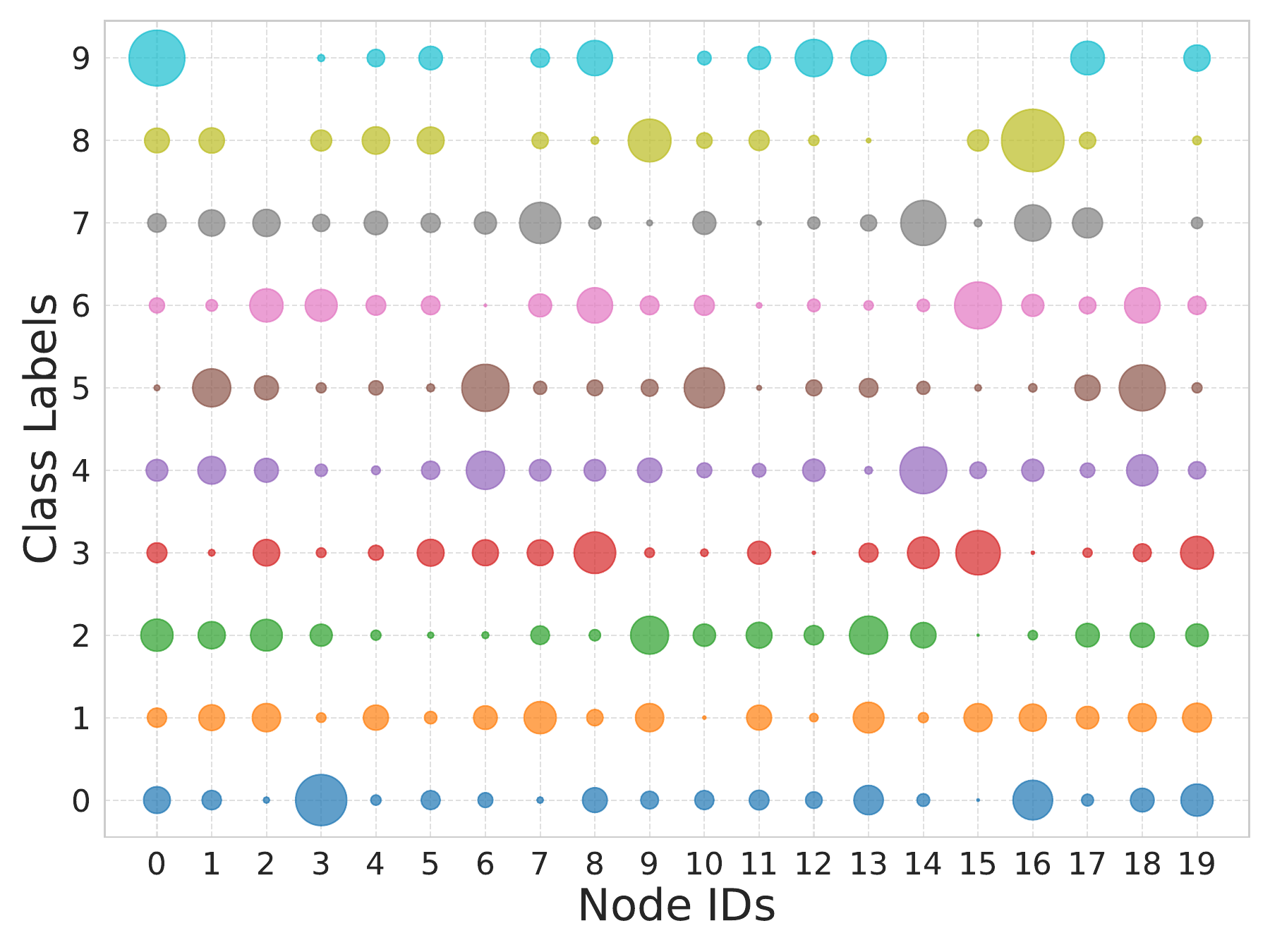}
         \caption{$\alpha = 1$.}
         \label{1}
     \end{subfigure}
     \begin{subfigure}[b]{0.329\textwidth}
         \centering
        \includegraphics[width=1.025\textwidth]{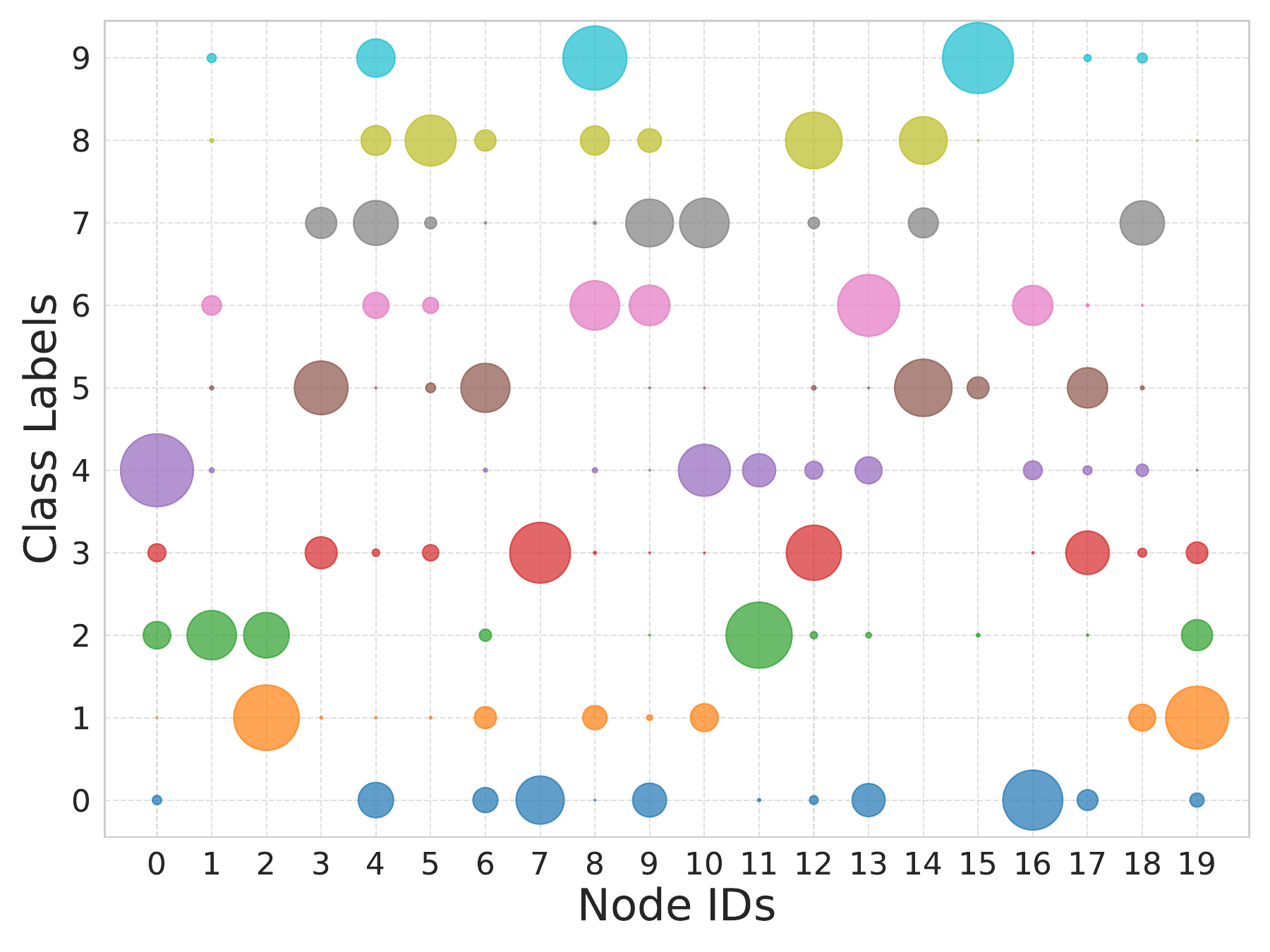}
         \caption{$\alpha = 0.1$.}
         \label{0.1}
     \end{subfigure}
        \caption{Data heterogeneity visualization for CIFAR-10 across a 20-node network using Dirichlet distributions with varying parameter $\alpha$.}
        \label{dirichlet}
\end{figure*}

We have conducted the experiments on
the National Resource Platform (NRP) \cite{NRP_UserGuide} cluster.
Figure~\ref{fig:clustered-network} provides a schematic representation of the network topology and node distribution used in our experiments. The setup consists of 20 nodes grouped into 5 geographic clusters labeled CA, NV, IA, IL, and KS, corresponding to the US states of California, Nevada, Iowa, Illinois, and Kansas, respectively. Within each cluster, nodes are connected locally, while additional links enable communication across clusters, implementing decentralized computation and communication patterns.
As the NRP dynamically assigns resources for each run, the exact node distribution may differ slightly from what is shown in Figure~\ref{fig:clustered-network}. This variability in node assignment is one of the sources of randomness in network conditions, which we account for in our results by reporting error bars.
All nodes are provisioned with $1$ GPU each, along with $10$ CPU cores and $80$ GiB of memory. Additionally, each node mounts a $13$ GiB in-memory volume for high-performance shared memory usage.
The GPU type (e.g., A100, V100, etc.) is determined based on node and cluster assignments made by the National Resource Platform (NRP), which matches resource requests to suitable hardware across participating sites. This assignment introduces an element of randomness into our experiments, as the exact GPU model may vary between runs depending on resource availability.
Most nodes are connected via high-speed research networks such as Science DMZs, with interconnect speeds ranging from 10G to 100G. This setup reflects a realistic decentralized learning environment over a wide-area network and introduces practical considerations like heterogeneous latency and bandwidth, which are difficult to model in simulation.

We use the Dirichlet distribution to create disjoint non-iid nodes \cite{pmlr-v139-lin21c}. The degree of data heterogeneity is controlled by the distribution parameter $\alpha$; the smaller $\alpha$ is, the more likely the nodes hold examples from only one class.
Throughout the experiments, we use three levels of $\alpha$; $10$, $1$, and $0.1$. 

In Figure \ref{dirichlet}, we include the effect of different values of $\alpha$ in creating disjoint noniid data from CIFAR-10 across nodes using the Dirichlet distribution. We observe that as $\alpha$ decreases, the probability of each node containing data from only one class increases.

\subsection{Graph topology} \label{exp-topologu}

\begin{figure*}[tb]
     \centering
     \begin{subfigure}[b]{1\textwidth}
         \centering
         \includegraphics[width=1\textwidth]{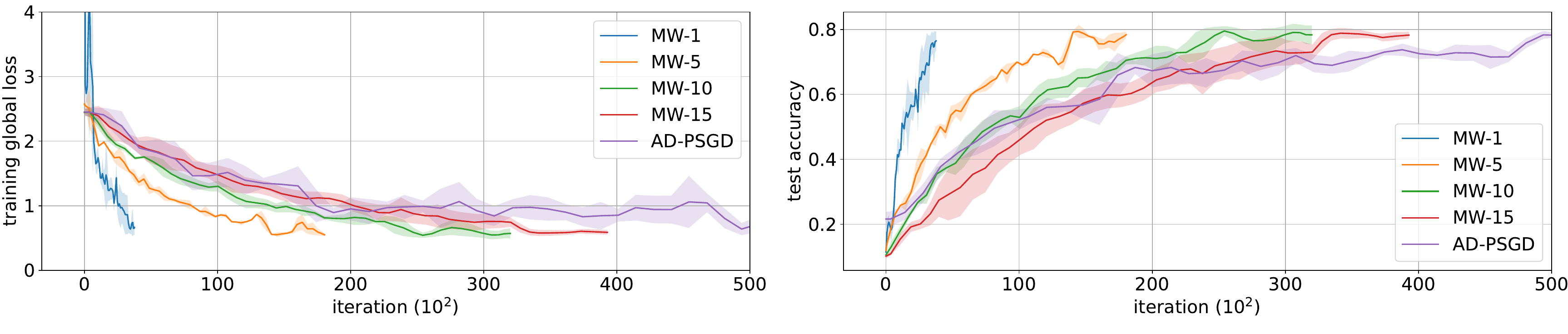}
         \caption{Cycle graph.}
         \label{cycle}
     \end{subfigure}
     \begin{subfigure}[b]{1\textwidth}
         \centering
     \includegraphics[width=1\textwidth]{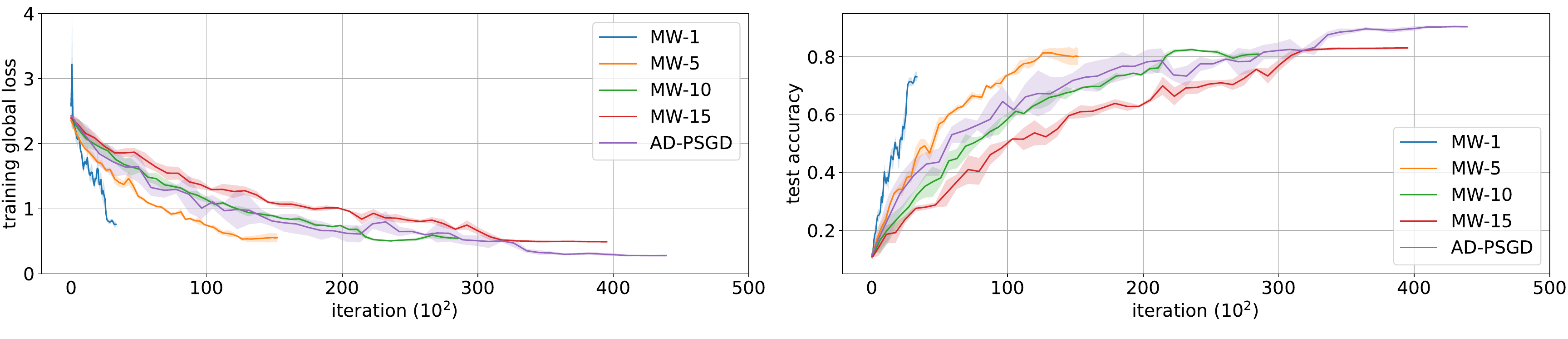}
         \caption{Complete graph.}
         \label{complete}
     \end{subfigure}
     \begin{subfigure}[b]{1\textwidth}
         \centering
        \includegraphics[width=1\textwidth]{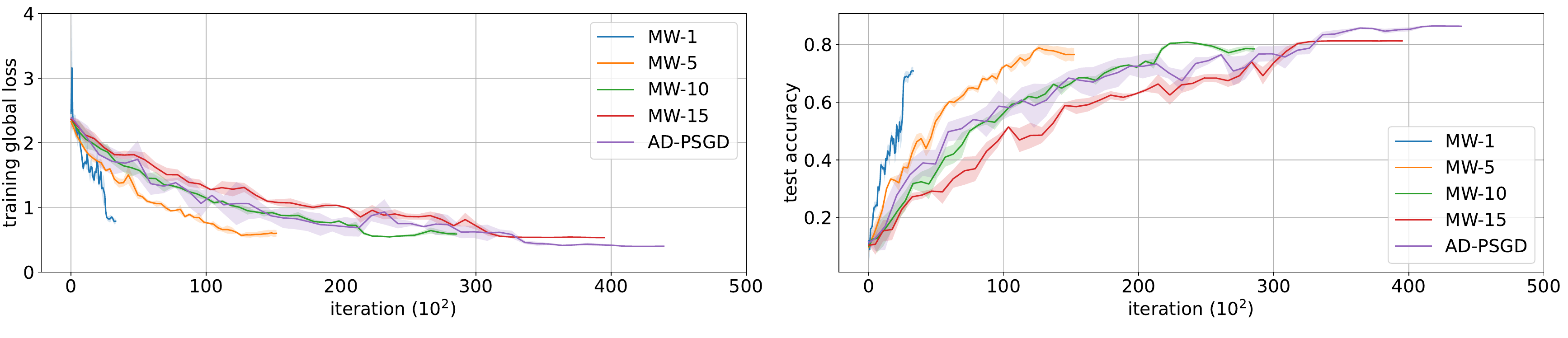}
         \caption{Erdős–Rényi ($0.3$) graph.}
         \label{erdos}
     \end{subfigure}
        \caption{Comparison across different network topologies for a $20$-node graph: Training loss (left column) and test accuracy (right column) of ResNet-$20$ on CIFAR-$10$.}
        \label{topology}
\end{figure*}

Figure \ref{topology} presents the training loss (left column) and test accuracy (right column) of the image classification task in a graph of $20$ nodes. We consider three topologies of cycle, complete, and Erdős–Rényi with connection probability of each pair of nodes being $0.3$.
The noniid-ness level for this experiment is set to $\alpha = 1$.
We observe in Figure \ref{cycle} that the convergence rate \wrt iterations in cycle topology is faster for \ours, regardless of the number of walks ($R$).
We also observe that as we decrease $R$, the convergence rate of \ours \wrt iterations improves.
These are consistent with the theoreticlal results derived in section \ref{conv} and shown Table \ref{iid_table} and \ref{noniid_table}. 
In the small-diameter topology shown in Figure \ref{complete} (a complete graph), we observe that \ours is no longer superior across all numbers of walks; specifically, \agos outperforms \ours when $15$ walks are used. This observation is consistent with the theoretical results indicating that in small-diameter graphs, there exists a specific threshold for the number of walks: below this threshold, Multi-Walk outperforms \agos, whereas above it, performance degrades.
Figure \ref{erdos} presents the results for an Erdős–Rényi topology with the connection probability of $0.3$. This topology, where each node is connected to every other node with a probability of $0.3$, is a well-connected graph with a small diameter. We observe that the Erdős–Rényi graph results are quite similar to the complete graph.

\subsection{Data heterogeneity (Noniid-ness)} \label{exp-noniid}
In this section, we present experiments to evaluate the impact of data heterogeneity on convergence behavior in small and large diameter graphs. We provide comparisons across three domains: iterations, wall-clock time, and transmitted bits.

\subsubsection{Convergence \wrt iterations}

\begin{figure*}[tb]
     \centering
     \begin{subfigure}[b]{0.49\textwidth}
         \centering
         \includegraphics[width=1\textwidth]{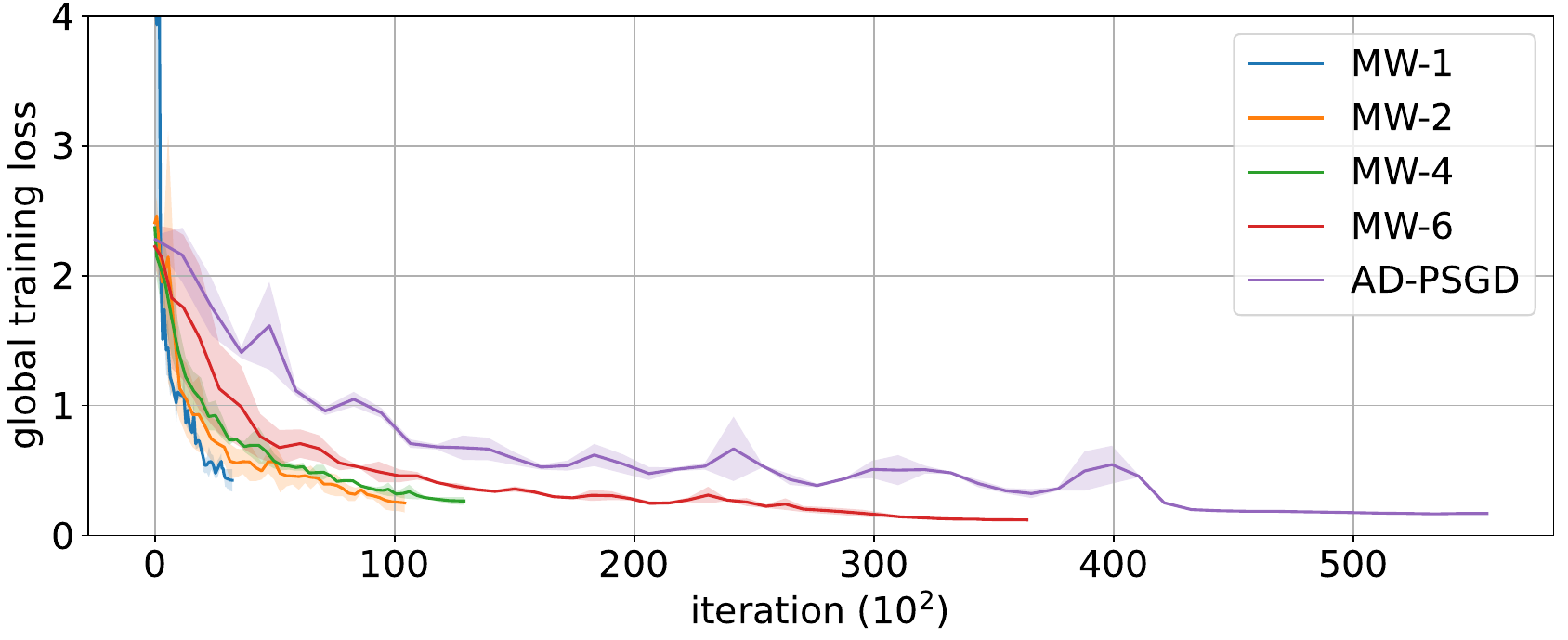}
         \caption{$\alpha=10$, Erdős–Rényi ($0.3$) graph.}
         \label{10-iter-erdos}
     \end{subfigure}
     \vspace{8pt}
     \begin{subfigure}[b]{0.49\textwidth}
         \centering
         \includegraphics[width=1\textwidth]{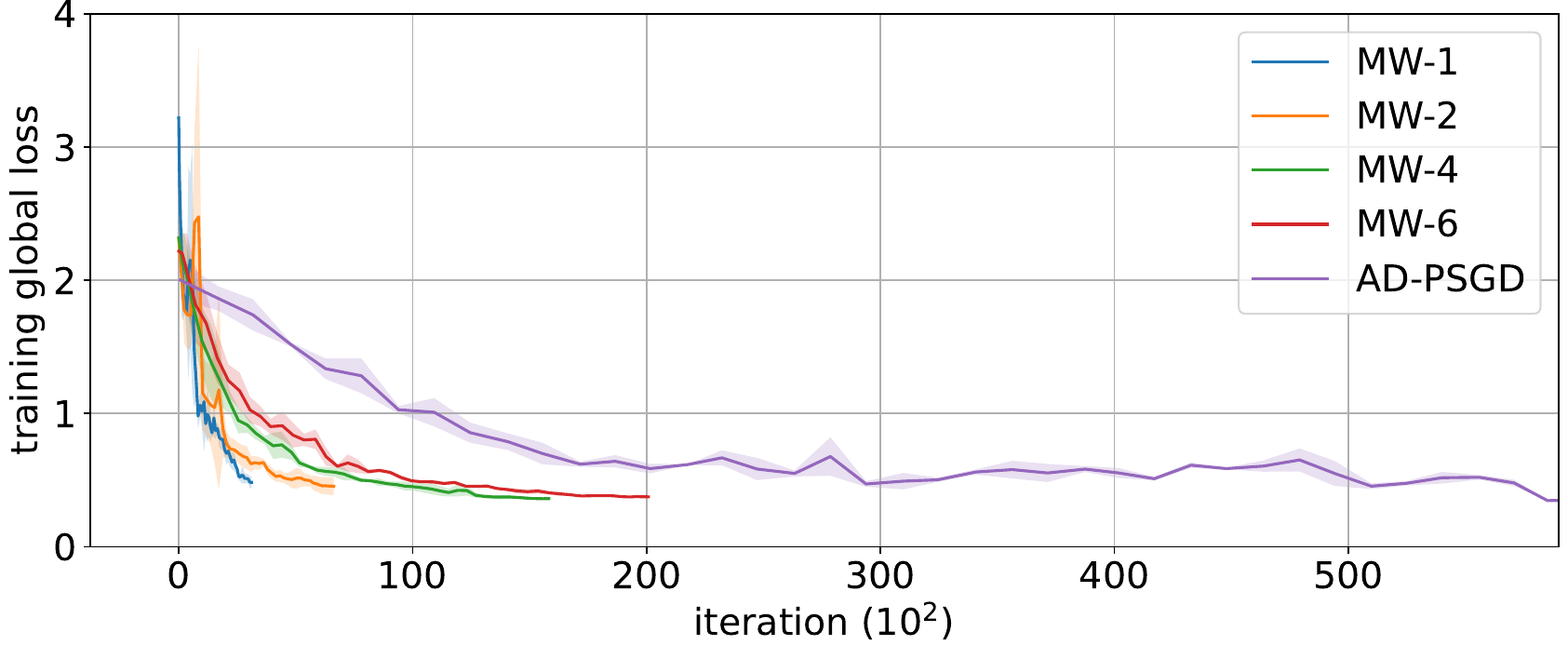}
         \caption{$\alpha=10$, cycle graph.}
         \label{10-iter-cycle}
     \end{subfigure}
     \vspace{8pt}
     \begin{subfigure}[b]{0.49\textwidth}
         \centering
        \includegraphics[width=1\textwidth]{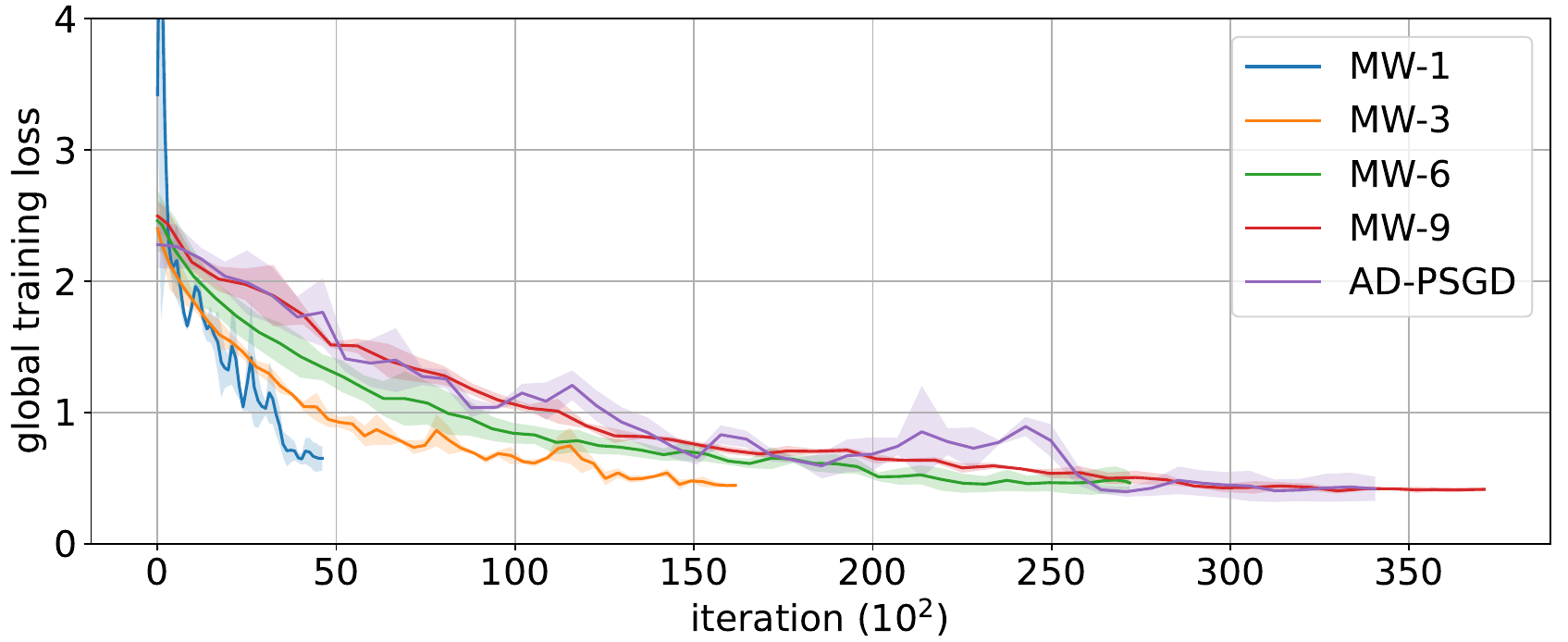}
         \caption{$\alpha=1$, Erdős–Rényi ($0.3$) graph.}
         \label{1-iter-erdos}
     \end{subfigure}
     \begin{subfigure}[b]{0.49\textwidth}
         \centering
        \includegraphics[width=1\textwidth]{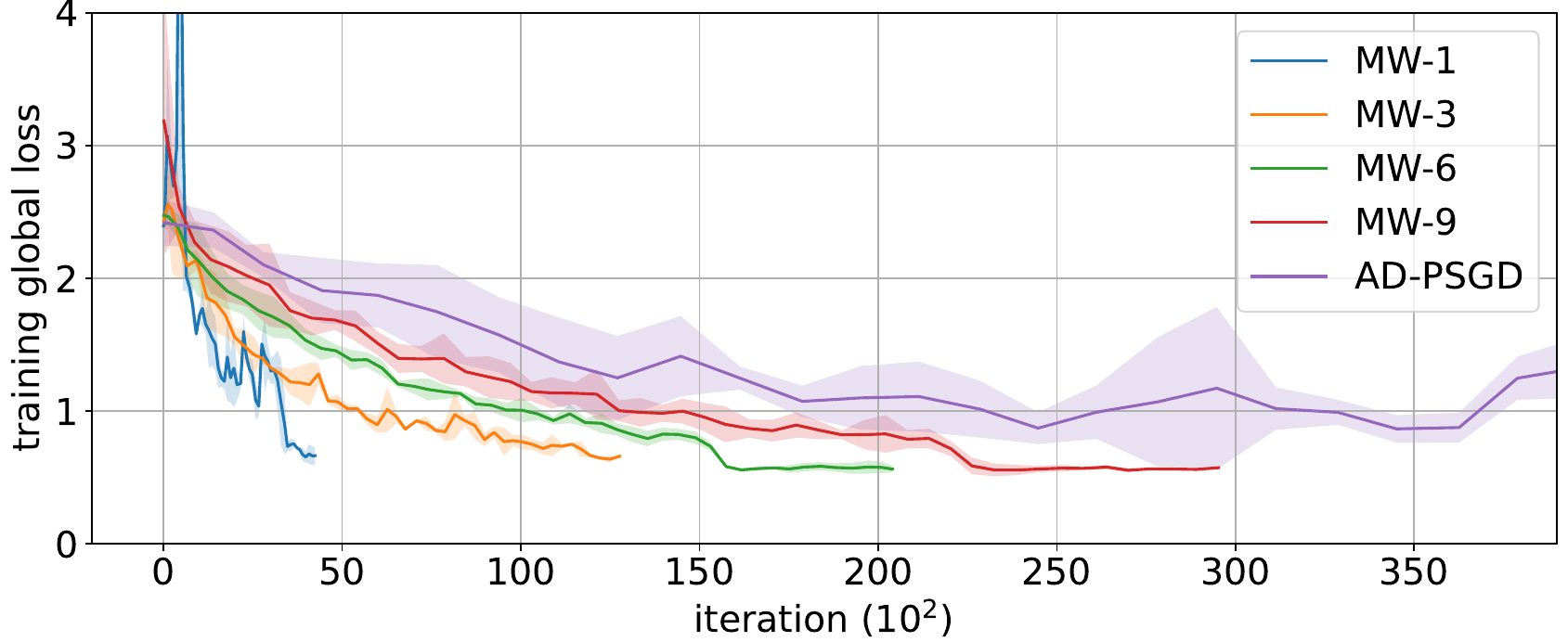}
         \caption{$\alpha=1$, cycle graph.}
         \label{1-iter-cycle}
     \end{subfigure}
     \begin{subfigure}[b]{0.49\textwidth}
         \centering
        \includegraphics[width=1\textwidth]{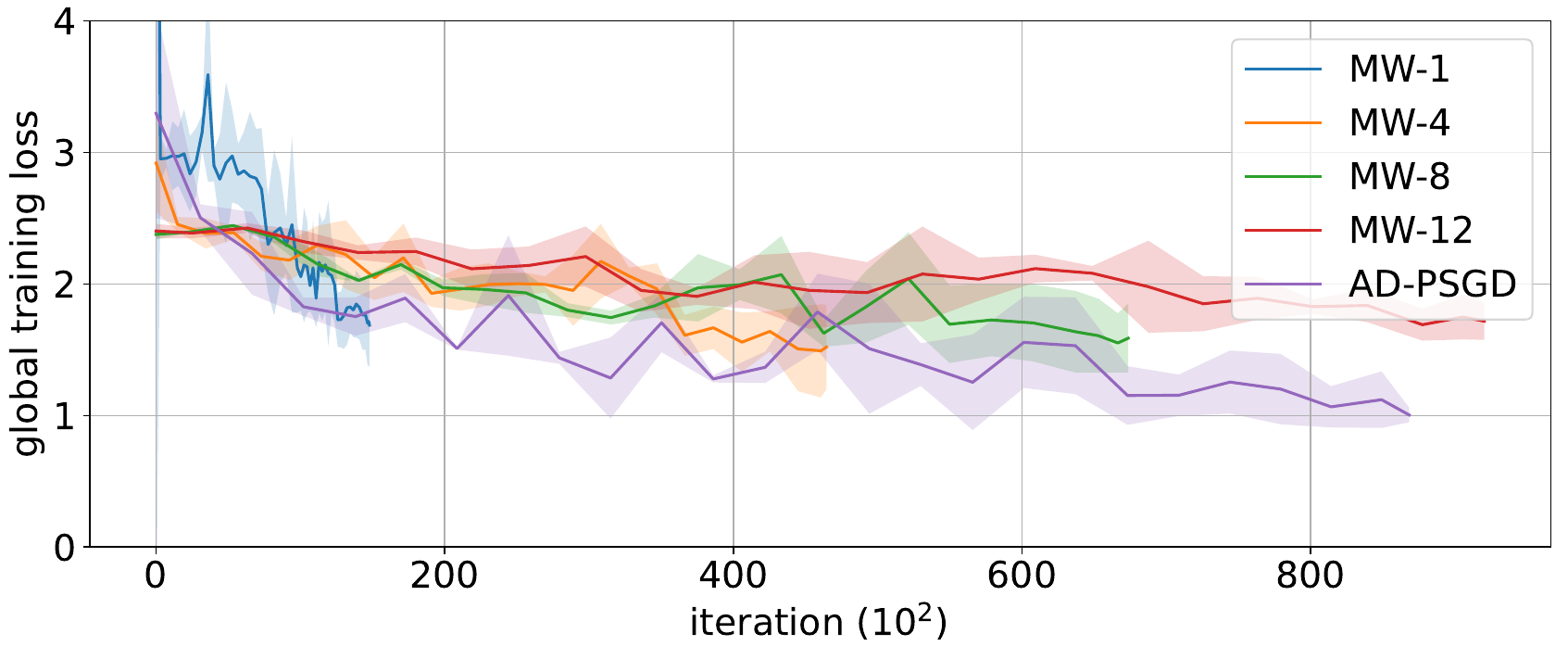}
         \caption{$\alpha=0.1$, Erdős–Rényi ($0.3$) graph.}
         \label{0.1-iter-erdos}
     \end{subfigure}
     \begin{subfigure}[b]{0.49\textwidth}
         \centering
        \includegraphics[width=.98\textwidth]{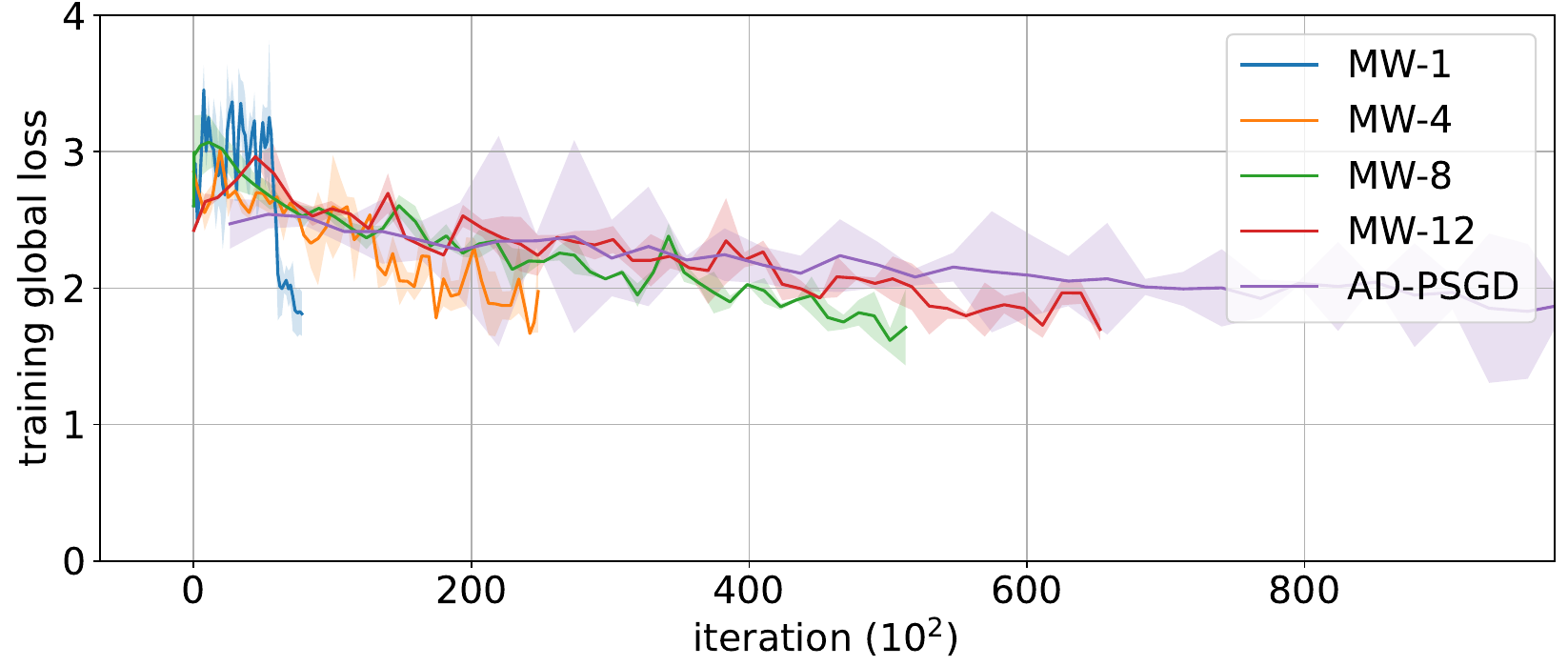}
         \caption{$\alpha=0.1$, cycle graph.}
         \label{0.1-iter-cycle}
     \end{subfigure}
        \caption{Comparison across different noniid-ness levels \textbf{\wrt iterations} for a $20$-node network with Erdős–Rényi ($0.3$) (left column) and cycle (right column) topology: Training loss for ResNet-$20$ on CIFAR-$10$.}
        \label{noniid-iter}
\end{figure*}

Figure \ref{noniid-iter} illustrates the convergence behavior of two different topologies over iterations under varying levels of data heterogeneity. The left column (subfigures d, c, e) corresponds to a 20-node Erdős–Rényi topology ($p=0.3$), while the right column (subfigures b, d, f) depicts a 20-node cycle topology.
We observed in section \ref{exp-topologu} that Erdős–Rényi ($0.3$) has a quite small diameter. 
In the first row (subfigures a, b), where $\alpha=10$ and the data distribution is nearly iid, \ours outperforms \agos in terms of iterations across both graph topologies. We further observe that in this iid scenario, the impact of graph topology is minimal compared to settings with higher data heterogeneity (shown in the second and third rows). Decreasing $\alpha$ to $1$ introduces a more noniid scenario, causing the performance of both \ours and \agos to degrade; however, even at this level of heterogeneity, \ours continues to outperform in both topologies. 
We can go further and reduce $\alpha$ to $0.1$ to get extreme non-iid data distribution (third row). 
Consistent with our theoretical results, \ours is no longer superior in small-diameter graphs under extreme noniid scenarios, as verified in Figure \ref{0.1-iter-erdos}. Conversely, in the cycle topology (characterized by a large diameter), \ours remains faster in terms of iterations.

This result is expected based on the theoretical bounds presented in Table \ref{noniid_table}. In the convergence rate of \ours, the heterogeneity term $\zeta^2$ is scaled by $H^2$, whereas in \agos, it is scaled by $1/p^2$.
In small-diameter graphs (e.g., complete graph), we have $H^2 = \mathcal{O}(V^2)$ and $p = 1$, meaning the impact of noniid data on \ours ($\mathcal{O}(V^2)$) is far more severe than on \agos ($\mathcal{O}(1)$). This confirms the degradation observed in our experiments.
On the other hand, in large-diameter graphs like the cycle topology, we have $H^2 = \mathcal{O}(V^3)$ and $p =\Theta(1/V^2)$ (implying $1/p^2 = \mathcal{O}(V^4)$). Consequently, the impact of heterogeneity scales better for \ours ($\mathcal{O}(V^3)$) compared to \agos ($\mathcal{O}(V^4)$), explaining why \ours retains its advantage in this setting.

\subsubsection{Convergence \wrt wall-clock time}

\begin{figure*}[tb]
    \centering
    \begin{subfigure}[b]{0.49\textwidth}
         \centering
         \includegraphics[width=1\textwidth]{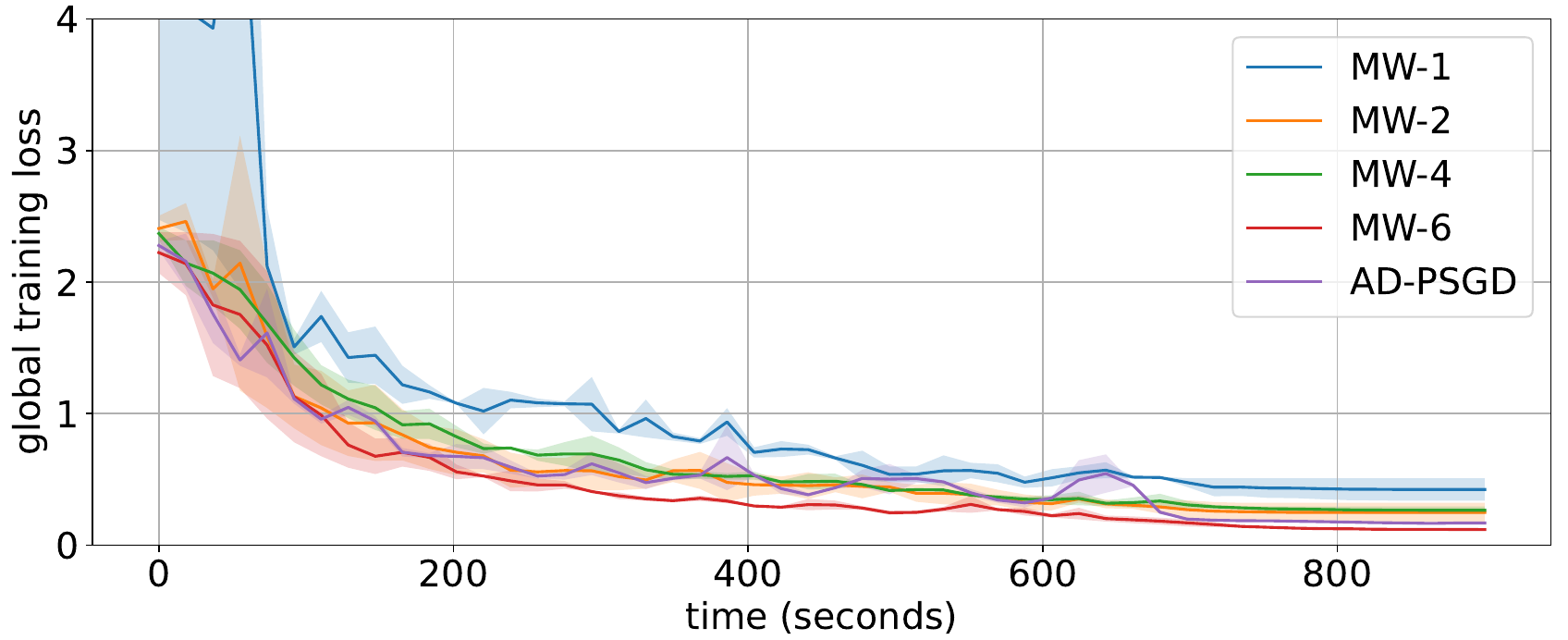}
         \caption{$\alpha=10$, Erdős–Rényi ($0.3$) graph.}
         \label{10-time-erdos}
     \end{subfigure}
     \vspace{7pt}
     \begin{subfigure}[b]{0.49\textwidth}
         \centering
         \includegraphics[width=.98\textwidth]{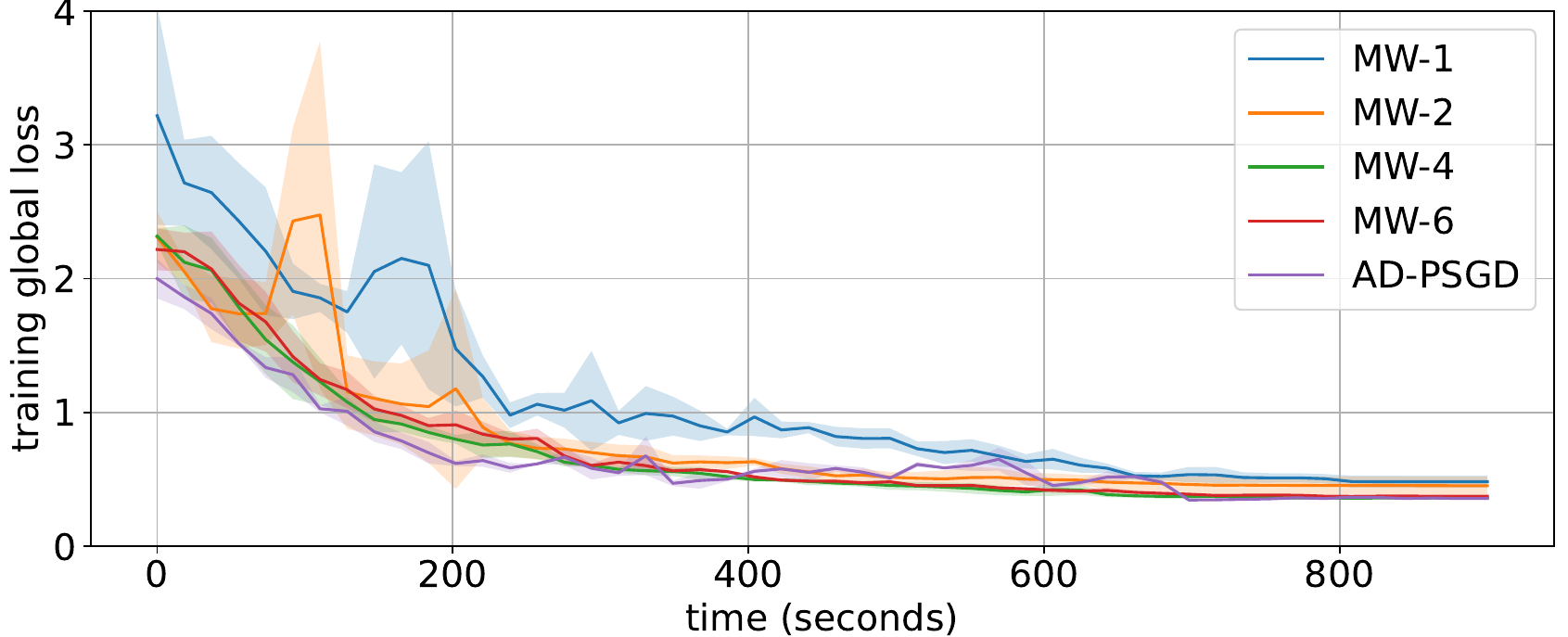}
         \caption{$\alpha=10$, cycle graph.}
         \label{10-time-cycle}
     \end{subfigure}
     \vspace{7pt}
     \begin{subfigure}[b]{0.49\textwidth}
         \centering
        \includegraphics[width=1\textwidth]{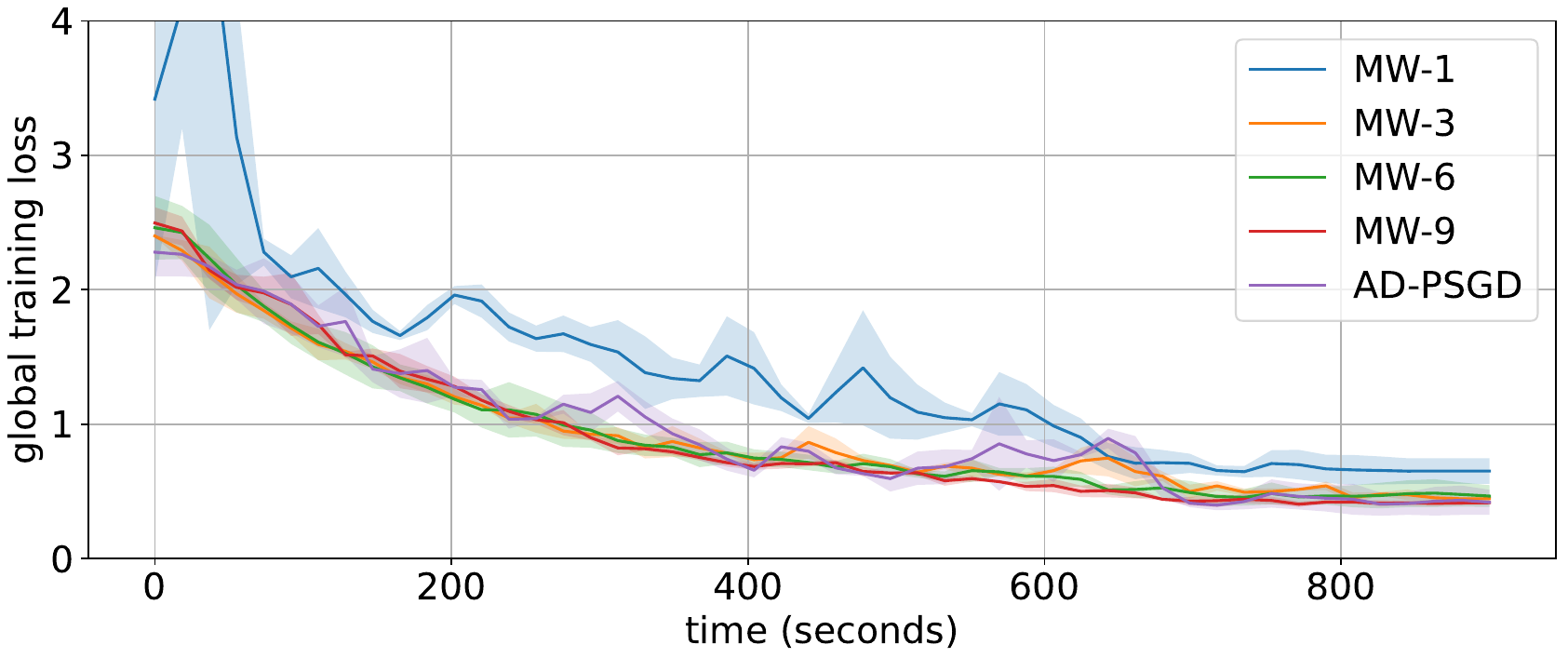}
         \caption{$\alpha=1$, Erdős–Rényi ($0.3$) graph.}
         \label{1-time-erdos}
     \end{subfigure}
     \begin{subfigure}[b]{0.49\textwidth}
         \centering
        \includegraphics[width=1\textwidth]{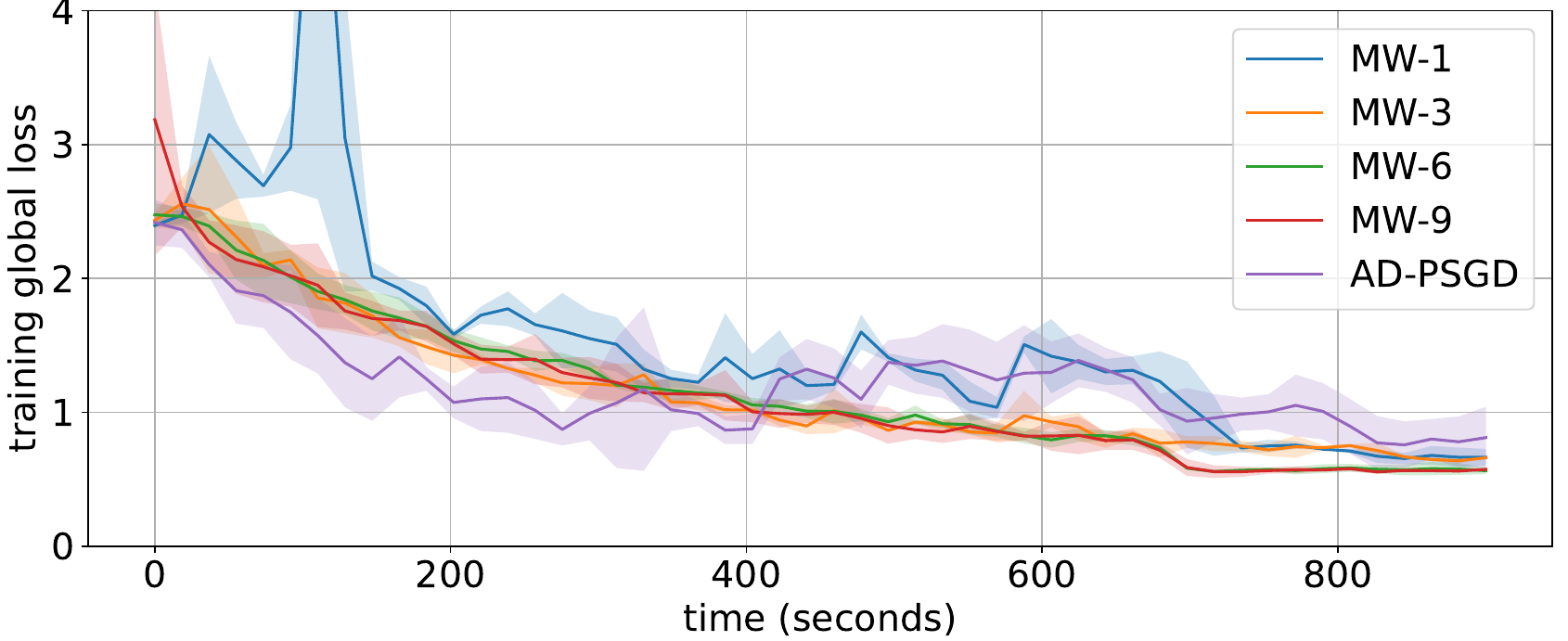}
         \caption{$\alpha=1$, cycle graph.}
         \label{1-time-cycle}
     \end{subfigure}
     \begin{subfigure}[b]{0.49\textwidth}
         \centering
        \includegraphics[width=1.025\textwidth]{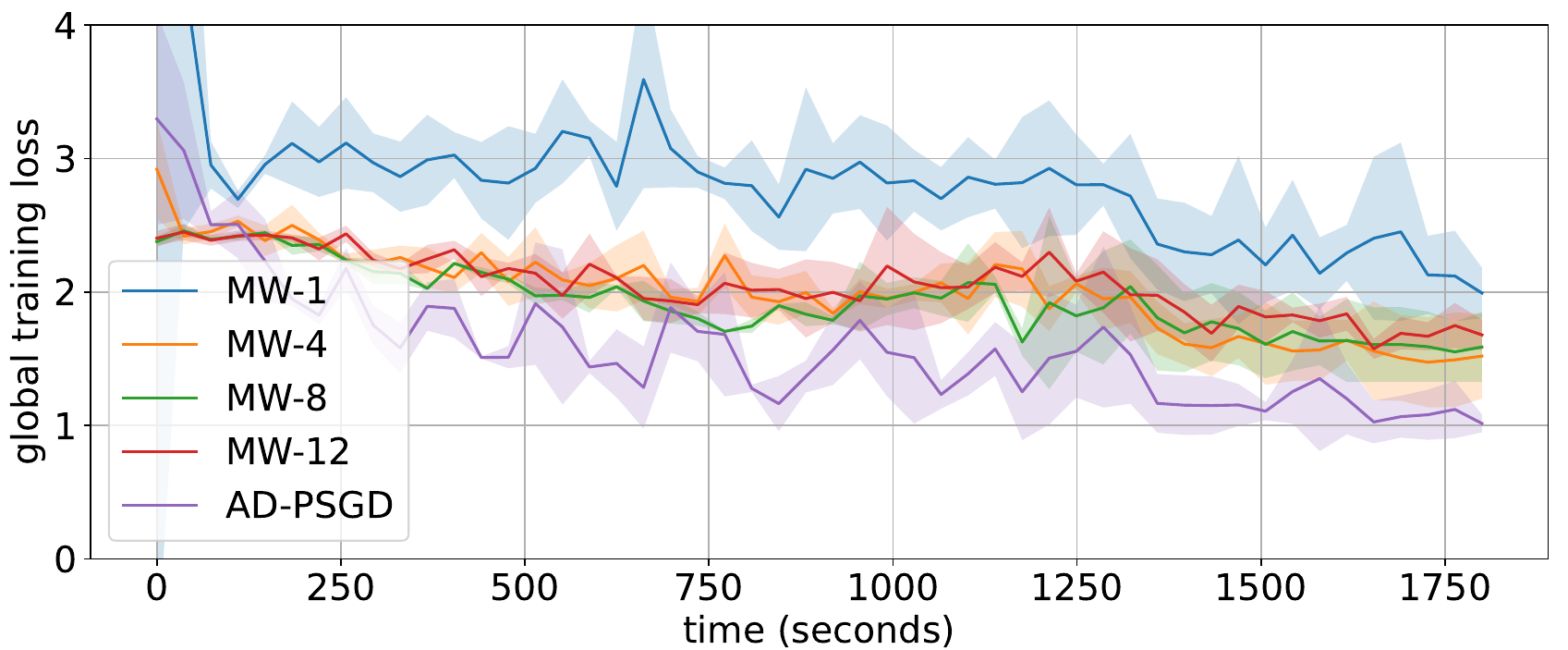}
         \caption{$\alpha=0.1$, Erdős–Rényi ($0.3$) graph.}
         \label{0.1-time-erdos}
     \end{subfigure}
     \begin{subfigure}[b]{0.49\textwidth}
         \centering
        \includegraphics[width=1.025\textwidth]{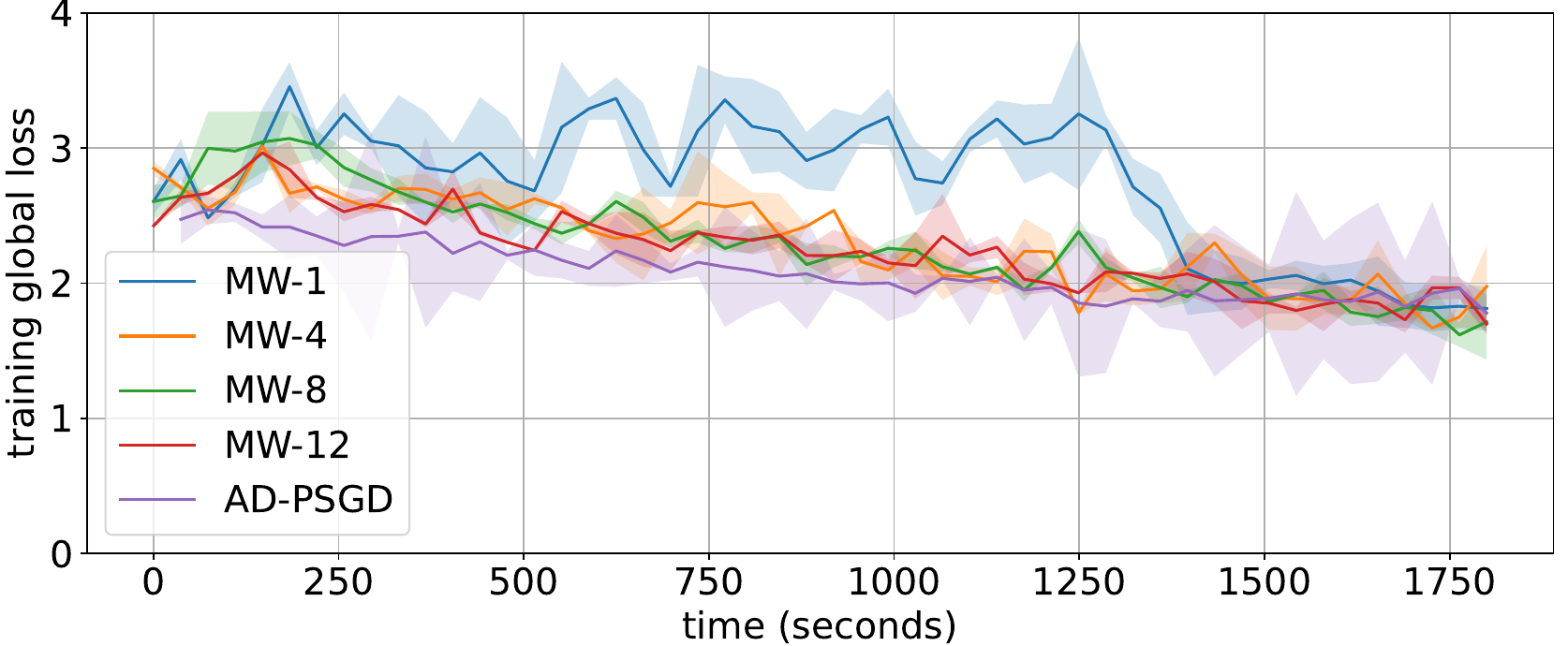}
         \caption{$\alpha=0.1$, cycle graph.}
         \label{0.1-time-cycle}
     \end{subfigure}
     \caption{Comparison across different noniid-ness levels \textbf{\wrt wall-clock time} for a $20$-node network with Erdős–Rényi ($0.3$) (left column) and cycle (right column) topology: Training loss for ResNet-$20$ on CIFAR-$10$.}
        \label{noniid-time}
\end{figure*}

Figure \ref{noniid-time} shows convergence versus wall-clock time.
We know that in the time domain, \agos achieves a linear speed-up with the number of nodes compared to \ours with a single walk. However, \ours can improve its time-domain performance by increasing the number of walks, yielding a linear speed-up with respect to the walk count.In the first and second rows, we observe that increasing the number of walks to 2 or 4 is sufficient to catch up with \agos. Conversely, as we increase heterogeneity to $\alpha = 0.1$, we observe that in the Erdős–Rényi ($0.3$) topology (which has a small diameter), the gap between \agos and \ours cannot be bridged even by increasing the number of walks. This reinforces the fact that \agos performs better in small-diameter graphs under extreme heterogeneity.

\subsubsection{Convergence \wrt transmitted bits}

\begin{figure*}
    \begin{subfigure}[b]{0.49\textwidth}
         \centering
         \includegraphics[width=1\textwidth]{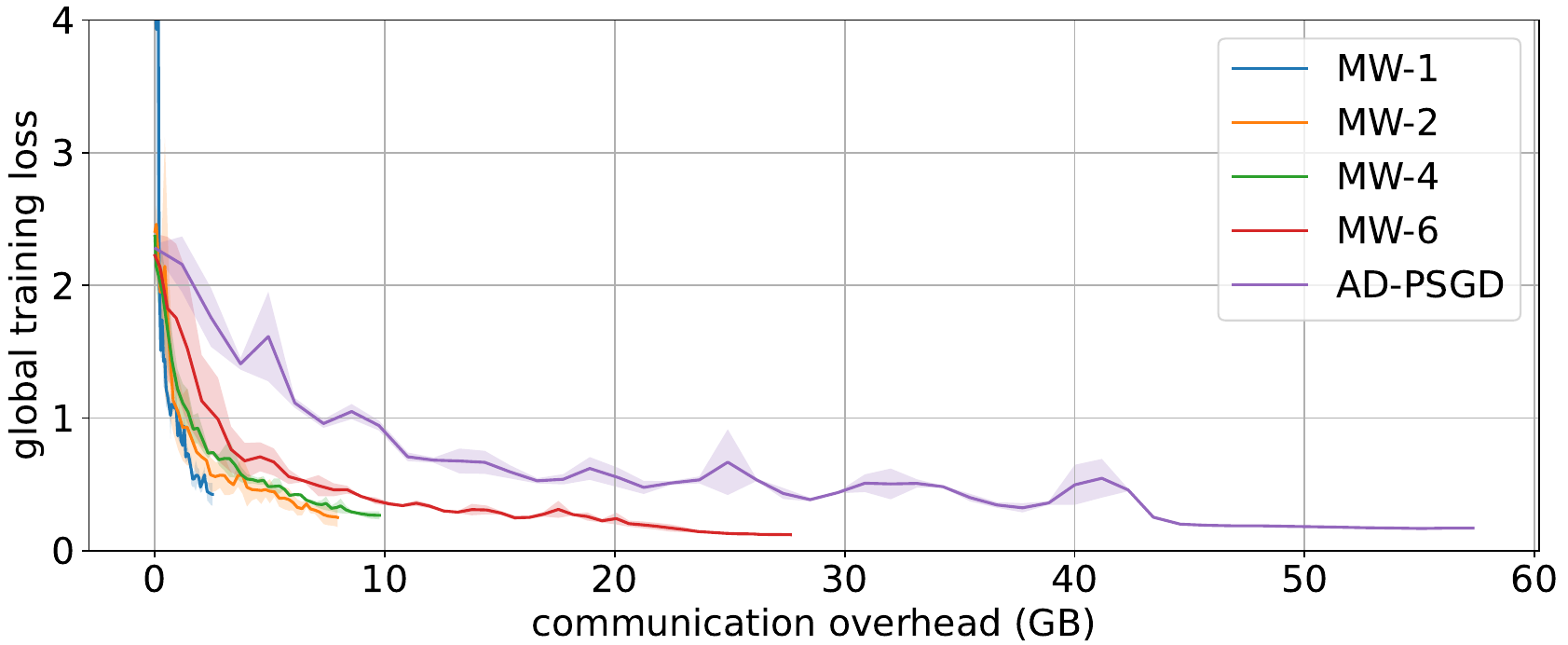}
         \caption{$\alpha=10$, Erdős–Rényi ($0.3$) graph.}
         \label{10-comm-erdos}
     \end{subfigure}
     \vspace{7pt}
     \begin{subfigure}[b]{0.49\textwidth}
         \centering
         \includegraphics[width=1\textwidth]{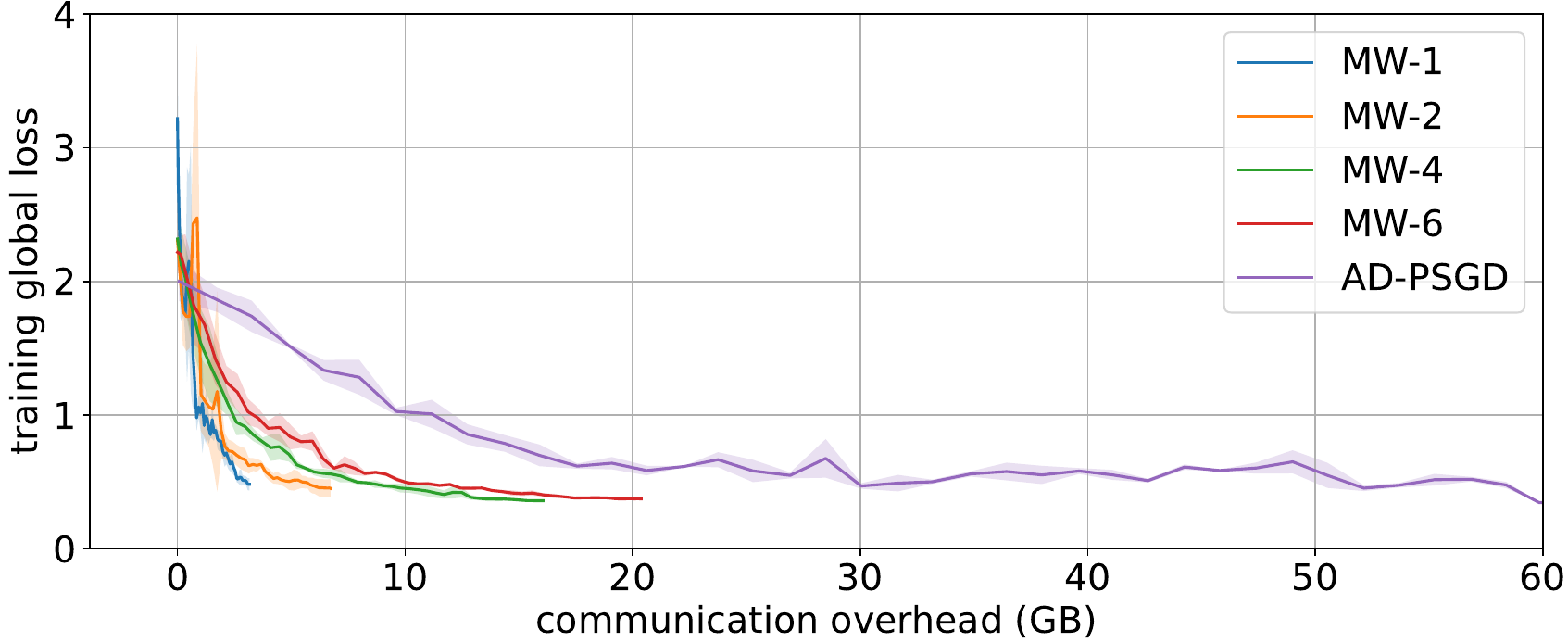}
         \caption{$\alpha=10$, cycle graph}
         \label{10-comm-cycle}
     \end{subfigure}
     \vspace{7pt}
     \begin{subfigure}[b]{0.49\textwidth}
         \centering
        \includegraphics[width=1\textwidth]{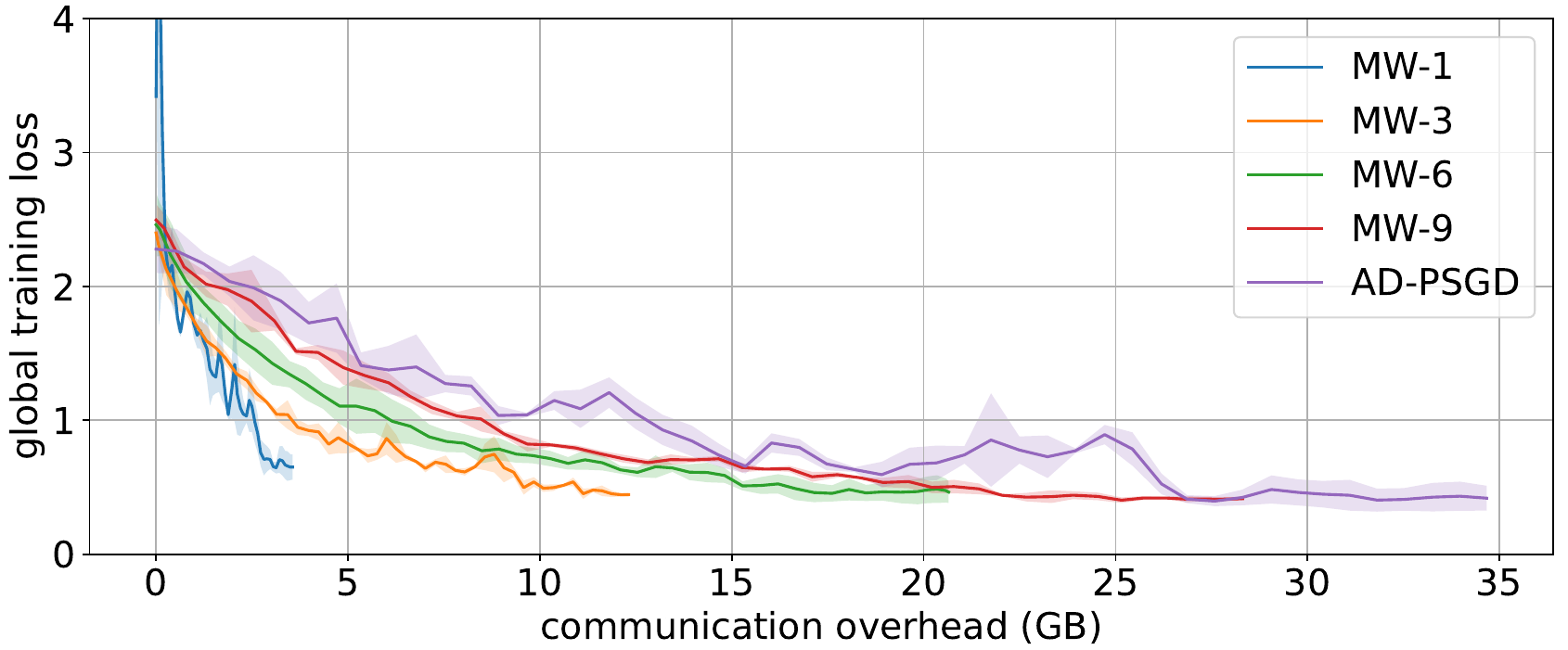}
         \caption{$\alpha=1$, Erdős–Rényi ($0.3$) graph.}
         \label{1-comm-erdos}
     \end{subfigure}
     \begin{subfigure}[b]{0.49\textwidth}
         \centering
        \includegraphics[width=1\textwidth]{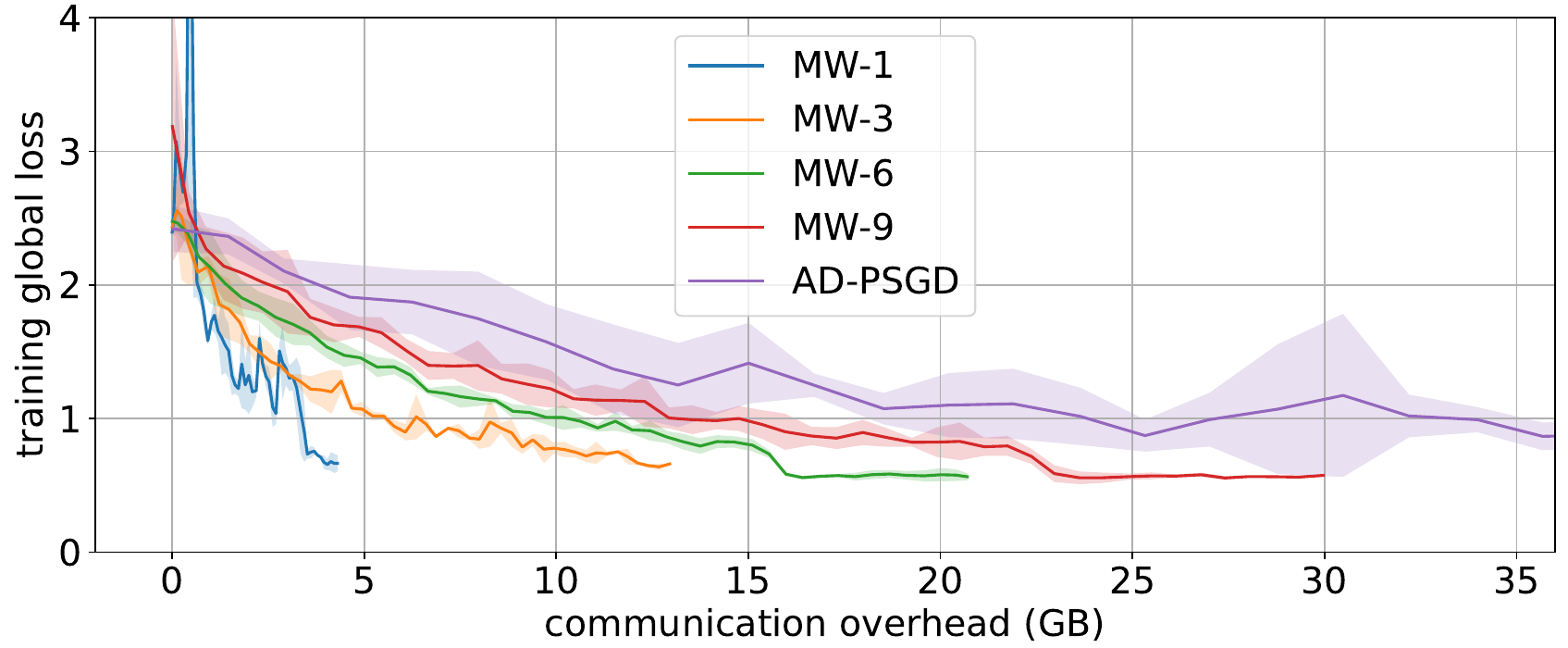}
         \caption{$\alpha=1$, cycle graph.}
         \label{1-comm-cycle}
     \end{subfigure}
     \begin{subfigure}[b]{0.49\textwidth}
         \centering
        \includegraphics[width=1\textwidth]{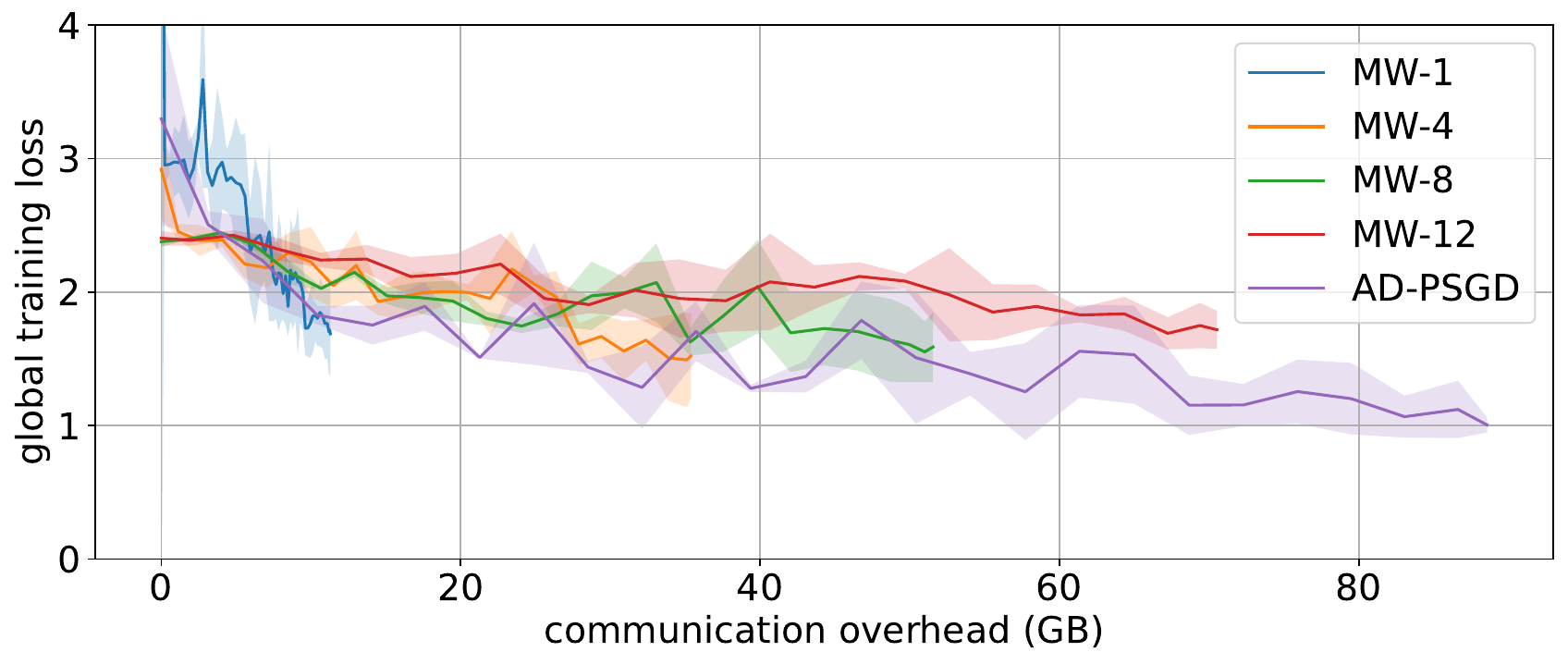}
         \caption{$\alpha=0.1$, Erdős–Rényi ($0.3$) graph.}
         \label{0.1-comm-erdos}
     \end{subfigure}
     \begin{subfigure}[b]{0.49\textwidth}
         \centering
        \includegraphics[width=.99\textwidth]{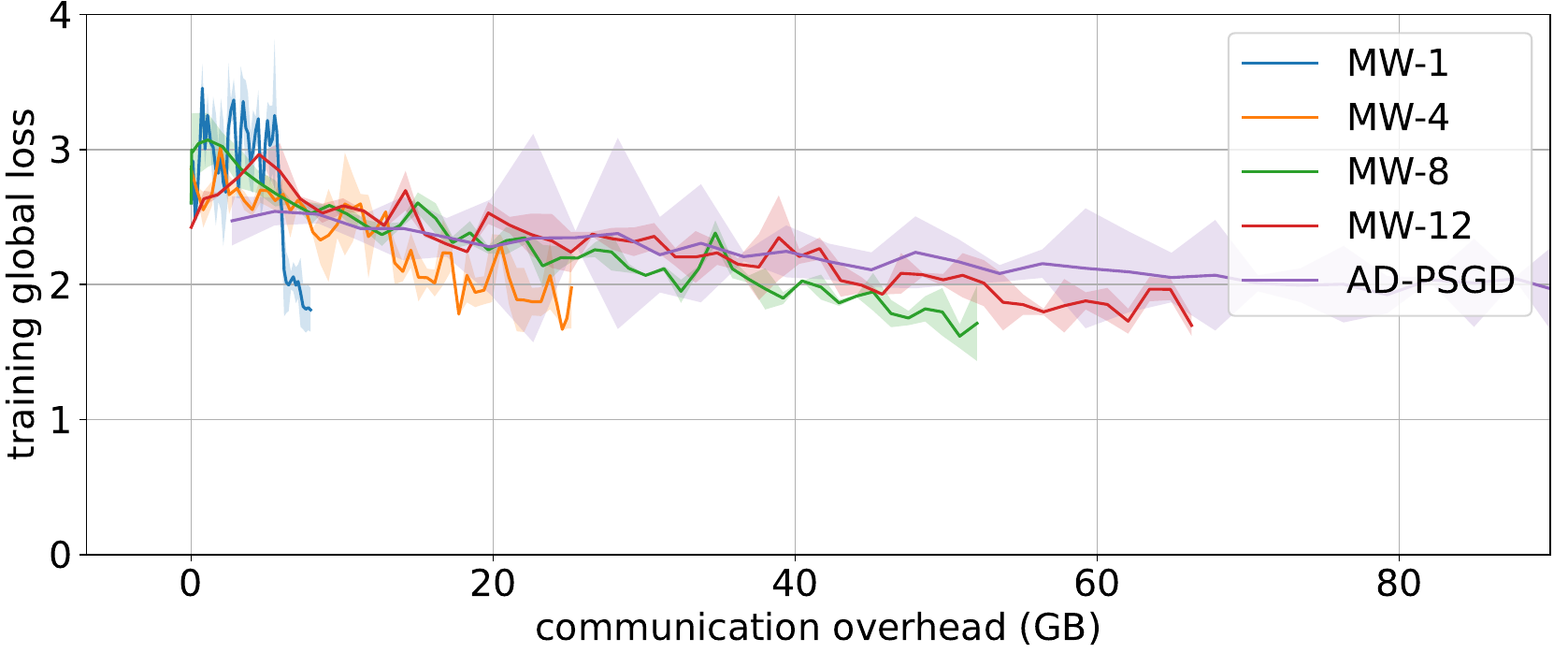}
         \caption{$\alpha=0.1$, cycle graph.}
         \label{0.1-comm-cycle}
     \end{subfigure}
     \caption{Comparison across different noniid-ness levels \textbf{\wrt transmitted bits} for a $20$-node network with Erdős–Rényi ($0.3$) (left column) and cycle (right column) topology: Training loss for ResNet-$20$ on CIFAR-$10$.}
        \label{noniid-bits}
\end{figure*}

Figure \ref{noniid-bits} shows convergence versus transmitted bits.
In terms of communication overhead, we observe that in settings that are not extremely non-iid, where the second dominating term is negligible (due to small $\zeta$), \ours outperforms \agos as predicted by the results in Table \ref{bits}. This can be seen in the first ans second row in Figure \ref{noniid-bits}.
However, in the extreme non-iid setting of the third row, the value of $\zeta$ becomes too large that the second dominating term in Table \ref{bits} comes into play. In this term, the impact of noniid-ness in a graph topology with a small diameter (Erdős–Rényi ($0.3$)) significantly disfavors \ours, which is evident from the observed results in Figure \ref{0.1-comm-erdos}.
In Figure \ref{0.1-comm-cycle}, we again observe that, in contrast to small-diameter topologies, \ours continues to outperform even under extreme noniid conditions. This is again predicted based on the theoretical results in section \ref{conv-bits}.
Intuitively, in small-diameter graphs where connectivity is dense, Gossip algorithms propagate information across the network more efficiently than Random Walks. This rapid information spread is critical under extreme heterogeneity (noniid settings), as nodes rely on global information to maintain a trajectory toward the global minimum and avoid getting trapped in local optima.

\subsection{Communication restricted settings} \label{exp-comm}


Figure \ref{comm} shows that fine-tuning OPT-125M on MultiNLI in an Erdős–Rényi (0.3) graph with 20 nodes.
The deployment of this larger model necessitates a data transfer of $500$ MB per communication between two nodes. This is significantly higher than the overhead for our image classification task using ResNet-20, which requires only $1.08$ MB.

In Figure \ref{llm-comm}, the horizontal axis represents the total communicated bits during fine-tuning. We observe that while \ours with a single walk requires approximately $50$ GB to converge, \agos requires roughly $12$ times that amount, around $600$ GB.
We have also presented the convergence \wrt wall-clock time in Figure \ref{llm-time}. Although \agos benefits from linear speedup due to the increased number of active nodes, \ours still outperforms it.
This is because, as the model size increases, using \agos with numerous concurrent communications in the network leads to congestion, which, in turn, increases the average communication delay across network links. Consequently, this results in a larger $d$ i.e., the average computation and gossip communication delay in the system (section \ref{conv-time}), making \agos slower with respect to wall-clock time as well.
In other words, even though we theoretically achieve a linear speed-up with the number of nodes (a $20\times$ factor in our setting) in \agos, the resulting network congestion and increased delay per iteration negate these gains in terms of wall-clock time compared to a random walk with a single walk. Consequently, we observe that in settings where communication resources are restricted, \ours offers a promising alternative.

\begin{figure}[tb]
     \centering
     \begin{subfigure}[b]{0.49\textwidth}
         \centering
         \includegraphics[width=1\textwidth]{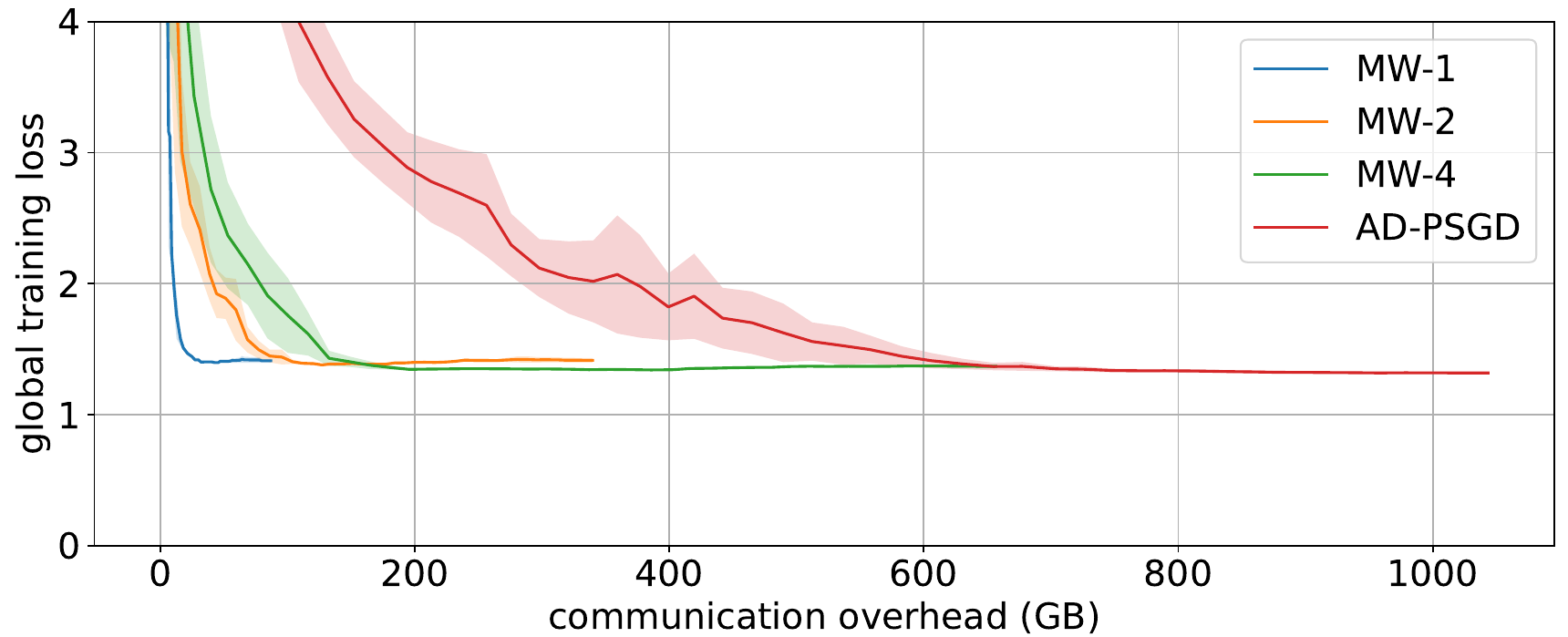}
         \caption{ \wrt  transmitted bits.}
         \label{llm-comm}
     \end{subfigure}
     \begin{subfigure}[b]{0.49\textwidth}
         \centering
         \includegraphics[width=1\textwidth]{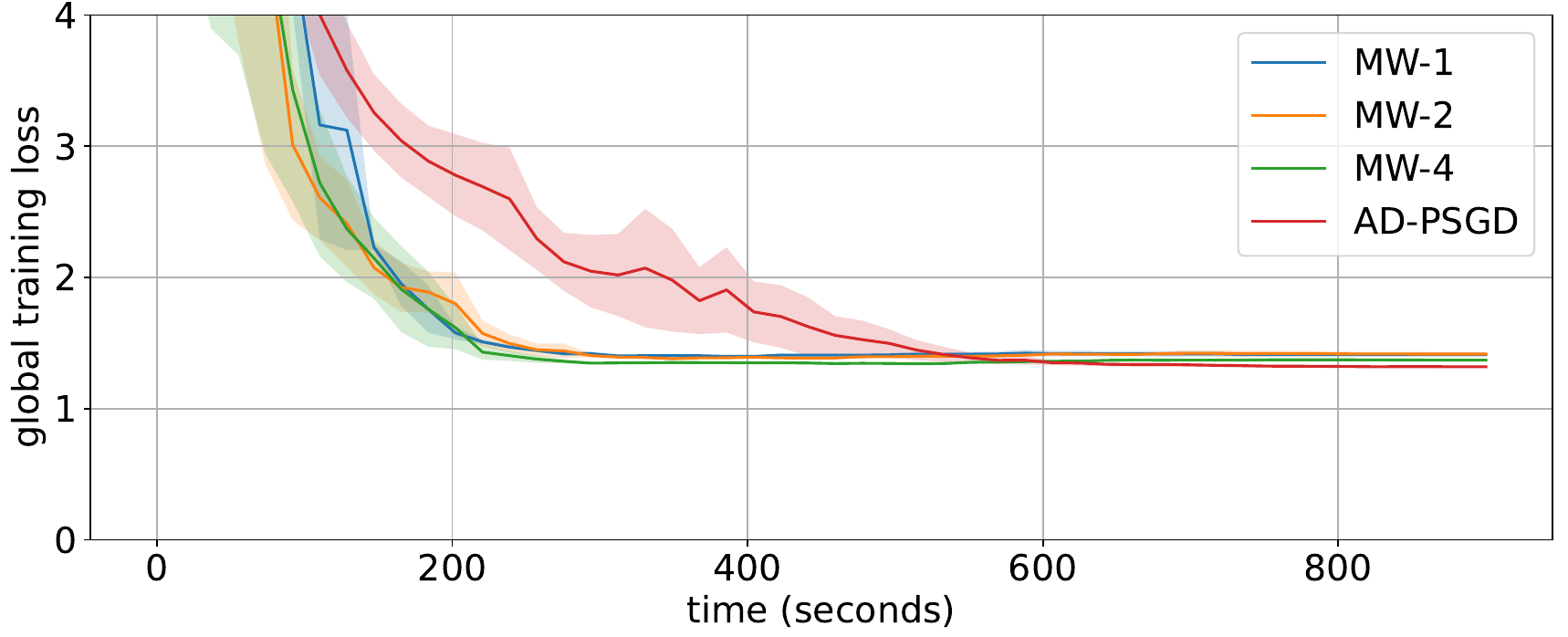}
         \caption{ \wrt  wall-clock time.}
         \label{llm-time}
     \end{subfigure}
        \caption{Fine-tuning OPT-$125$M on the MultiNLI corpus in a $20$-node Erdős–Rényi ($0.3$) graph. Congestion increases communication latency in \agos, allowing \ours to outperform \agos in the time domain as well.}
        \label{comm}
\end{figure}

\subsection{Resilience 
}\label{exp_res}

Figure \ref{Fail} depicts the convergence behavior of ResNet-20 training on CIFAR-10 over a 20-node Erdős–Rényi ($0.3$) graph, subject to two \master failures. The first failure happens at 300 second and the second one happend at 600 seconds as shown with gray dashed lines in Figure \ref{Fail}.

\begin{figure}[tb]
     \centering
     \begin{subfigure}[b]{0.49\textwidth}
         \centering
         \includegraphics[width=1\textwidth]{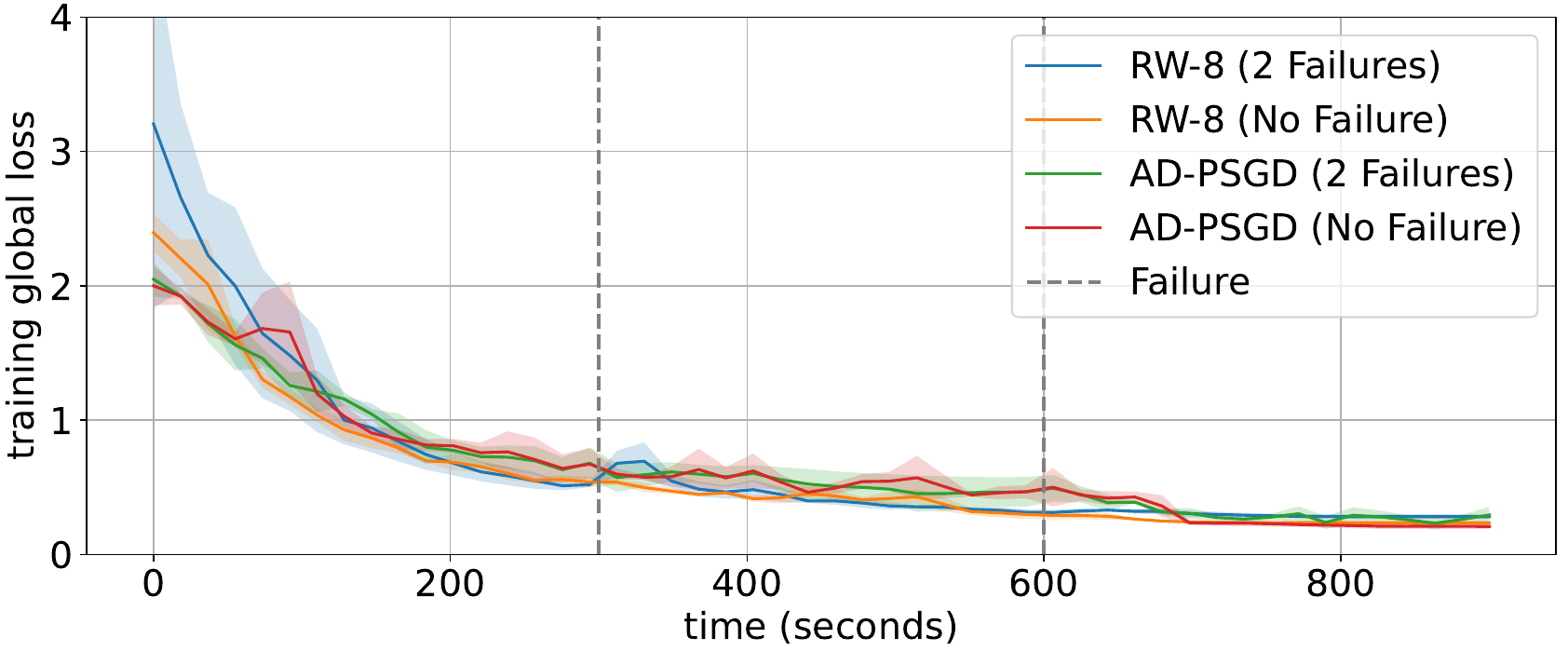}
         \caption{Training loss}
         \label{fail-loss}
     \end{subfigure}
     \begin{subfigure}[b]{0.49\textwidth}
         \centering
         \includegraphics[width=1\textwidth]{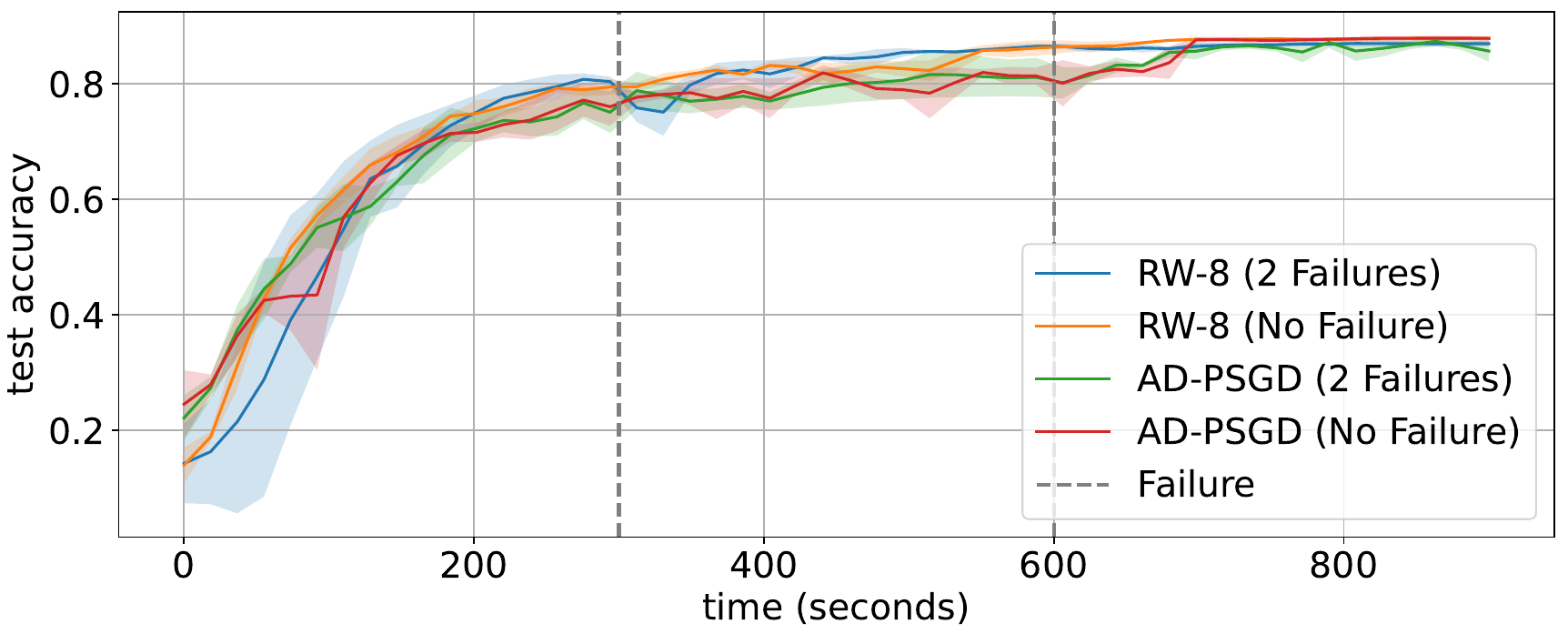}
         \caption{Test Accuracy}
         \label{fail-acc}
     \end{subfigure}
        \caption{Convergence behavior of ResNet-$20$ on CIFAR-$10$ subject to two \master failures occurring at $300$ and $600$ seconds.}
        \label{Fail}
\end{figure}

Here, we observe that although the \master fails twice during training, convergence is still achieved. This is because the \master is primarily used to integrate information from different walks, while the core model information is retained within each individual walk (as slightly outdated versions of the global model). Consequently, the loss of the \master has a minimal effect on the convergence rate, allowing training to proceed after each failure.
For comparison, we also evaluated \agos subject to an identical sequence of node failures. We observe that the performance results are nearly identical.
Note that when nodes are lost, the model converges to a slightly different solution. This occurs because the training proceeds without the corresponding partition of the dataset, effectively altering the global objective function.
\section{Conclusions}


In this paper, we designed and analyzed Multi-Walk, a random walk-based learning algorithm with multiple streams. We compared gossip- and random walk-based decentralized learning algorithms through analysis and exhaustive experiments. We observed that gossip-based methods perform better in networks with small diameters when the data distribution is extremely heterogeneous, while random walk-based ones excel in large-diameter networks. These results highlight the effectiveness of each algorithm across varying graph topologies and levels of data heterogeneity.
In terms of limitations, our analysis considers an upper bound for the rate of convergence to gain deeper insight into the behavior of these algorithms. However, without a corresponding lower bound that is close to the derived upper bounds, we cannot assess the tightness of these bounds or confidently characterize their properties.
Our analysis focuses on deriving an upper bound for the convergence rate to provide deeper insight into the behavior of these algorithms. While a matching lower bound would allow for a tighter characterization, our results still offer valuable understanding of their theoretical properties.

\bibliographystyle{ACM-Reference-Format}
\bibliography{main}

\appendix


\section{Notation Table}\label{appendixa}

\renewcommand{\arraystretch}{1.5} 
\begin{tabularx}{1\textwidth} { 
  | >{\raggedright\arraybackslash}m{5em}
  | >{\raggedright\arraybackslash}X |}
 \hline
 $G=(\mathcal{V},\xi)$ & The graph representing the network\\
 \hline
 $V$ & Number of nodes\\
 \hline
 $\mathcal{D}_v$ & Local dataset at node $v$\\
 \hline
 $\F[v][][][]$ & Loss function of $\vec{x}$ associated with the data sample $\xi$ at node $v$ \\
 \hline
 $f(\vec{x})$ & Global loss function of model $\vec{x}$ \\
 \hline
 $f_v(\vec{x})$  & Local loss function of model $\vec{x}$ on local dataset $\data{v}$ at node $v$\\
 \hline
 $f^*$ & $\min_{\vec{x} \in \mathbb{R}^d}f(\vec{x})$\\
\hline
 $\xz$ & Initial model\\
 \hline
$T$ & Total number of iterations\\
 \hline
 $\eta_t$ & Learning rate at iteration $t$\\
\hline
$\xs[r][t]$ & Model of walk $r$ at iteration $t$ in \ours Algorithm\\
\hline
$\xs[v][t]$ & Local model of node $v$ at iteration $t$ in \agos Algorithm\\
\hline
$u^{r}$ &  A copy of the model of walk $r$ at the most recent instance when that walk was at \master in \ours Algorithm; to be kept at \master\\
\hline
$l$ & The index of the latest walk visited \master in \ours Algorithm\\
\hline
 \(\vec{P}\) & The transition matrix of each walk in \ours, and in \agos, it defines the mixing step of the gossip process\\
 \hline
$p_{ij}$ &   The element in row $i$ and column $j$ of $\vec{P}$\\
\hline
\(p\) & The spectral gap of \(\vec{P}^\top \vec{P}\)\\
 \hline
\(p'\) & The spectral gap of  \(\vec{P}\) \\
\hline
 $m$ & Model size in bits \\
\hline
$B$  & Total transmitted bits\\
\hline
$Z$ &  Wall-clock time\\
\hline
$L$ & $f_v(\vec{x})$'s gradient is $L$-Lipschitz\\
\hline
$\sigma^2$ & Upper Bound for local variance\\
\hline
$\zeta^2$ & Upper Bound for diversity\\
\hline
$F$ & $f(\xz)-f^*$\\
\hline
$H^2$ &  The second moment of the first return time to \master for the Markov chain representing each walk\\
\hline
$\alpha$ &   The degree of noniid-ness in the Dirichlet distribution is used to create disjoint noniid nodes; smaller values indicate a higher level of noniid-ness\\
\hline
\end{tabularx}

\renewcommand{\arraystretch}{1} 
\newpage

\section{Proof of Theorem \ref{T1}}\label{appendixb}

Motivated by \cite{stich2019localsgdconvergesfast}, a virtual sequence $\{\xtild\}_{t\geq 0}$ is defined as follows.
\begin{equation}
\xtild[t+1] =\xtild - \frac{\eta 
}{R} \nabla \F,
\end{equation} 
where we define $\hat{\tau}_t$ as the delay with which the gradient of the corresponding point $(\xs)$ will be computed. If we
denote $t' = t + \hat{\tau}_t$, then it holds that $t' - \tau_{t'} = t$.
We do not need to calculate this sequence in the algorithm explicitly and it is only used for the sake of analysis.

First, we illustrate how the virtual sequence, $\{\xtild\}_{t\geq 0}$, approaches to the optimal. Second, we depict that there is a little deviation from the virtual sequence in the actual iterates, $\xs$. Finally, the convergence rate is proved.

\begin{lemma}[Descent Lemma for Multi-Walk]\label{lem1}
Under Assumptions \ref{as1}, \ref{as2}, \ref{as3}, and learning rate $\eta \leq \frac{R}{6L}$, it holds that
\begin{align}
    \EX f(\xtild[t+1]) \leq f(\xtild) -\frac{\eta}{4R} \norm{\nabla \f[]}  +\frac{2c \eta}{R}\zeta^2 (1-p')^{2|\mathcal{T}_{r_t}|} + \frac{3\eta^2L^2}{2R^2} \left(\sigma^2 + \zeta^2 \right)+ \frac{\eta L^2}{2R} \norm{\xtild - \xs}, 
\end{align}
where $\mathcal{T}_{r_t} = \{t' \leq t: r_{t'} = r_t\}$.
\end{lemma}

\begin{proof}
Based on the definition of $\xtild$ and $L$-smoothness of $f(\vec{x})$ we have
\begin{align}
    f(\xtild[t+1]) &= f(\xtild - \frac{\eta}{R} \nabla \F) \\
    &\leq f(\xtild) + \frac{\eta}{R} \inpr{\nabla f(\xtild)}{-\nabla \F} + \frac{\eta^2L}{2R^2} \norm{\nabla \F}. \label{rw1}
\end{align}
Lets take expectation of the second term on the right-hand side of (\ref{rw1}).
\begin{align}
    \frac{\eta}{R} \EX &\inpr{\nabla f(\xtild)}{-\nabla \F} \\
    &= \frac{\eta}{R} \EX_{v_t} \EX_{\xi_{t+\hat{\tau}_t}} \inpr{\nabla f(\xtild)}{-\nabla \F}\\
    &= \frac{\eta}{R} \EX_{v_t} \inpr{\nabla f(\xtild)}{-\nabla \f}\\
    &= \frac{\eta}{R} \inpr{\nabla f(\xtild)}{-\EX_{v_t} \nabla \f}\\
    &= \frac{\eta}{R} \inpr{\nabla f(\xtild)}{-\nabla \f[] + \nabla \f[] -\EX_{v_t} \nabla \f}\\
    &= \frac{\eta}{R} \underbrace{\inpr{\nabla f(\xtild)}{-\nabla \f[]}}_{=:T_1} +\frac{\eta}{R} \underbrace{\inpr{\nabla f(\xtild)}{ \nabla \f[] -\EX_{v_t} \nabla \f}}_{=:T_2}.
\end{align}
We estimate $T_1$ and $T_2$ separately.
\begin{align}
    T_1 &= -\frac{1}{2} \norm{\nabla f(\xtild)} -\frac{1}{2} \norm{\nabla \f[]} + \frac{1}{2} \norm{\nabla f(\xtild) -  \nabla \f[] }.
\end{align}
We also obtain

\begin{align}
    T_2 &\leq \frac{1}{2}\norm{\nabla f(\xtild)} + \frac{1}{2} \norm{\EX_{v_t} \left[\nabla \f - \f[]\right]} \label{rw1.5}\\
    &= \frac{1}{2} \norm{\nabla f(\xtild)} + \frac{1}{2} \norm{\sum_{v=1}^V P_v^t \left(\nabla \f - \f[]\right)}\\
    &= \frac{1}{2} \norm{\nabla f(\xtild)} + \frac{1}{2} \norm{\sum_{v=1}^V (P_v^t - \pi_v) \left(\nabla \f - \f[]\right)}\\
    & \leq \frac{1}{2} \norm{\nabla f(\xtild)} + \frac{1}{2} \left(\sum_{v=1}^V |P_v^t - \pi_v| \normm{\nabla \f - \f[]}\right)^2\\
    & \leq \frac{1}{2} \norm{\nabla f(\xtild)} + \frac{1}{2} \zeta^2\left(\sum_{v=1}^V |P_v^t - \pi_v|\right)^2\\
    & \leq \frac{1}{2} \norm{\nabla f(\xtild)} + \frac{1}{2} \zeta^2\left(2\normm{P^t - \pi}_{TV}\right)^2 \label{rw2} \\
    & \leq \frac{1}{2} \norm{\nabla f(\xtild)} + 2 c \zeta^2 (1-p')^{2|\mathcal{T}_{r_t}|}, \label{rw3}
\end{align} 
 where (\ref{rw1.5}) is based on the fact that for any $\lambda > 0$,
\begin{align}
2\inpr{a}{b} \leq \lambda \norm{a} + \frac{1}{\lambda} \norm{b}.\label{lambda_ineq}
\end{align} 
$P_v^t$ shows the probability of being at node $v$ at iteration $t$ and $\pi_v$ is the steady state distribution of node $v$.
In (\ref{rw2}) we have used the fact that the total variation distance between two probability distributions $\mu$ and $\nu$ on $\mathcal{X}$ satisfies 
\begin{align}
    \normm{\mu - \nu}_{TV} = \frac{1}{2} \sum_{x \in \mathcal{X}} |\mu(x) - \nu(x)|.
\end{align}
(\ref{rw3}) is based on the following well-known bound on the mixing time for a Markov chain (see, for example, \citet{guruswami2016rapidlymixingmarkovchains,Levin2017MarkovCA}).
\begin{align}
    \normm{P^t - \pi}_{TV} \leq c(1-p')^{|\mathcal{T}_{r_t}|},
\end{align}
where $\mathcal{T}_{r_t} = \{t' \leq t: r_{t'} = r_t\}$ is the set of all iteration on walk $r_t$.
$(1-p')$ is the second largest eigenvalue of matrix $\vec{P}$ representing the irreducible aperiodic Markov chain of each walk and $c > 0$ is a constant.

So we get 
\begin{align}
    \frac{\eta}{R} \EX &\inpr{\nabla f(\xtild)}{-\nabla \F} \leq -\frac{\eta}{2R} \norm{\nabla \f[]} + \frac{\eta}{2R} \norm{\nabla f(\xtild) -  \nabla \f[] } +\frac{2c \eta}{R}\zeta^2 (1-p')^{2|\mathcal{T}_{r_t}|}.
\end{align}

Now we derive expectation of the last term on the right-hand side of (\ref{rw1}).
\begin{align}
    &\EX\norm{\nabla \F} = \EX\norm{\nabla \F \pm \nabla \f \pm \nabla \f[]}\\
    &\leq 3\EX \norm{\nabla \F - \nabla \f} + 3\EX \norm{\nabla \f - \nabla \f[]} + 3 \norm{\nabla \f[]} \label{rw4}\\
    &\leq 3\sigma^2 + 3\zeta^2 + 3 \norm{\nabla \f[]},
\end{align}
where (\ref{rw4}) is based on the following inequality.
\begin{align}
    \norm{\sum_{i=1}^n \vec{a}_i} \leq n \sum_{i=1}^n \norm{\vec{a}_i}.\label{ineq_sum}
\end{align}

Combining these together and using $L$-smoothness to estimate $\norm{\nabla f(\xtild) -  \nabla \f[] }$ we obtain
\begin{align}
    \EX f(\xtild[t+1]) \leq f(\xtild) -\left(\frac{\eta}{2R} - \frac{3\eta^2L}{2R^2} \right)\norm{\nabla \f[]} + \frac{\eta L^2}{2R} \norm{\xtild - \xs} +\frac{2c \eta}{R}\zeta^2 (1-p')^{2|\mathcal{T}_{r_t}|} + \frac{3\eta^2L}{2R^2} \left(\sigma^2 + \zeta^2 \right).
\end{align}

Considering $\eta \leq \frac{R}{6L}$ we obtain
\begin{align}
    \EX f(\xtild[t+1]) \leq f(\xtild) -\frac{\eta}{4R} \norm{\nabla \f[]} + \frac{\eta L^2}{2R} \norm{\xtild - \xs} +\frac{2c \eta}{R}\zeta^2 (1-p')^{2|\mathcal{T}_{r_t}|} + \frac{3\eta^2L}{2R^2} \left(\sigma^2 + 2\zeta^2 \right).
\end{align}
\end{proof}

\begin{lemma}[Bounding Deviation for Multi-Walk]\label{lem2}
Under Assumptions \ref{as2}, \ref{as3}, \ref{as4}, and learning rate $\eta \leq \frac{1}{7LH}$, it holds that
\begin{align}
          \frac{1}{T} \sum_{t=1}^T \EX \norm{\xtild - \xs} \leq 12V\sigma^2 \eta^2   + 12H^2\zeta^2 \eta^2  + \frac{1}{4L^2T}\sum_{t=1}^{T-1}  \EX\norm{\nabla \f[]},\label{ap_eq_lem2}
\end{align}
where $H^2$ is the second moment of the first return time to the \master.
\end{lemma}

\begin{proof}
First we define $\lrt[r]$ as the last iteration before $t$ when walk $r$ has visited \master, \ie $\lrt[r] = \max \{t'\mid  t'\leq t, r_t = r, v_t = 0\}$.

\begin{align}
    &\EX \norm{\xtild - \xs} = \EX \norm{\sum_{z=\lrt[r_t][\lrt], r_z \neq r_t }^{t-1} -\frac{\eta}{R} \nabla \Fz +  \sum_{z=\lrt, r_z = r_t }^{t-1} \left(1-\frac{1}{R}\right) \eta \nabla \Fz}\\
    &\leq \frac{2}{R^2} \EX\norm{\sum_{z=\lrt[r_t][\lrt], r_z \neq r_t }^{t-1} \eta \nabla \Fz} + 2 \EX\norm{\sum_{z=\lrt, r_z = r_t }^{t-1} \eta \nabla \Fz}\\
    &\leq \underbrace{\frac{2}{R^2} \EX\norm{\sum_{z \in U^1_t}\eta \nabla \Fz}}_{:=T_1} + \underbrace{2 \EX\norm{\sum_{z \in U^2_t} \eta \nabla \Fz}}_{:=T_2},
\end{align}
where $U^1_t =  \{ \lrt[r_t][\lrt] \leq z \leq t-1 \mid r_z \neq r_t \}$, and $U^2_t =  \{ \lrt \leq z \leq t-1 \mid r_z = r_t \}$.

We have 
\begin{align}
    \nabla \F = \left( \nabla \F - \nabla \f \right) + \left( \nabla \f - \nabla \f[]\right) + \nabla \f[].
\end{align}

So, based on (\ref{ineq_sum}) we can write
\begin{align}
    T_1 &\leq \frac{6}{R^2} \EX\biggl( \norm{\sum_{z \in U^1_t}\eta \left( \nabla \Fz - \nabla \fz \right)} \\
     && \mathllap{+ \norm{\sum_{z \in U^1_t}\eta \left( \nabla \fz - \nabla \fz[]\right)}+ \norm{\sum_{z \in U^1_t}\eta\nabla \fz[]} \biggl)} \nonumber\\
     &\leq \frac{6}{R^2} \EX\left( \sum_{z \in U^1_t}\eta^2 \sigma^2 + |U^1_t|\sum_{z \in U^1_t}\eta^2 \zeta^2 + |U^1_t|\sum_{z \in U^1_t}\eta^2 \norm{\nabla \fz[]}\right), \label{rw5}
\end{align}
where in (\ref{rw5}) we have applied (\ref{ineq_sum}) and the fact that for independent zero-mean random variables, we get a tighter bound as follows.
\begin{align}
    \EX \norm{\sum_{i=1}^n \vec{a}_i} \leq \sum_{i=1}^n \EX \norm{\vec{a}_i}.\label{zero-mean_ineq_sum}
\end{align}

Averaging over $T$, we get
\begin{align}
    &\frac{1}{T} \sum_{t=1}^{T-1} T_1 \leq \frac{6}{TR^2} \EX \left( \sum_{t=1}^{T-1} \sum_{z \in U^1_t}\eta^2 \sigma^2  + \sum_{t=1}^{T-1} |U^1_t|\sum_{z \in U^1_t}\eta^2 \zeta^2 + \sum_{t=1}^{T-1} |U^1_t|\sum_{z \in U^1_t}\eta^2 \norm{\nabla \fz[]}\right) \\
    &\leq \frac{6}{TR^2} \EX \left( \sum_{t=1}^{T-1}|U^1_t|\eta^2 \sigma^2  + \sum_{t=1}^{T-1} |U^1_t|^2\eta^2 \zeta^2 + \sum_{t=1}^{T-1} |U^1_t| \sum_{z \in U^1_t} \eta^2 \norm{\nabla \fz[]}\right)\\
    &\leq \frac{6}{TR^2} \EX \left( \sum_{t=1}^{T-1}(R-1)h\eta^2 \sigma^2  + \sum_{t=1}^{T-1} (R-1)^2h^2\eta^2 \zeta^2 + (R-1)h \eta^2 \sum_{t=1}^{T-1} \sum_{z \in U^1_t}  \norm{\nabla \fz[]}\right) \label{rw5.5}\\
    &\leq \frac{6}{TR^2} \EX \left( \sum_{t=1}^{T-1}(R-1)h\eta^2 \sigma^2  + \sum_{t=1}^{T-1} (R-1)^2h^2\eta^2 \zeta^2 + (R-1)^2h^2 \eta^2 \sum_{t=1}^{T-1}  \norm{\nabla \f[]}\right) \label{rw5.6}\\
    &\leq \frac{6}{TR^2}  \left( \sum_{t=1}^{T-1}\left(R-1\right)V\eta^2 \sigma^2  + \sum_{t=1}^{T-1} \left(R-1\right)^2 H^2\eta^2 \zeta^2 + \sum_{t=1}^{T-1} \left(R-1\right)^2 H^2\eta^2 \EX\norm{\nabla \fz[]}\right) \label{rw6}\\
    &\leq \frac{6}{T} \left( \sum_{t=1}^{T-1}\frac{V}{R}\eta^2 \sigma^2  + \sum_{t=1}^{T-1} H^2\eta^2 \zeta^2 + \sum_{t=1}^{T-1}  H^2\eta^2 \EX\norm{\nabla \f[]}\right),
\end{align}
where in (\ref{rw5.5}) and (\ref{rw5.6}), we have used the fact that $|U^1_t|$ is upper bounded with  $R-1$ times the first return time to \master ($h$). Expectation of the first return time is $\frac{1}{\pi_0} = V$ and the second moment of this random variable is assumed $H^2$ that are applied in (\ref{rw6}).

Following the same approach for $T_2$ and considering $|U^2_t|$ is upper bounded with the first return time to \master. we can get
\begin{align}
    \frac{1}{T} \sum_{t=1}^{T-1} T_2 &\leq \frac{6}{T} \left( \sum_{t=1}^{T-1}V\eta^2 \sigma^2  + \sum_{t=1}^{T-1} H^2\eta^2 \zeta^2 + \sum_{t=1}^{T-1}  H^2\eta^2 \EX\norm{\nabla \f[]}\right).
\end{align}

Putting these together, we obtain 
\begin{align}
     \frac{1}{T} \sum_{t=1}^T \EX \norm{\xtild - \xs} &\leq \frac{12}{T} \left( \sum_{t=1}^{T-1}V\eta^2 \sigma^2  + \sum_{t=1}^{T-1} H^2\eta^2 \zeta^2 + \sum_{t=1}^{T-1}  H^2\eta^2 \EX\norm{\nabla \fz[]}\right)\\
     &\leq 12V\eta^2 \sigma^2  + 12H^2\eta^2 \zeta^2 + \frac{12H^2\eta^2}{T}\sum_{t=1}^{T-1}   \EX\norm{\nabla \f[]}.
\end{align}

Let $\eta \leq \frac{1}{7LH}$ to get
\begin{align}
     \frac{1}{T} \sum_{t=1}^T \EX \norm{\xtild - \xs} \leq 12V\sigma^2 \eta^2   + 12H^2\zeta^2 \eta^2  + \frac{1}{4L^2T}\sum_{t=1}^{T-1}  \EX\norm{\nabla \f[]}.
\end{align}
\end{proof}

Now we complete the proof of Theorem \ref{T1}. By multiplication of $\frac{4R}{\eta}$ in both sides and averaging over $t$ in lemma \ref{lem1}, we get
\begin{align}
    \frac{1}{T} \sum_{t=0}^{T-1} \EX \norm{\nabla \f[]}  &\leq \frac{1}{T} \sum_{t=0}^{T-1} \frac{4R}{\eta} \left (f(\xtild) - \EX  f\left(\xtild[t+1]\right) \right)  + \frac{1}{T}\sum_{t=0}^{T-1} 8c\zeta^2 (1-p')^{2|\mathcal{T}_{r_t}|} \\
    && \mathllap{+ \frac{6\eta L^2}{R} \left(\sigma^2 + \zeta^2 \right)+ \frac{1}{T}\sum_{t=0}^{T-1}2L^2 \EX \norm{\xtild - \xs}}. \nonumber
\end{align}
By replacing result of lemma \ref{lem2} and using $\sum_{t=0}^{T-1} (1-p')^{2|\mathcal{T}_{r_t}|} \leq \sum_{t=0}^{T-1} (1-p')^{|\mathcal{T}_{r_t}|} \leq R\sum_{t=0}^{T-1} (1-p')^{t} \leq \frac{R}{p'}$ (Since we have $R$ walks in total, at each iteration some walk may be active. Hence, replacing $|\mathcal{T}_{r_t}|$ with $t$ requires adding a $R$ factor  to the summation across all possible walks up to iteration), then rearranging, we have
\begin{align}
    \frac{1}{2T} \sum_{t=0}^{T-1} \EX \norm{\nabla \f[]} & \leq \frac{1}{T}\sum_{t=0}^{T-1} \frac{4R}{\eta} \left (f(\xtild) - \EX  f\left(\xtild[t+1]\right) \right)  + \frac{8cR\zeta^2}{p'T} + \frac{6\eta L^2}{R} \left(\sigma^2 + \zeta^2 \right) \label{rw7} + 24L^2\left(V\sigma^2 + H^2 \zeta^2 \right) \eta^2.
\end{align}

Now, we state a lemma to obtain the final convergence rate based on (\ref{rw7}).
\begin{lemma}[Similar to Lemma 16 in \cite{unified-koloskova20a}]\label{lem5}
    For every non-negative sequence $\{r_t\}_{t\geq0}$ and any parameters $d \geq 0, b\geq 0, c\geq 0, T\geq 0$, there exist a constant $\eta \leq \frac{1}{d}$, it holds
    \begin{align}
        \frac{1}{T\eta}\sum_{t=0}^{T-1}\big(r_t- r_{t+1} \big)+ b\eta + c\eta^2 \leq \frac{2\sqrt{br_0}}{\sqrt{T}} + 2 (\frac{r_0\sqrt{c}}{T})^{\frac{2}{3}} + \frac{dr_0}{T}. 
    \end{align}
\end{lemma}
\begin{proof}
By canceling the same terms in the telescopic sum, we get
        \begin{align}
        \frac{1}{T\eta}\sum_{t=0}^{T-1}\big(r_t- r_{t+1} \big)+ b\eta + c\eta^2 \leq \frac{r_0}{T\eta}+ b\eta + c\eta^2.
    \end{align}

It is now followed by a $\eta$-tuning, the same way as in \cite{unified-koloskova20a}, which shows we need to choose $\eta = \min\{\frac{1}{d},\sqrt{\frac{r_0}{bT}},(\frac{r_0}{cT})^{\frac{1}{3}}\}$.
\end{proof}

Bounding the right hand side of inequality (\ref{rw7}) with Lemma \ref{lem5} and considering that $\eta=\eta \leq \frac{1}{7LH}$, provides $\frac{1}{T}\sum_{t=0}^{T-1}\EX\norm{\nabla f(\xbar)}$ is
\begin{align}
     \mathcal{O} \bigg(\frac{(f(\xz)-f^*)RLH}{T} + \frac{R\zeta^2}{p'T}+\frac{\sqrt{ L(f(\xz)-f^*)(\sigma^2 +\zeta^2)}}{\sqrt{T}}+ 
     (\frac{RL(f(\xz)-f^*)\sqrt{V\sigma^2 + H^2 \zeta^2 }}{T})^{\frac{2}{3}} \bigg).
\end{align}

This completes the proof of Theorem \ref{T1}.

\section{Proof of Theorem \ref{T2}}\label{appendixc}

For Async-\gos algorithm, we define a virtual sequence $\{\xtild\}_{t\geq 0}$ as shown below.
\begin{equation}
\xtild[t+1] =\xtild - \frac{\eta}{V} \nabla \gF.
\end{equation}

\begin{lemma}[Descent Lemma for Async-Gossip]\label{lem3}
Under Assumptions \ref{as1}, \ref{as2}, \ref{as3}, and learning rate $\eta \leq \frac{V}{4L}$, it holds that
\begin{align}
    \EX f(\xtild[t+1]) \leq f(\xtild) -\frac{\eta}{4V} \norm{\nabla \gf[]} + \frac{\eta L^2}{2V} \norm{\xtild - \gxs} + \frac{\eta^2L}{2V^2} \left(\sigma^2 + 2\zeta^2 \right).
\end{align}
\end{lemma}

\begin{proof}
Based on the definition of $\xtild$ and $L$-smoothness of $f(\vec{x})$ we have
\begin{align}
    f(\xtild[t+1]) &= f(\xtild - \frac{\eta}{V} \nabla \gF) \\
    &\leq f(\xtild) + \frac{\eta}{V} \inpr{\nabla f(\xtild)}{-\nabla \gF} + \frac{\eta^2L}{2V^2} \norm{\nabla \gF}. \label{eq39}
\end{align}
Lets take expectation of the second term on the right-hand side of (\ref{eq39})
\begin{align}
    \frac{\eta}{V} \EX &\inpr{\nabla f(\xtild)}{-\nabla \gF} \\
    &= \frac{\eta}{V} \EX_{v_t} \EX_{\xi_{t+\hat{\tau}_t}} \inpr{\nabla f(\xtild)}{-\nabla \gF}\\
    &= \frac{\eta}{V} \EX_{v_t} \inpr{\nabla f(\xtild)}{-\nabla \gf}\\
    &= \frac{\eta}{V} \inpr{\nabla f(\xtild)}{- \nabla \gf[]}\\
    & = -\frac{1}{2} \norm{\nabla f(\xtild)} -\frac{1}{2} \norm{\nabla \gf[]} + \frac{1}{2} \norm{\nabla f(\xtild) -  \nabla \gf[] }\\
    &\leq  -\frac{1}{2} \norm{\nabla \gf[]} + \frac{1}{2} \norm{\nabla f(\xtild) -  \nabla \gf[]}.
\end{align}

Now we derive expectation of the last term on the right-hand side of (\ref{eq39}).
\begin{align}
    \EX\norm{\nabla \gF} &= \EX\norm{\nabla \gF \pm \nabla \gf \pm \nabla \gf[]}\\
    &\leq \sigma^2 + 2\EX \norm{\nabla \gf - \nabla \gf[]} + 2 \norm{\nabla \gf[]}\\
    &\leq \sigma^2 + 2\zeta^2 + 2 \norm{\nabla \gf[]}.
\end{align}

Combining these together and using $L$-smoothness to estimate $\norm{\nabla f(\xtild) -  \nabla \gf[] }$ we obtain
\begin{align}
    \EX f(\xtild[t+1]) \leq f(\xtild) -\left(\frac{\eta}{2V} - \frac{\eta^2L}{V^2} \right)\norm{\nabla \gf[]} + \frac{\eta L^2}{2V} \norm{\xtild - \gxs} + \frac{\eta^2L}{2V^2} \left(\sigma^2 + 2\zeta^2 \right).
\end{align}

Considering $\eta \leq \frac{V}{4L}$ we obtain
\begin{align}
    \EX f(\xtild[t+1]) \leq f(\xtild) -\frac{\eta}{4V} \norm{\nabla \gf[]} + \frac{\eta L^2}{2V} \norm{\xtild - \gxs} + \frac{\eta^2L}{2V^2} \left(\sigma^2 + 2\zeta^2 \right).
\end{align}
\end{proof}


\begin{lemma}[Bounding Deviation for Async-Gossip]\label{lem4}
Under Assumptions \ref{as2}, \ref{as3}, \ref{as4}, and learning rate $\eta \leq \frac{p}{14L}$, it holds that
\begin{align}
    \frac{1}{T} \sum_{t=0}^{T-1} \EX \norm{\xtild - \gxs}  \leq \frac{1}{{4L^2}} \sum_{z=0}^{T-1} \norm{\nabla \gf[][\xs[v_z][z]]} + \left(\frac{16 \sigma^2}{p} + \frac{96\zeta^2}{p^2} \right) \sum_{t=0}^{T-1}  \eta^2.\label{ap_eq_lem2}
\end{align}
\end{lemma}

\begin{proof}
We will be using the following matrix notation.
\begin{align}
    \vec{X}_t &:= \left[ \gxs[1], \dots , \gxs[V] \right] \in \mathbb{R}^{d \times V},\\
    \Tilde{\vec{X}}_t &:= \left[ \xtild[t], \dots , \xtild[t] \right]\in \mathbb{R}^{d \times V},\\
    \partial \GF &:=\left[ \nabla \gF[1][1], \dots , \nabla \gF[V][V]  \right] \in \mathbb{R}^{d \times V},\\
    \partial \Gf &:=\left[ \nabla \gf[1][\xs[1][t]], \dots , \nabla \gf[V][\xs[V][t]]  \right] \in \mathbb{R}^{d \times V}.
\end{align}

Considering that $v_t$ is uniformly random among all nodes, we have
\begin{align}
    V \EX \norm{\xtild - \gxs} &=  \EX \fnorm{\vec{X}_t - \Tilde{\vec{X}}_t}\\
    &=  \EX \fnorm{\vec{X}_{t-1} \vec{W} -\eta \partial \GF \vec{W}  - \Tilde{\vec{X}}_t}\\
    &=  \EX \fnorm{\vec{X}_{t-1} \vec{W} -\eta \partial \GF \vec{W} - \Tilde{\vec{X}}_{t-1} + \frac{\eta}{V} \partial \GF} \\
    &= \EX \fnorm{\vec{X}_{t-1} \vec{W} - \Tilde{\vec{X}}_{t-1} -\eta \partial \GF \left( \vec{W} - \frac{\vec{I}}{V}\right)}\\
    & \leq  \EX \fnorm{\vec{X}_{t-1} \vec{W} - \Tilde{\vec{X}}_{t-1} -\eta \partial \Gf \left( \vec{W} - \frac{\vec{I}}{V}\right)}\\
    && \mathllap{+ \fnorm{\eta \left( \partial \GF -\partial \Gf \right) \left( \vec{W} - \frac{\vec{I}}{V}\right)},} \nonumber
    \end{align}
where we used that $\EX \partial \GF = \partial \Gf$. We can further separate the second term as the following.
    \begin{align}
    &V \EX  \norm{\xtild - \gxs}  \leq  \EX \fnorm{\vec{X}_{t-1} \vec{W} - \Tilde{\vec{X}}_{t-1} -\eta \partial \Gf \left( \vec{W} - \frac{\vec{I}}{V}\right)} \\
    && \mathllap{+ 2 \eta^2 \fnorm{\left( \partial \GF -\partial \Gf \right) \vec{W} }+2 \frac{\eta^2}{V^2} \fnorm{\left( \partial \GF -\partial \Gf \right)}} \nonumber\\
    &\leq \EX \fnorm{\vec{X}_{t-1} \vec{W} - \Tilde{\vec{X}}_{t-1} -\eta \partial \Gf \left( \vec{W} - \frac{\vec{I}}{V}\right)}  \\
    && \mathllap{+ 2 \eta^2 \fnorm{\left( \partial \GF -\partial \Gf \right) }+2 \frac{\eta^2}{V^2} \fnorm{\left( \partial \GF -\partial \Gf \right)}} \nonumber\\
    &\leq \left( 1+\lambda \right) \EX \fnorm{\vec{X}_{t-1} \vec{W} - \Tilde{\vec{X}}_{t-1}} + \left( 1+\lambda^{-1} \right) \EX \fnorm{\eta \partial \Gf \left( \vec{W} - \frac{\vec{I}}{V}\right)}+ 2 \eta^2 V\sigma^2  \label{asy1}+2 \frac{\eta^2}{V^2} V\sigma^2\\
    &\leq \left( 1+\lambda \right)\EX \fnorm{\vec{X}_{t-1} \vec{W} - \Tilde{\vec{X}}_{t-1}} + 2\eta^2 \left( 1+\lambda^{-1} \right) \EX \fnorm{ \partial \Gf  \vec{W}}  \\
    && \mathllap{+ \frac{2\eta^2 \left( 1+\lambda^{-1} \right)}{V^2} \EX \fnorm{\partial \Gf }+ 4 \eta^2 V\sigma^2} \nonumber \\
    &\leq \left( 1+\lambda \right)\EX \fnorm{\vec{X}_{t-1} \vec{W} - \Tilde{\vec{X}}_{t-1}} + 4\eta^2 \left( 1+\lambda^{-1} \right) \EX \fnorm{ \partial \Gf } + 4 \eta^2 V\sigma^2 \\
    &\leq \left( 1+\lambda \right) (1-p)\EX \fnorm{\vec{X}_{t-1} - \Tilde{\vec{X}}_{t-1}} + 4\eta^2 \left( 1+\lambda^{-1} \right) \underbrace{\EX \fnorm{ \partial \Gf }}_{:=T_1} + 4 \eta^2 V\sigma^2.
\end{align}
(\ref{asy1}) is based on the fact that for any $\lambda > 0$,
\begin{align}
\norm{\vec{a}+\vec{b}} \leq (1+\lambda) \norm{\vec{a}} + (1+\lambda^{-1}) \norm{\vec{b}}.\label{lambda_ineq}
\end{align} 
We bound $T_1$ separately.
\begin{align}
    T_1 &= \EX \fnorm{ \partial \Gf }\\
    &= \EX \sum_{v=1}^V \norm{\nabla \gf[v][\xs[v]]}\\
    &\leq \EX \sum_{v=1}^V 2\norm{\nabla \gf[v][\xs[v]] - \nabla \gf[][\xs[v]]} + \EX \sum_{v=1}^V 2 \norm{\nabla \gf[][\xs[v]]}\\
    &\leq \EX \sum_{v=1}^V 2 \zeta^2 + \EX \sum_{v=1}^V 2 \norm{\nabla \gf[][\xs[v]]}\\
    &= 2V \zeta^2 + 2V\EX \EX_{v_t} \norm{\nabla \gf[]}\\
    &= 2V \zeta^2 + 2V\EX \norm{\nabla \gf[]}.
\end{align}

So, we get 
\begin{align}
    &\EX \norm{\xtild - \gxs}  \leq \left( 1+\lambda \right) (1-p) \EX \norm{\xtild[t-1] - \gxs[v_{t-1}][t-1]} + 8\eta^2 \left( 1+\lambda^{-1} \right) \left( \zeta^2 + \norm{\nabla \gf[]} \right) + 4 \eta^2 \sigma^2\\
    &\leq \left( 1- \frac{p}{2} \right)\EX \norm{\xtild[t-1] - \gxs[v_{t-1}][t-1]} + \frac{24}{p}\eta^2  \zeta^2 + \frac{24}{p}\eta^2 \norm{\nabla \gf[]} + 4\eta^2 \sigma^2 \label{async1}\\
    &\leq \left( 1- \frac{p}{2} \right)^{t-1}\EX \norm{\xtild[0] - \gxs[v_0][0]} + \frac{24\zeta^2}{p}\sum_{z=0}^{t-1} \eta^2 \left( 1- \frac{p}{2} \right)^{t-z}  \\
    && \mathllap{+ \frac{24}{p}\sum_{z=0}^{t-1} \eta^2 \left( 1- \frac{p}{2} \right)^{t-z}  \norm{\nabla \gf[][\xs[v_z][z]]} + 4 \sigma^2 \sum_{z=0}^{t-1} \eta^2 \left( 1- \frac{p}{2} \right)^{t-z}}\nonumber\\
    &\leq \frac{24\zeta^2}{p} \eta^2 \sum_{z=0}^{t-1} \left( 1- \frac{p}{2} \right)^{t-z} + \frac{24}{p}\eta^2 \sum_{z=0}^{t-1} \left( 1- \frac{p}{2} \right)^{t-z}  \norm{\nabla \gf[][\xs[v_z][z]]} + 4 \sigma^2 \eta^2 \sum_{z=0}^{t-1} \left( 1- \frac{p}{2} \right)^{t-z} \label{async2}\\
    &\leq  \frac{24}{p}\eta^2 \sum_{z=0}^{t-1} \left( 1- \frac{p}{2} \right)^{t-z}  \norm{\nabla \gf[][\xs[v_z][z]]} + \left(\frac{8 \sigma^2}{p} + \frac{48\zeta^2}{p^2} \right) \eta^2,
\end{align}
where we used $\lambda = \frac{p}{2}$ in (\ref{async1}). 

Now by averaging over $T$ and considering $\eta \leq \frac{p}{14L}$, we get
\begin{align}
    \frac{1}{T} \sum_{t=0}^{T-1} &\EX \norm{\xtild - \gxs}  \leq \frac{24}{pT} \sum_{t=0}^{T-1}  \eta^2 \sum_{z=0}^{t-1} \left( 1- \frac{p}{2} \right)^{t-z}  \norm{\nabla \gf[][\xs[v_z][z]]} + \left(\frac{8 \sigma^2}{p} + \frac{48\zeta^2}{p^2} \right)\frac{1}{T} \sum_{t=0}^{T-1}  \eta^2\\
    & \leq \frac{24p}{196L^2T} \sum_{z=0}^{T-1} \norm{\nabla \gf[][\xs[v_z][z]]} \sum_{t=j+1}^{T-1}  \left( 1- \frac{p}{2} \right)^{t-z}  + \left(\frac{8 \sigma^2}{p} + \frac{48\zeta^2}{p^2} \right) \frac{1}{T}\sum_{t=0}^{T-1}  \eta^2\\
    & \leq \frac{24p}{196L^2T} \sum_{z=0}^{T-1} \norm{\nabla \gf[][\xs[v_z][z]]} \sum_{t=0}^{\infty}  \left( 1- \frac{p}{2} \right)^{t-z}  + \left(\frac{8 \sigma^2}{p} + \frac{48\zeta^2}{p^2} \right) \frac{1}{T}\sum_{t=0}^{T-1}  \eta^2\\
    & \leq \frac{48}{{196L^2T}} \sum_{z=0}^{T-1} \norm{\nabla \gf[][\xs[v_z][z]]} + \left(\frac{8 \sigma^2}{p} + \frac{48\zeta^2}{p^2} \right)\frac{1}{T} \sum_{t=0}^{T-1}  \eta^2\\
    & \leq \frac{1}{{4L^2T}} \sum_{z=0}^{T-1} \norm{\nabla \gf[][\xs[v_z][z]]} + \left(\frac{8 \sigma^2}{p} + \frac{48\zeta^2}{p^2} \right)\frac{1}{T} \sum_{t=0}^{T-1}  \eta^2.
\end{align}
\end{proof}

Now we complete the proof of Theorem \ref{T2}. By multiplication of $\frac{4V}{\eta}$ in both sides and averaging over $t$ in lemma \ref{lem3}, we get
\begin{align}
     \frac{1}{T}\sum_{t=0}^{T-1}\EX \norm{\nabla \gf[]}  \leq \frac{1}{T}\sum_{t=0}^{T-1}\frac{4V}{\eta}\left(f(\xtild) -\EX f(\xtild[t+1]) \right)  + \frac{4L\eta }{V} \left(\sigma^2 + 2\zeta^2 \right)+ \frac{1}{T}\sum_{t=0}^{T-1}2L^2 \EX \norm{\xtild - \gxs}.
\end{align}
By replacing result of lemma \ref{lem4} and rearranging, we have
\begin{align}
    \frac{1}{2T}\sum_{t=0}^{T-1}\EX \norm{\nabla \gf[]}\leq \frac{1}{T}\sum_{t=0}^{T-1}\frac{4V}{\eta}\left(f(\xtild) -\EX f(\xtild[t+1]) \right)  + \frac{4L\eta }{V} \left(\sigma^2 + 2\zeta^2 \right)  \label{async-1}+ 2L^2\eta^2\left(\frac{8 \sigma^2}{p} + \frac{48\zeta^2}{p^2} \right).
\end{align}
Bounding the right-hand side of inequality (\ref{async-1}) with Lemma \ref{lem5} and considering that $\eta=\eta \leq \frac{p}{14L}$, provides $\frac{1}{T}\sum_{t=0}^{T-1}\EX\norm{\nabla \gf[]}$ is
\begin{align}
    \mathcal{O} \bigg(\frac{(f(\xz)-f^*)VL}{pT} + \frac{\sqrt{ L(f(\xz)-f^*)(\sigma^2 +\zeta^2)}}{\sqrt{T}} + (\frac{VL(f(\xz)-f^*)\sqrt{\frac{\sigma^2}{p} + \frac{\zeta^2}{p^2}}}{T})^{\frac{2}{3}} \bigg).
\end{align}

{
\section{Proof of Theorem \ref{T3}} \label{appendixd}

\begin{lemma}[Bounding Deviation for Multi-Walk with Failure]\label{lem6}
Under Assumptions \ref{as2}, \ref{as3}, \ref{as4}, and learning rate $\eta \leq \frac{1}{15LE\max\limits_{i} H_i}$, it holds that
\begin{align}
          \frac{1}{T} \sum_{t=1}^T \EX \norm{\xtild - \xs} \leq 36V\sigma^2 \eta^2   + 36H^2\zeta^2 \eta^2  + \frac{1}{4L^2T} \sum_{i=0}^{E} \sum_{t=e_i}^{e_{i+1}-1} \EX \norm{\nabla \ff[i][]},\label{ap_eq_lem5}
\end{align}
where $H^2_i$ is the second moment of the first return time to \master chosen after the $i$-th failure out of $E$ failures. We assume $e_i$ as the iteration of $i$-th failure, also $e_0 = 0, e_{E+1} = T$.
\end{lemma}

\begin{proof}
Recall $\lrt[r]$ as the last iteration before $t$ when walk $r$ has visited \master, \ie $\lrt[r] = \max \{t'\mid  t'\leq t, r_t = r, v_t = 0\}$. 
We also define $\drt[r] = \min \{t'\mid  t'\geq t, r_t = r, v_t = 0\}$ and $r_{e_i}$ as the first walk that reaches to the new Node 0 after $e_i$.

\begin{align}
    &\EX \norm{\xtild - \xs} = \EX ||\sum_{z=\lrt[r_t][\lrt], r_z \neq r_t }^{t-1} -\frac{\eta}{R} \nabla \Fz +  \sum_{z=\lrt, r_z = r_t }^{t-1} \left(1-\frac{1}{R}\right) \eta \nabla \Fz\\
    &+ \sum_{i=1}^E \sum_{r=1}^R \sum_{z=\lrt[r_t][e_i], r_z = r }^{\drt[r][e_i]-1} -\frac{\eta}{R} \nabla \Fz + \sum_{i=1}^E \sum_{z=\lrt[r_{e_i}][e_i], r_z = r_{e_i} }^{\drt[r_{e_i}][e_i]-1} \eta \nabla \Fz ||^2\\
    &\leq \frac{4}{R^2} \EX\norm{\sum_{z=\lrt[r_t][\lrt], r_z \neq r_t }^{t-1} \eta \nabla \Fz} + 4 \EX\norm{\sum_{z=\lrt, r_z = r_t }^{t-1} \eta \nabla \Fz}\\
    &+ \frac{4}{R^2} \EX\norm{\sum_{i=1}^E \sum_{r=1}^R \sum_{z=\lrt[r_t][e_i], r_z = r }^{\drt[r][e_i]-1} {\eta} \nabla \Fz }+ 4 \EX\norm{\sum_{i=1}^E \sum_{z=\lrt[r_{e_i}][e_i], r_z = r_{e_i} }^{\drt[r_{e_i}][e_i]-1} \eta \nabla \Fz }\\
    &\leq \underbrace{\frac{4}{R^2} \EX\norm{\sum_{z \in U^1_t}\eta \nabla \Fz}}_{:=T_1} + \underbrace{4 \EX\norm{\sum_{z \in U^2_t} \eta \nabla \Fz}}_{:=T_2}\\
    &+\frac{4}{R^2}\underbrace{ \EX\norm{\sum_{i=1}^E \sum_{r=1}^R \sum_{z=\lrt[r_t][e_i], r_z = r }^{\drt[r][e_i]-1} {\eta} \nabla \Fz }}_{:=T_3} +4  \underbrace{\EX\norm{\sum_{i=1}^E \sum_{z=\lrt[r_{e_i}][e_i], r_z = r_{e_i} }^{\drt[r_{e_i}][e_i]-1} \eta \nabla \Fz }}_{:=T_4},
\end{align}
where $U^1_t =  \{ \lrt[r][\lrt] \leq z \leq t-1 \mid r_z \neq r_t \}$, and $U^2_t =  \{ \lrt \leq z \leq t-1 \mid r_z = r_t \}$.
In the same way as we bounded $T_1$ and $T_2$, in Lemma \ref{lem2}, we can bound $T_3$ and $T_4$.

$T_3$ and $T_4$ is bounded by first and second moment of a random variable that  the sum of two quantities: the first return time to old \master ($h_{i-1}$) and the hitting time to the new \master. This hitting time is, in turn, upper-bounded by the first return time to new \master ($h_{i}$). Expectation of the this random variable is bounded with $2V$ and the second moment of this random variable is bounded with twice the sum of the second moments of two random variables. So the second moment is $2(\max\limits_{i} H^2_i)$.

Following the same approach as in Lemma \ref{lem2}
and assuimg $\eta \leq \frac{1}{15LE\max\limits_{i} H_i}$ to get
\begin{align}
     \frac{1}{T} \sum_{t=1}^T \EX \norm{\xtild - \xs} \leq 36 EV\sigma^2 \eta^2   + 36 E^2\zeta^2 \eta^2 \max\limits_{i} H^2_i  + \frac{1}{4L^2T} \sum_{i=0}^{E} \sum_{t=e_i}^{e_{i+1}-1} \EX \norm{\nabla \ff[i][]}.
\end{align}

\end{proof}

Now we complete the proof of Theorem \ref{T1}. By multiplication of $\frac{4R}{\eta}$ in both sides and averaging over $t$ in lemma \ref{lem1}, we get
\begin{align}
    \frac{1}{T} \sum_{i=0}^{E} \sum_{t=e_i}^{e_{i+1}-1} \EX \norm{\nabla \ff[i][]}  &\leq \frac{1}{T}  \sum_{i=0}^{E}\sum_{t=e_i}^{e_{i+1}-1} \frac{4R}{\eta} \left (\ff[i][][\xtild] - \EX  \ff[i][][\xtild[t+1]] \right)  + \frac{1}{T} \sum_{i=0}^{E} \sum_{t=e_i}^{e_{i+1}-1} 8c\zeta^2 (1-p'_i)^{2|\mathcal{T}^{i}_{r_t}|} \\
    && \mathllap{+ \frac{6\eta L^2}{R} \left(\sigma^2 + \zeta^2 \right)+ \frac{1}{T}\sum_{t=0}^{T-1}2L^2 \EX \norm{\xtild - \xs}} \nonumber.
\end{align}
By replacing result of lemma \ref{lem6} and using $\sum_{i=0}^{E}\sum_{t=e_i}^{e_{i+1}-1} (1-p')^{2|\mathcal{T}^i_{r_t}|} \leq \sum_{i=0}^{E}\sum_{t=0}^{T-1} (1-p'_i)^{|\mathcal{T}^i_{r_t}|}  \leq \sum_{i=0}^{E}R\sum_{t=0}^{T-1} (1-p'_i)^{t} \leq \sum_{i=0}^{E}\frac{R}{p'_i}$, then rearranging, we have
\begin{align}
    \frac{1}{2T} \sum_{i=0}^{E} \sum_{t=e_i}^{e_{i+1}-1} \EX \norm{\nabla \ff[i][]} & \leq \frac{1}{T}  \sum_{i=0}^{E}\sum_{t=e_i}^{e_{i+1}-1} \frac{4R}{\eta} \left (\ff[i][][\xtild] - \EX  \ff[i][][\xtild[t+1]] \right) + 8c\zeta^2 \sum_{i=0}^{E}\frac{R}{p'_i}+ \frac{6\eta L^2}{R} \left(\sigma^2 + \zeta^2 \right) \label{rw-1} \\
    && \mathllap{+ 72EL^2\left(V\sigma^2 + E \max\limits_{i} H^2_i \zeta^2 \right) \eta^2} \nonumber 
\end{align}

We assume the failure of \master does not change the objective function, \ie $f^0(\vec{x}) = \dots = f^E(\vec{x}) = f(\vec{x})$. Then
\begin{align}
    \frac{1}{2T} \sum_{t=0}^{T-1} \EX \norm{\nabla \f[]} & \leq \frac{1}{T}\sum_{t=0}^{T-1} \frac{4R}{\eta} \left (f(\xtild) - \EX  f\left(\xtild[t+1]\right) \right) + 8c\zeta^2 \sum_{i=0}^{E}\frac{R}{p'_i}+ \frac{6\eta L^2}{R} \left(\sigma^2 + \zeta^2 \right) \label{rw-2} \\
    && \mathllap{+ 72EL^2\left(V\sigma^2 + E \max\limits_{i} H^2_i \zeta^2 \right) \eta^2} \nonumber 
\end{align}

Bounding the right hand side of inequality (\ref{rw-2}) with Lemma \ref{lem5} and considering that $\eta=\eta \leq \frac{1}{15LE\max\limits_{i} H_i}$, provides $\frac{1}{T} \sum_{t=0}^{T-1} \EX \norm{\nabla \f[]}$ is
\begin{align}
     \mathcal{O} \bigg(\frac{(f(\xz)-f^*)RLE\max\limits_{i} H_i}{T} +\sum_{i=0}^{E}\frac{R\zeta^2}{p'_i}+\frac{\sqrt{ L(f(\xz)-f^*)(\sigma^2 +\zeta^2)}}{\sqrt{T}}\\
     && \mathllap{+ 
     (\frac{RL(f(\xz)-f^*)\sqrt{EV\sigma^2 + E^2 \zeta^2 \max\limits_{i} H^2_i}}{T})^{\frac{2}{3}} \bigg).} \nonumber
\end{align}

}

\section{Derivation of $H^2$}\label{appendixe}

\subsection{Complete graph under Metropolis--Hastings $\vec{P}$}

We have a complete graph on $V$ vertices, labeled $0,1,\dots,V-1.$ 
Each vertex $i$ has degree $\deg(i)=V-1.$ 
The Metropolis--Hastings (MH) probability between two adjacent vertices $(i,j)$ is 
\[
  p_{ij} \;=\; \min\!\Bigl\{\,\frac{1}{\deg(i)+1}, \,\frac{1}{\deg(j)+1}\Bigr\}.
\]
Since $\deg(i)+1 = V$ for every vertex $i$ in a complete graph, it follows that
\[
  p_{ij} \;=\; \min\!\bigl\{\tfrac{1}{V},\,\tfrac{1}{V}\bigr\} 
            \;=\;\tfrac{1}{V}.
\]
Moreover, the leftover probability is also $\tfrac{1}{V}$ for staying in place (lazy step). Hence, from any state $i$, the chain picks each of the $V$ vertices with probability $1/V$, including $i$ itself. 

Because each state is chosen uniformly at each step, independently of the past, the process $\{X_k\}_{k\ge0}$ is an iid sequence of $\operatorname{Uniform}\{0,\dots,V-1\}$. 

Define the first return time to state $0$ by
\[
h \;=\; \min\{\, k \ge 1 : X_k = 0 \,\mid\, X_0 = 0\}.
\]
Since each $X_k$ for $k\ge1$ is uniformly distributed over $\{0,\dots,V-1\}$, 
the probability that $X_k = 0$ is $1/V$, independent of previous steps. 
Thus, $h$ is a $\mathrm{Geometric}(p = 1/V)$ random variable in the usual ``first success'' sense (with success probability $1/V$ each trial). 

For a geometric random variable $Y\sim\mathrm{Geom}(p)$ (where $p=1/V$), the second moment is a standard formula:
\[
  \mathbb{E}[Y^2] 
    \;=\; \frac{2 - p}{p^2}.
\]
Plugging in $p = 1/V$ yields
\[
  H^2 \;=\; \mathbb{E}[h^2]
   \;=\; \frac{2 - \tfrac{1}{V}}{\bigl(\tfrac{1}{V}\bigr)^2}
   \;=\; V^2\Bigl(2 - \frac{1}{V}\Bigr)
   \;=\; 2V^2 - V.
\]

Hence, under Metropolis-Hastings on the complete graph of $V$ vertices, the first return time to state $0$ has second moment $2\,V^2 - V$.

\subsection{Cycle graph under Metropolis-Hastings $\vec{P}$}

Consider a cycle graph with $V$ vertices labeled $0,1,\dots,V-1$ (indices mod $V$).  
Each vertex $i$ has degree $2$, so the Metropolis--Hastings (MH) transition rule gives
\[
  p_{i,i} \;=\; \frac{1}{3}, 
  \quad
  p_{i,i+1} \;=\; \frac{1}{3}, 
  \quad
  p_{i,i-1} \;=\; \frac{1}{3},
\]
where addition/subtraction of indices is modulo $V$.
Hence from each state $i$, the chain either stays put with probability $1/3$, or moves one step left or right (each with probability $1/3$).  

Define
\[
  h \;=\; \min\{\, k \ge 1 : X_k = 0 \,\mid\, X_0 = 0\}.
\]
Our goal is to derive $\mathbb{E}[h^2]$.  
To handle this systematically, for any initial state $i$, define the \emph{first hitting time} of $0$:
\[
  T_0 \;=\; \min\{\, k \ge 1 : X_k = 0\}.
\]
And then set
\[
  m_i \;=\; \mathbb{E}[T_0 \mid X_0 = i],
  \quad
  M_i \;=\; \mathbb{E}[\,T_0^2 \mid X_0 = i].
\]
In particular, $\mathbb{E}[h^2] = M_0$, since for $i=0$, we interpret $T_0$ as the \emph{first return time} to 0.

\paragraph{Recurrences for the First Moments \boldmath$(m_i)$.}
Based on the symmetry of the topology, we consider only half of the vertices, i.e., $2 \leq i \leq \lceil \frac{V}{2} \rceil$.

\paragraph*{(1) $m_0$.}
Starting at $0$, in one step:
\begin{itemize}
\item With probability $1/3$, we \emph{stay} at 0, so the hitting time $T_0=1$ immediately.
\item With probability $1/3$ each, we move to $1$ or $V-1$.  From such a neighbor, the expected time to hit $0$ is $1 + m_1$ (by symmetry, $m_1$ is the same whether we step to $1$ or $V-1$).
\end{itemize}
Thus
\begin{align}  
  m_0 
  \;=\; \frac{1}{3}\cdot 1
        \;+\; \frac{1}{3}\bigl(1 + m_1\bigr)
        \;+\; \frac{1}{3}\bigl(1 + m_1\bigr)
  \;=\; 1 + \frac{2}{3}\,m_1.
\end{align}

\paragraph*{(2) $m_1$ (separate expression).}
From state $1$:
\begin{itemize}
\item With probability $1/3$, we jump \emph{directly} to $0$. Then $T_0 = 1$ (not $1 + m_0$, because hitting 0 completes the journey right away).
\item With probability $1/3$, we \emph{stay} at $1$. Then $T_0 = 1 + m_1$.
\item With probability $1/3$, we move to $2$. Then $T_0 = 1 + m_2$.
\end{itemize}
Hence
\[
  m_1
  \;=\; \frac{1}{3}\cdot 1 
        \;+\;\frac{1}{3}\bigl(1 + m_1\bigr)
        \;+\;\frac{1}{3}\bigl(1 + m_2\bigr).
\]
Simplify:
\begin{align} 
  m_1
  \;=\; 1 + \frac{1}{3}\,m_1 + \frac{1}{3}\,m_2
  \;\Longrightarrow\;
  \frac{2}{3}\,m_1 = 1 + \frac{1}{3}\,m_2
  \;\Longrightarrow\;
  m_1 = \frac{3}{2} + \frac{1}{2}\,m_2.
\end{align}

\paragraph*{(3) General $m_i$ for $2 \le i \le \lceil \frac{V}{2} \rceil$.}
From state $i$, we have three possibilities (stay at $i$, move to $i+1$, or move to $i-1$).  Each event occurs with probability $1/3$, and in each case we add $1$ step plus the hitting time from the new state.  Thus
\[
  m_i 
  \;=\; \frac{1}{3}\bigl(1 + m_i\bigr)
        \;+\;\frac{1}{3}\bigl(1 + m_{i+1}\bigr)
        \;+\;\frac{1}{3}\bigl(1 + m_{i-1}\bigr),
\]
where indices are taken mod $V$.  
Rearranging gives
\begin{align}
m_i
  \;=\;
  \frac{3 + m_{i+1} + m_{i-1}}{2}. \label{h1}
\end{align}

\paragraph{Recurrences for the Second Moments \boldmath$(M_i)$.}

Define $M_i = \mathbb{E}[T_0^2 \mid X_0=i]$.  We again do a first-step analysis.

\paragraph*{(1) $M_0$.}
From state $0$:
\begin{itemize}
\item With prob $1/3$, stay at $0$ immediately: $T_0=1$, contributing $1^2$.
\item With prob $2/3$, move to a neighbor (1 or $V-1$), then $T_0 = 1 + T_0'$.  Squaring, $(1 + T_0')^2 = 1 + 2T_0' + (T_0')^2$, so
  $\mathbb{E}[(1+T_0')^2] = 1 + 2\,m_1 + M_1$.
\end{itemize}
Hence
\begin{align}    
  M_0 
  \;=\; \frac{1}{3}\cdot 1^2
        \;+\;\frac{2}{3}\,\bigl[1 + 2\,m_1 + M_1\bigr]
  \;=\; 1 + \frac{4}{3}\,m_1 + \frac{2}{3}\,M_1. \label{h5}
\end{align}

\paragraph*{(2) $M_1$.}
From state $1$:
\begin{itemize}
\item With prob $1/3$, jump directly to $0$: $T_0=1$, so contribution $1^2$.
\item With prob $1/3$, stay at $1$: then $T_0 = 1 + T_0'$, so $\mathbb{E}[(1+T_0')^2] = 1 + 2\,m_1 + M_1$.
\item With prob $1/3$, move to $2$: then $T_0 = 1 + T_0''$, so $\mathbb{E}[(1+T_0'')^2] = 1 + 2\,m_2 + M_2$.
\end{itemize}
Thus
\[
  M_1 
  \;=\; \frac{1}{3}\cdot 1
        \;+\;\frac{1}{3}\bigl[\,1 + 2\,m_1 + M_1\bigr]
        \;+\;\frac{1}{3}\bigl[\,1 + 2\,m_2 + M_2\bigr].
\]
Simplifying leads to a linear relation among $M_1$, $m_1$, $m_2$, and $M_2$:
\begin{align}
  M_1 
  \;&=\; 1 + \frac{2}{3}\,m_1 + \frac{2}{3}\,m_2 
        + \frac{1}{3}\,M_1 + \frac{1}{3}\,M_2\\
  &= \frac{3}{2} + m_1 + m_2 + \frac{1}{2}\,M_2\\
  &= 3m_1 - \frac{3}{2}  + \frac{1}{2}\,M_2. \label{h4}
\end{align}

\paragraph*{(3) General $M_i$ for $2 \le i \le \lceil \frac{V}{2} \rceil$.}
By the same logic:
\[
  M_i 
  \;=\; \frac{1}{3}\bigl[1 + 2\,m_i + M_i\bigr]
        + \frac{1}{3}\bigl[1 + 2\,m_{i+1} + M_{i+1}\bigr]
        + \frac{1}{3}\bigl[1 + 2\,m_{i-1} + M_{i-1}\bigr],
\]
with indices mod $V$.  
Rearrange to get
\begin{align}
  M_i
  \;&=\; \frac{3}{2}
        \;+\;\bigl(m_i + m_{i+1} + m_{i-1}\bigr)
        \;+\;\frac{1}{2}\,\bigl(M_{i+1} + M_{i-1}\bigr)\\
  \;&=\; \frac{3}{2}
        \;+\;3\bigl(m_i -1\bigr)
        \;+\;\frac{1}{2}\,\bigl(M_{i+1} + M_{i-1}\bigr)\\
    \;&=\; \;3m_i -\frac{3}{2}
        \;+\;\frac{1}{2}\,\bigl(M_{i+1} + M_{i-1}\bigr),
\end{align}
where we have used (\ref{h1}).

\paragraph{Solving the System.}
Altogether, we have:
\[
\begin{cases}
  \text{(First moments)} \\
  m_0 = 1 + \tfrac{2}{3}\,m_1, \\[6pt]
  m_1 = \frac{3}{2} + \frac{1}{2}\,m_2, \\[6pt]
  m_i = \dfrac{3 + m_{i+1} + m_{i-1}}{2}, \quad \text{for } 2 \le i \le \lceil \frac{V}{2} \rceil,
\end{cases}
\]
\[
\begin{cases}
  \text{(Second moments)} \\
  M_0 = 1 + \tfrac{4}{3}\,m_1 + \tfrac{2}{3}\,M_1, \\[6pt]
  M_1 = 3m_1 - \frac{3}{2}  + \frac{1}{2}\,M_2\\[6pt]
  M_i = 3m_i -\frac{3}{2}
        \;+\;\frac{1}{2}\,\bigl(M_{i+1} + M_{i-1}\bigr),
  \quad \text{for } 2 \le i \le \lceil \frac{V}{2} \rceil.
\end{cases}
\]
One can solve this $2\lceil \frac{V}{2} \rceil$-dimensional linear system to find $M_0 = \mathbb{E}[h^2]$. 

Here, we assume that $V$ is even (a similar approach can be used to derive the result for $V$ being odd).

First, we solve for $m_i$, $0 \leq i \leq \frac{V}{2} $, starting from $i = \frac{V}{2}$ and using $m_{\frac{V}{2}-1} = m_{\frac{V}{2}+1}$, we get
\begin{align}
    m_{\frac{V}{2}} = \frac{3}{2} + m_{\frac{V}{2}-1}.
\end{align}
Putting it in the equation for $i = \frac{V}{2} - 1$, we obtain
\begin{align}
    m_{\frac{V}{2}-1} &= \frac{3+m_{\frac{V}{2}}+m_{\frac{V}{2}-2}}{2}\\
    &=\frac{3+ \frac{3}{2} + m_{\frac{V}{2}-1} +m_{\frac{V}{2}-2}}{2}.
\end{align}
By rearranging the terms, we derive
\begin{align}
    m_{\frac{V}{2}-1} = 3 + \frac{3}{2} + m_{\frac{V}{2}-2}.
\end{align}

By doing this, we observe the general relationship of
\begin{align}
    m_{\frac{V}{2}-i} = 3i + \frac{3}{2} + m_{\frac{V}{2}-i-1}, \label{gen_m}
\end{align}
where $0\le i \le \frac{V}{2} - 2$. Putting $i=\frac{V}{2} - 2$, gives us
\begin{align}
    m_{2} = \frac{3V}{2} - \frac{9}{2} + m_{1}.
\end{align}
So, we will reach to the following equations
\[
\begin{cases}
  m_0 = 1 + \tfrac{2}{3}\,m_1, \\[6pt]
  m_1 = \frac{3}{2} + \frac{1}{2}\,m_2, \\[6pt]
   m_{2} = \frac{3V}{2} - \frac{9}{2} + m_{1},
\end{cases}
\]
which provides us with $m_0 = V, m_1 = \frac{3V}{2} + \frac{3}{2}$.
Using (\ref{gen_m}) iteratively we get
\begin{align}
    m_{\frac{V}{2}-i} &= 3i + \frac{3}{2} + m_{\frac{V}{2}-i-1}\\
    &= 3i + \frac{3}{2} + 3(i-1) + \frac{3}{2} + m_{\frac{V}{2}-i-2}\\
    &= 3\left(i + (i-1) + \dots + (\frac{V}{2}-2) \right) + \frac{3}{2}(\frac{V}{2}-i) + m_{1}\\
    &= 3 \frac{(\frac{V}{2}-2-i)(\frac{V}{2}-2+i)}{2} + \frac{3}{2}(\frac{V}{2}-i) + \frac{3V}{2} + \frac{3}{2}\\
    & = \mathcal{O}(V^2).
\end{align}

Now, we repeat the same approach for the second moment variables. starting from $i =\frac{V}{2}$ and using $M_{\frac{V}{2}-1} = M_{\frac{V}{2}+1}$ based on symmetry, we get
\begin{align}
        M_{\frac{V}{2}} = 3m_{\frac{V}{2}} - \frac{3}{2} + M_{\frac{V}{2}-1}.
\end{align}
Putting it in the equation for $i = \frac{V}{2} - 1$, we obtain
\begin{align}
    M_{\frac{V}{2}-1} &= 3m_{\frac{V}{2}-1} -\frac{3}{2} + \frac{1}{2} (M_{\frac{V}{2}} + M_{\frac{V}{2}-2})\\
    &=3m_{\frac{V}{2}-1} -\frac{3}{2} + \frac{1}{2} (3m_{\frac{V}{2}} - \frac{3}{2} + M_{\frac{V}{2}-1} + M_{\frac{V}{2}-2}).
\end{align}
By rearranging the terms, we derive
\begin{align}
    M_{\frac{V}{2}-1} =6m_{\frac{V}{2}-1} +3m_{\frac{V}{2}} -3 - \frac{3}{2}  + M_{\frac{V}{2}-2}.
\end{align}
By keep doing this, we observe the general relationship of
\begin{align}
    M_{\frac{V}{2}-i} =6\left(m_{\frac{V}{2}-i} + \dots + m_{\frac{V}{2}-1}\right) +3m_{\frac{V}{2}} -3i - \frac{3}{2}  + M_{\frac{V}{2}-i-1}, \label{gen_M}
\end{align}
where $0\le i \le \frac{V}{2} - 2$. Putting $i=\frac{V}{2} - 2$, gives us
\begin{align}
    M_{2} =6\left( \sum_{i=2}^{\frac{V}{2}-1} m_{i} \right) +3m_{\frac{V}{2}} -3(\frac{V}{2}-2) - \frac{3}{2}  + M_{1}. \label{h3}
\end{align}
Applying (\ref{h3}) in (\ref{h4}) provides
\begin{align}
    M_1 = 6\left( \sum_{i=1}^{\frac{V}{2}-1} m_{i} \right) +3m_{\frac{V}{2}} -3(\frac{V}{2}-1) - \frac{3}{2}.
\end{align}
If we use this in (\ref{h5}) we obtain
\begin{align}
    H^2 \;=\; \mathbb{E}[h^2] = M_0 &= 1 + \frac{4}{3}\,m_1 + \frac{2}{3}\,M_1 = \mathcal{O}(V^3),
\end{align}
this is due to the fact that we derived $m_i = \mathcal{O}(V^2)$ earlier.

\end{document}